%% file: main.tex
\documentclass[twoside,11pt]{article}

\usepackage{algorithm}
\usepackage{algorithmic}
\usepackage{xcolor}

\usepackage[nohyperref]{jmlr2e}
\usepackage{graphicx}
\usepackage{amsmath}
\usepackage{mathrsfs}

\usepackage{thmtools,thm-restate}
\usepackage{microtype}
\usepackage{booktabs}
\usepackage[hidelinks]{hyperref}

\newcommand{\htheta}{\hat{\theta}}
\newcommand{\ruleset}{M}

\newcommand{\ly}[1]{\textcolor{black}{#1}}
\newcommand{\Ptrue}{P}
\newcommand{\lincen}[1]{\textcolor{black}{#1}}
\newcommand{\crev}[1]{\textcolor{black}{#1}} 
\newcommand{\crevb}[1]{\textcolor{black}{#1}} 

\newif\ifshowchanges
\showchangesfalse
\ifshowchanges
  \usepackage[normalem]{ulem}
  \definecolor{editorange}{RGB}{217,95,2}
  \newcommand{\cfix}[1]{\textcolor{red}{\relax#1}}
  \DeclareRobustCommand{\cnew}[1]{\textcolor{editorange}{\relax#1}}
  \DeclareRobustCommand{\cdel}[1]{\textcolor{editorange}{\sout{#1}}\hspace{0pt}}
  \definecolor{editpurple}{RGB}{128,0,160}
  \DeclareRobustCommand{\cfour}[1]{\textcolor{editpurple}{\relax#1}}
  \DeclareRobustCommand{\cfourdel}[1]{\textcolor{editpurple}{\sout{#1}}\hspace{0pt}}
\else
  \newcommand{\cfix}[1]{#1}
  \DeclareRobustCommand{\cnew}[1]{#1}
  \DeclareRobustCommand{\cdel}[1]{}
  \DeclareRobustCommand{\cfour}[1]{#1}
  \DeclareRobustCommand{\cfourdel}[1]{}
\fi
\newcommand{\litset}{V}                               % value set of a literal, X_j \in V_j    (was R, which is also the generic region in Section 3.3)
\newcommand{\numlit}{m}                               % number of literals in a rule            (was k, which is also the class index)
\newcommand{\poprisk}{\mathcal{L}_{\mathrm{pop}}}     % population risk in Definition 2         (was \mathcal{L}, which is also the number of categorical levels)
\newcommand{\ovmass}{\rho}                            % overlap-mass constant in (H1)           (was \beta, which is also the categorical parameters)
\newcommand{\mixw}{\pi}                               % mixture weights in the proof of Corollary 5 (was \theta, which is also the model parameters)

\usepackage{tikz} % for the running-example figure
\newenvironment{proofof}[1]{\par\noindent{\bf Proof of #1\ }}{\hfill\BlackBox\\[2mm]} % named proofs in the appendix

\renewcommand\vec{\boldsymbol}

\usepackage{lastpage}
\jmlrheading{27}{2026}{1-\pageref{LastPage}}{1/24; Revised 8/26}{9/26}{24-0085}{Lincen Yang and Matthijs van Leeuwen}

\ShortHeadings{Probabilistic Truly Unordered Rule Sets}{Yang and van Leeuwen}
\firstpageno{1}

\begin{document}

\title{Probabilistic Truly Unordered Rule Sets}

\author{\name Lincen Yang \email l.yang@liacs.leidenuniv.nl \\
       \addr LIACS, 
       Leiden University\\
       Einsteinweg 55, 2333 CC, Leiden, The Netherlands
       \AND
       \name Matthijs van Leeuwen \email m.van.leeuwen@liacs.leidenuniv.nl \\
       \addr LIACS, 
       Leiden University\\
              Einsteinweg 55, 2333 CC, Leiden, The Netherlands}

\editor{Samuel Kaski}

\maketitle
% Unnumbered first-page footnote (same layout as the JMLR footnotes, without a mark).
\begingroup
\makeatletter
\long\def\@makefntext#1{\@setpar{\@@par\@tempdima \hsize
             \advance\@tempdima-15pt\parshape \@ne 15pt \@tempdima}\par
             \parindent 2em\noindent #1}
\makeatother
\footnotetext{\cfour{Correspondence to: Lincen Yang (\texttt{l.yang@liacs.leidenuniv.nl}).}}
\endgroup

\begin{abstract}

Rule set learning has long been studied because individual rules can be directly read and understood by domain experts, making them particularly valuable in high-stakes decision-making settings. Existing methods have several shortcomings though. First, most existing methods impose \lincen{an order} among rules, either explicitly or implicitly, which makes the \lincen{comprehension of any rule's condition depend on the conditions of preceding rules.} Second, due to the difficulty of handling conflicts caused by overlaps (i.e., instances covered by multiple rules), existing methods often do not consider \emph{probabilistic} rules. Third, learning classification rules for a multi-class target is understudied, as most existing methods focus \lincen{on either a native binary classification task, or a multi-class task decomposed \cdel{to}\cnew{into} binary tasks by adopting the \cdel{`one-vs-rest'}\cnew{one-versus-rest} learning strategy.}

To address these shortcomings, we propose TURS, for Truly Unordered Rule Sets. To resolve conflicts caused by overlapping rules, we formulate a probabilistic model class based on rule sets that encourages \emph{overlap consistency}---rules that overlap \lincen{should} give similar probabilistic outputs. 
We next formalize the problem of learning a TURS model based on the \cnew{minimum description length (}MDL\cnew{)} principle and develop a carefully designed heuristic algorithm. We benchmark against a wide range of rule-based methods and demonstrate that our method learns rule sets that have lower model complexity and highly competitive predictive performance.

\end{abstract}


\section{Introduction} \label{sec:intro}
Despite the great success of black-box models in a wide range of tasks, intrinsically interpretable machine learning models have also received a lot of attention due to their transparency and hence their applicability to sensitive real-world scenarios, such as health care and judicial systems~\citep{Rudin2019StopExp}. In this paper, we particularly focus on probabilistic rule sets for multi-class classification. 

A probabilistic rule is \cdel{in the form of}\cnew{of the form} \textbf{`IF $X$ meets a given condition, THEN \lincen{${Y \sim {P_{\cfix{\vec{\beta}}}}(Y)}$}\textbf{'}}, where the condition consists of a conjunction of literals over \lincen{the feature variables $X$, with $Y$ denoting the target variable and ${P_{\cfix{\vec{\beta}}}}$ the associated class probability with the parameter vector $\cfix{\vec{\beta}}$ to be \cdel{estimate}\cnew{estimated} from data}. \crevb{As an example, we show in Figure~\ref{fig:intro_example} three such rules for predicting flight delays.}

\input{fig_intro_example}

Rule-based methods have the unique \lincen{advantage of being accessible and interpretable not only to statisticians and data scientists but also to domain experts}, since rules can be directly read and examined\cdel{ by analysts and domain experts}. While a single rule summarizes a local pattern from the data and hence only describes a subset of the instances, existing rule-based methods adopt various approaches to \lincen{combine individual rules into a global predictive model}. How individual rules are combined, however, has important consequences for interpretability: depending on the combination mechanism, \emph{understanding a single rule may require understanding the entire model}.

These existing methods can be roughly categorized as rule lists and rule sets. To begin with, rule lists (or decision lists)~\citep{furnkranz2012foundations} connect all individual rules by an ``\textbf{IF ... ELSE IF ... ELSE ...}" statement (possibly with multiple ``else if" inside the statement). That is, when predicting a certain instance, it is evaluated against each rule in sequence and assigned the prediction of the first rule whose condition it satisfies, with the final ``else" serving as a default prediction for instances not covered by any rule. This is equivalent to specifying an explicit \lincen{and total order over} the rules. 
Therefore, since the condition of each rule depends on all preceding rules, \crevb{interpreting} a single rule requires going over all preceding rules' conditions, which is impractical when the rule list becomes large. 

\crevb{Concretely, if the rules of Figure~\ref{fig:intro_example} are read as a rule list with order $(S_1, S_2, S_3)$, the effective meaning of rule $S_2$ acquires the negation of $S_1$'s condition, and therefore becomes}
\begin{center}
\crevb{\textbf{IF} \textbf{NOT} (winter \textbf{AND} snow) \textbf{AND} departure 17:00--21:00 \textbf{AND} large hub \\ \textbf{THEN} $P(\text{delayed}) = 0.35$,}
\end{center}
\crevb{which differs from the printed \cdel{$S_2$'s condition}\cnew{condition of $S_2$}.}

Further, rule set methods put rules together without specifying explicit orders. 
In this case, an instance can be covered by a single rule or by multiple rules. When two or more rules cover the same instance, we say that the rules \emph{overlap}.
Although existing rule set methods typically claim that rules in a rule set are unordered~\citep{clark1991cn2Improve, matthijs2012diverse, kotsiantis2013decision}, we argue that \cnew{most }such methods are not \emph{truly} unordered. 
Specifically, when overlapping rules produce conflicting (probabilistic) predictions, ad-hoc schemes are commonly used, typically by ranking the involved rules according to certain criteria (e.g., accuracy) and selecting the highest-ranked rule for prediction~\citep{zhang2020diverseRuleSets,lakkaraju2016interpretable}. \crevb{While rule sets do not impose an explicit order among rules, these conflict-resolution mechanisms effectively introduce implicit (partial) orders between overlapping rules: within each overlap, the rules locally behave like a rule list, with the higher-ranked rule taking precedence.}

\crevb{Consider the example in Figure~\ref{fig:intro_example} again. Suppose the ranking criterion happens to rank $S_1$ above $S_2$\cdel{, the}\cnew{. The} effective meaning of $S_2$ then becomes}
\begin{center}
\crevb{\textbf{IF} departure 17:00--21:00 \textbf{AND} large hub \textbf{THEN} $P(\text{delayed}) = 0.35$, \\ \emph{unless a higher-ranked overlapping rule applies},}
\end{center}
\crevb{where the `unless' clause depends on the ranking criterion and on which rules happen to overlap, neither of which is visible in the printed rule set itself. The resulting dependence is thus even harder to trace than in a rule list.}

\smallskip
\noindent \textbf{Our approach.} \ly{To overcome these shortcomings of existing rule-based models, we introduce \emph{truly unordered rule sets (TURS)}. Our central contribution is a principled \emph{probabilistic} formalization of rule sets that \emph{does not introduce explicit or implicit orders among rules}. To the best of our knowledge, this is the first framework that treats rule sets as probabilistic models while \lincen{keeping the rules genuinely unordered}.}

This formalization has two fundamental consequences. First, whereas most existing rule set/list methods only predict labels as outputs~\citep{wang2017bayesian, dash2018boolean, yang2021learning}, our framework produces probabilistic predictions, in which the label is accompanied by a predicted probability. Such predictions provide substantially greater value for decision-making and knowledge discovery, especially in sensitive areas such as health care where domain experts are responsible for taking actions. \crevb{Moreover, since a rule's output is a categorical distribution over an arbitrary number of classes, the framework applies to multi-class targets directly, without resorting to the common one-versus-rest decomposition~\citep{clark1991cn2Improve, huhn2009furia}.}

\lincen{Second, our formalization involves no explicit or implicit order among rules. This makes individual rules \emph{self-standing} in the sense that whether a rule applies to a given instance \emph{only} depends on whether the rule's condition is satisfied, regardless of whether other rules also apply. This ensures that the rule itself can be interpreted without any additional (hidden) conditions, which we consider the key interpretability bottleneck of existing methods, as described above.}

\lincen{The central innovation to achieve this is that, rather than resolving conflicts through ad-hoc ranking schemes, we define the probabilistic model given a rule set in which the overlap has an inherent and well-defined probabilistic semantics. In particular, when an instance satisfies the conditions of multiple rules, we model it by the disjunction (i.e., the logical `OR') of all these rules. Notably, one may be tempted to think that if the conditions of multiple rules are met by a given instance, then it is seemingly only natural to model that instance by the conjunction (i.e., the logical `AND'); however, going beyond such cognitive inertia and explicitly assigning an `OR' meaning to the overlap enriches the probabilistic semantics of a rule set as a model, which we demonstrate in this paper to be an effective way of resolving the interpretability issue caused by the order among rules.}

\crevb{Specifically, when considering the example in Figure~\ref{fig:intro_example}, both $S_1$ and $S_2$ apply to the flight exactly as printed, and the effective meaning of $S_2$ remains}
\begin{center}
\crevb{\textbf{IF} departure 17:00--21:00 \textbf{AND} large hub \textbf{THEN} $P(\text{delayed}) = 0.35$, \\
\emph{with no hidden clause attached}.}
\end{center}
\crevb{The flight itself is modeled under the disjunction of the two conditions, i.e., conditioned on the event that a flight satisfies the condition of $S_1$ \emph{or} that of $S_2$. That is, the model semantics is enriched by allowing such disjunctions to express uncertainty among rules. }

\lincen{Importantly, besides rules being self-standing, our formalization promotes what we call \emph{overlap consistency}: for rules that overlap, overlap consistency characterizes how similar the probabilistic predictions made by these rules are. Thus, when an instance satisfies the conditions of multiple rules yet the overlap consistency is high, every single rule among these rules might be treated as an approximately faithful explanation for the prediction for that instance. We formally justify our model formalization by proving that, at the population level, overlapping rules that make dissimilar probabilistic predictions incur an unavoidable excess loss, which vanishes when the overlapping rules (approximately) agree; our formalization hence implicitly steers learning towards rule sets with high overlap consistency, which we also empirically validate in our experiments in Section~\ref{sec:exp}. }

Next, to learn a TURS model from data, we consider it a probabilistic model selection task, for which we design a model selection criterion based on the minimum description length (MDL) principle~\citep{grunwald2007minimum, grunwald2019minimum}. The MDL principle states that the best model is the one that minimizes the total number of bits required to encode the model and to encode the data given the model. Hence, the selected model balances model complexity and \cdel{goodness of fit}\cnew{goodness-of-fit} without requiring a regularization parameter to be tuned. 

Optimizing the proposed MDL-based model selection objective over the space of rule sets is computationally intractable, due to the exponential search space and the non-decomposable nature of the objective. This difficulty already arises in much simpler settings: prior work has shown that even finding an optimal single conjunctive Boolean rule is NP-hard~\citep{malioutov2017learning}. Since our setting strictly generalizes this problem by allowing richer feature predicates and probabilistic outputs, exact optimization is infeasible. \crevb{We therefore develop a tailored heuristic algorithm, with three main innovations. First, a \emph{learning speed score} measures the decrease of the MDL objective per additionally covered instance, so that candidate rules are selected by how efficiently, rather than merely by how much, they improve the objective. Second, a \emph{patience-diversified beam search} keeps candidate rules with different degrees of coverage reduction. This is to overcome the shortcoming that directly optimizing the score during rule growth favors literals that shrink a rule's coverage aggressively, which prunes the search space prematurely. Third, an \emph{MDL-based local test} looks ahead to the instances that a candidate rule leaves uncovered: while these instances must eventually be modeled by rules learned later, existing heuristics evaluate each candidate rule in isolation; thus, our MDL-based local test is designed to resolve this issue.}

In summary, our main contributions are as follows:
\ly{
\begin{enumerate}
    \item We introduce a principled probabilistic formalization of rule sets, TURS, for truly unordered rule sets, which preserves their unordered nature. Unlike most recent rule list/set methods that focus on deterministic predictions and binary classification, our framework produces class probability estimates, naturally supports multi-class targets, and provides richer transparency and uncertainty information. 
    \item We identify a fundamental interpretability issue in existing rule set methods: conflicts caused by overlapping rules are typically resolved through ad-hoc ranking schemes, which implicitly impose rule orders that harm \cdel{the }interpretability. Our model formalization addresses it because 1) the rules remain self-standing, and 2) overlaps of rules are assigned explicit probabilistic semantics.
    \item We formally show that the way we formalize the TURS model encourages overlap consistency, i.e., \cdel{the learned}\cnew{overlapping} rules tend to have similar probabilistic outputs; otherwise, we prove that \cdel{these}\cnew{their} overlaps incur extra risk on the population level. 
    \item We cast the learning of TURS as a probabilistic model selection problem and develop an MDL-based scoring criterion. This criterion automatically balances goodness-of-fit and model complexity, avoiding the need for regularization hyper-parameters and costly cross-validation.
    \item We develop a heuristic optimization algorithm tailored to the proposed objective, featuring several algorithmic innovations. Through extensive empirical evaluation across a diverse set of rule-based classification baselines, we show that TURS delivers competitive predictive performance, learns probabilistically coherent rule sets with stable \crev{behavior} \lincen{in the presence of overlap}, produces probability estimates that generalize well, and achieves lower model complexity. 
\end{enumerate}
}

\smallskip
\noindent \textbf{Comparison with our previous work.}  This paper is based on our previous work~\citep{yang2022truly}, with substantial extensions and modifications in all components. We summarize the key differences between this paper and our previous work as follows. 
\ly{First, we formally define the probabilistic model without introducing an order among rules. We also provide theoretical justification showing that rule overlap degrades goodness-of-fit unless the overlapping rules induce similar class probabilities.}
Second, we \cdel{developed an improved and new}\cnew{develop a new and improved} algorithm, with 1) a learning-speed-score heuristic motivated by the ``normalized gain" used in the CLASSY algorithm for rule lists~\citep{proencca2020interpretable}; 2) a diverse beam search approach \crev{in which the beam maintains candidates with different degrees of} ``patience", \cdel{in which the concept of patience is}\cnew{a concept} taken from the PRIM method for regression rules~\citep{friedman1999bump}; and 3) an \crev{MDL-based local test} that serves as a look-ahead strategy for instances that \crev{have} not yet been covered in previous steps.
\ly{Lastly, we substantially extend the experiments in various aspects. We discuss all these differences in more detail in Appendix~\ref{app:diff}. }

\begin{table}[t]
\centering
\resizebox{\textwidth}{!}{%
\begin{tabular}{l|llcl}
\hline
\textbf{Algorithm} & \textbf{Model Type} & \textbf{Strategy} & \textbf{Prob.} & \textbf{Overlap Handling} \\
\hline
CBA~\citep{liu1998CBA}                   & ordered rule list         & div.\ and conq.    & $\checkmark$ & explicit order \\
CN2-ordered~\citep{clark1989cn2}         & ordered rule list         & div.\ and conq.    & $\checkmark$ & explicit order \\
PART~\citep{frank1998generating}         & ordered rule list         & div.\ and conq.    & $\checkmark$ & explicit order \\
CLASSY~\citep{proencca2020interpretable} & ordered rule list         & div.\ and conq.    & $\checkmark$ & explicit order \\
RIPPER~\citep{cohen1995ripper}           & ordered list of rule sets & div.\ and conq.    & $\times$     & explicit order \\
C4.5Rules~\citep{quinlan2014c4}          & ordered list of rule sets & one-vs-rest     & $\times$     & explicit order \\
BRS~\citep{wang2017bayesian}             & rule set (binary)         & pos.\ class only & $\times$ & no conflict \\
CG~\citep{dash2018boolean}               & rule set (binary)         & pos.\ class only & $\times$ & no conflict \\
Submodular~\citep{yang2021learning}      & rule set (binary)         & pos.\ class only & $\times$ & no conflict \\
CN2-unordered~\citep{clark1991cn2Improve}& rule set                  & one-vs-rest     & $\checkmark$ & weighted avg. \\
FURIA~\citep{huhn2009furia}              & fuzzy rule set            & one-vs-rest     & $\checkmark$ & fuzzy (weighted avg.) \\
CMAR~\citep{li2001cmar}                  & rule set                  & asso.\ rule mining     & $\times$     & impl.\ order ($\chi^2$) \\
CPAR~\citep{yin2003cpar}                 & rule set                  & asso.\ rule mining     & $\times$     & impl.\ order (acc.) \\
IDS~\citep{lakkaraju2016interpretable}   & rule set                  & multi-class optim. & $\times$   & impl.\ order (F1) \\
DRS~\citep{zhang2020diverseRuleSets}     & rule set                  & multi-class optim. & $\times$   & impl.\ order (acc.) \\
\emph{TURS (ours)} & \emph{truly unordered rule set} & \emph{multi-class optim.} & $\checkmark$ & \emph{integrated prob.\ semantics} \\
\hline
\end{tabular}%
}
\caption{Comparison of rule-learning algorithms. Abbreviations: div.\ and conq.\ (divide and conquer), asso.\ rule mining (association rule mining); Prob.\ indicates whether the method produces probabilistic output; impl.\ order denotes implicit ordering, with the tie-breaking criterion in parentheses.}
\label{table:relatec}
\end{table}

\section{Related Work} \label{sec:related}
\lincen{We review related work on rule-based classification, progressing from rule lists to \cdel{single-class}\cnew{binary-class} and one-versus-rest rule learning, then to multi-class rule sets, and finally to related paradigms such as decision trees and associative rule classification. Table~\ref{table:relatec} summarizes the key properties of representative approaches.}

\smallskip \noindent
\textbf{Rule lists.}
Rules in a rule list are connected through \textsc{if-then-else} statements and are therefore explicitly ordered. During prediction, rules are evaluated sequentially; the first rule whose condition is satisfied determines the output. Representative methods include CBA~\citep{liu1998CBA}, CN2-ordered~\citep{clark1989cn2}, PART~\citep{frank1998generating}, CLASSY~\citep{proencca2020interpretable}, and \cnew{Scalable }Bayesian Rule Lists~\citep{yang2017scalable-bayesian-rule-list}. 

Rule lists can be learned efficiently using divide-and-conquer (sequential covering) strategies. That is, when a rule is induced from data, the covered instances (i.e., instances satisfying the condition of the rule) can be removed and the search space is iteratively simplified. While this approach is very efficient, the explicit ordering introduces dependencies among rules: interpreting a rule typically requires considering earlier rules in the list, which may reduce its standalone interpretability.

\smallskip \noindent
\textbf{Binary-class rule learning.}
A related line of work focuses on learning rules that characterize a single target class (the ``positive'' class), which is sufficient for binary classification~\citep{wang2017bayesian,dash2018boolean,yang2021learning,quinlan1990learning}. Instances satisfying at least one rule are classified as positive, while all remaining instances are assigned to the negative class. Since all rules predict the same class label, prediction conflicts do not arise.

Approaches in this category can adopt discrete optimization techniques with attractive theoretical guarantees~\citep{dash2018boolean}. Nevertheless, these methods introduce several limitations in our setting. First, they are typically non-probabilistic, producing deterministic predictions rather than class probability estimates. Second, explanations for negative predictions rely on the negation of all positive rules, which may become complex when many rules are present. Third, these models commonly rely on discretized or binarized features~\citep{yang2021learning,dash2018boolean}, potentially restricting the granularity of conditions over continuous variables.

\smallskip \noindent
\textbf{One-versus-rest rule learning.}
Learning rules for a single class can be extended to multi-class classification via the one-versus-rest paradigm. Two main strategies are commonly adopted. The first strategy, used by methods such as RIPPER~\citep{cohen1995ripper} and C4.5Rules~\citep{quinlan2014c4}, \lincen{imposes an order on the classes and learns their rules sequentially, removing the covered instances before proceeding to the next class. The result is an ordered list of rule sets, one per class. Because the instances covered by earlier classes are removed, the rules of a given class are learned and interpreted only over the instances not already covered by any higher-ranked class. A rule's interpretation therefore depends on the rules of all higher-ranked classes, rather than being independent across classes.}

The second strategy does not impose an explicit order among classes; instead, it learns rules independently for each class against all others. Representative examples include CN2-unordered~\citep{clark1991cn2Improve} and FURIA~\citep{huhn2009furia}. FURIA aggregates predictions from multiple fuzzy rules using certainty factors derived from the degree to which rules cover an instance. While this improves robustness, it does not yield a single-rule explanation for individual predictions. CN2-unordered resolves overlaps by computing weighted averages of rule probabilities, combining the predictions of covering rules post hoc rather than incorporating overlap considerations into the training objective.

\smallskip \noindent
\textbf{Multi-class rule sets.}
Approaches that formulate multi-class rule set learning as a direct global optimization problem are comparatively rare. Representative examples include IDS~\citep{lakkaraju2016interpretable} and DRS~\citep{zhang2020diverseRuleSets}. These methods are non-probabilistic and resolve rule conflicts using ranking criteria such as F1-score or accuracy. Such conflict resolution mechanisms effectively introduce implicit ordering among rules, leading to conceptual dependencies similar to those in rule list models. Consequently, a rule’s predictive role depends on its relative ranking rather than solely on its condition, which may weaken the correspondence between a rule and the subspace it describes, as higher-ranked rules can override predictions in overlapping regions.

\smallskip \noindent
\textbf{Decision trees and association rules.}
Decision tree methods, including CART~\citep{breiman1984classification} and C4.5~\citep{quinlan2014c4}, can be viewed as rule sets \lincen{under a structural constraint: each rule is a root-to-leaf path, and the paths must assemble into a tree rather than being arbitrary. This constraint has two consequences at the same time: 1) sharing conditions among paths through common internal nodes makes the tree a compact view of the whole model, readable from a single tree-diagram; 2) nevertheless, a rule may carry conditions that do not change the set of data points it covers and serve only to keep the tree well-formed. When each prediction is explained by a single self-contained rule, this makes the extracted rules longer than necessary, as our experiments confirm (Section~\ref{sec:exp}).}

Another related paradigm is associative rule classification, which builds classifiers using association rule mining techniques~\citep{abdelhamid2014associative}. Association rule mining~\citep{agrawal1993mining, srikant1996mining} is known to generate large collections of potentially redundant rules~\citep{chen2006new}, leading to frequent overlaps among rules. Consequently, various conflict resolution mechanisms have been proposed.

For example, CMAR~\citep{li2001cmar} groups rules by predicted class and ranks these groups using a $\chi^2$-based statistic. CPAR~\citep{yin2003cpar} modifies sequential covering by reducing instance weights rather than removing instances, and resolves conflicts using expected accuracy estimates. Lazy associative classification~\citep{veloso2006lazy} avoids global conflicts by learning rules separately for each test instance, mining association rules from an instance-specific subset of the training data, although it does not yield a global interpretable model.

\smallskip \noindent \textbf{Summary.}
Overall, existing rule-based approaches differ in how they handle rule ordering, probabilistic outputs, and overlaps. \lincen{Despite their differences, existing approaches reconcile overlaps only after the rules are learned, whether by picking one rule via ordering or ranking, or by combining rules post hoc. Our work instead treats overlaps as part of the model's probabilistic semantics, modeling them during training rather than resolving them afterward.}

\section{Truly Unordered Rule Sets}\label{section:rulesetModel}

This section introduces the TURS model, a probabilistic formulation of truly unordered rule sets.
Section~\ref{subsec:prob_formulation} defines probabilistic rules, their associated regions in the input space, and the global conditional model induced by a set of such rules, including how overlaps are handled without imposing any explicit or implicit rule order.
Section~\ref{subsec:semantic_overlap} then explains the semantic meaning of the resulting overlap treatment, and Section~\ref{subsec:formal_justification} provides a population-level justification of the proposed overlap formulation.

\subsection{Probabilistic formulation of rule sets}
\label{subsec:prob_formulation}

\lincen{We build up the formulation in three steps: we first define a single probabilistic rule, then define rule sets and the predictive model $P_\theta(Y \mid X = x)$ that a rule set induces for every instance, and finally derive the resulting likelihood, parameter estimation, and prediction procedure.}
\subsubsection{A single probabilistic rule}\label{subsubsec:single_rule}
\noindent \textbf{Condition and region.}
Denote the input random variables by $X=(X_1,\ldots,X_d)\in\mathcal X$, \lincen{where each $X_j$ may be numerical, categorical, or Boolean}, and the categorical target by $Y$ with domain $\mathcal Y=\{1,\ldots,|\mathcal Y|\}$. Also, we denote the dataset as $D=\{(x_i,y_i)\}_{i\in[n]}$ (or $(x^n,y^n)$ for short).\cnew{ Here and throughout, $[n] := \{1, \ldots, n\}$ for a positive integer $n$.}

A \emph{literal} is a basic constraint on a single feature, of the form $X_j\in \cfix{\litset}$, \lincen{where $\cfix{\litset}$ is an interval for numerical features or a subset of categories (possibly a singleton) for categorical ones.} A \emph{condition}\footnote{\lincen{The condition is sometimes also called the \emph{antecedent} of a rule in the literature.}} is a finite conjunction of literals, written as $\ell_1 \land \cdots \land \ell_m$.
Each condition induces a subset of the input space, which we call its \emph{region}:
\[
\Gamma \;:=\;\{x\in\mathcal X:\text{$x$ satisfies the condition}\}.
\]
\lincen{As a result, $\Gamma$ is an axis-aligned box (hyper-rectangle): along each axis $j$, its extent is the interval or category subset specified by the literal on $X_j$, or the entire range of $X_j$ if the condition contains no literal on that feature. Equivalently,
\[
\Gamma \;=\; \prod_{j=1}^{d} \cfix{\litset_j},
\quad
\cfix{\litset_j} \;=\;
\begin{cases}
\text{the set specified by the literal on } X_j, & \text{if such a literal is present},\\[2pt]
\text{the full domain of } X_j, & \text{otherwise}.
\end{cases}
\]
Note that a region is dataset-independent and does \emph{not} refer to the instances satisfying the given condition.}

\smallskip \noindent \textbf{A rule and its conditional probability.}
A probabilistic rule $S$ pairs a condition with a distribution over $Y$. We write
\begin{equation}
    S:\quad \underbrace{X_1 \in \cfix{\litset_1} \,\land\, X_2 \in \cfix{\litset_2} \,\land\, \cdots}_{\text{condition, inducing region } \Gamma_S} \;\longrightarrow\; P_{\vec{\beta}_S}(Y)
\end{equation}
where the condition on the left induces the region $\Gamma_S \subseteq \mathcal{X}$ as defined above, and $P_{\vec{\beta}_S}(Y)$ on the right\footnote{\lincen{The probability distribution is sometimes called the rule's \emph{consequent}.}} is the local model for the unknown, ground-truth class distribution $P(Y|X \in \Gamma_S)$; i.e., $P(Y|X \in \Gamma_S)$ is induced from the ground-truth \cdel{data generative}\cnew{data-generating} distribution \cnew{by }conditioning on the fixed event $\{X \in \Gamma_S\}$. Thus, we also write $P_{\vec{\beta}_S}(Y \mid X \in \Gamma_S)$ for $P_{\vec{\beta}_S}(Y)$ to emphasize the association between the local model and its corresponding region. 

As $Y$ takes values in the finite set $\mathcal{Y}$, $Y|X \in \Gamma_S$ follows a categorical distribution and $\vec{\beta}_S$ parameterizes it, with the parameter vector $\vec{\beta}_S$ to be estimated from data; i.e., the $k$-th element of $\vec{\beta}_S$, denoted as $\beta_{S,k}$, parameterizes the probability of the $k$-th class label.

\smallskip \noindent \textbf{Coverage.}
\crevb{When the dataset is fixed, we can further define the notion of coverage: for a rule $S$, we denote the set of instances in $D$ that satisfy its condition by
\[
    c(S) = \{(x,y) \in D : x \in \Gamma_S\},
\]
and refer to the cardinality $|c(S)|$ as its \emph{coverage}. Note that $c(S)$ is usually referred to as the \emph{support set} or \emph{cover} of the rule's condition on $D$; unlike the region $\Gamma_S$, it depends on the dataset.}

\subsubsection{Rule sets and the induced probabilistic model}
\label{subsubsec:local_models}
\noindent \textbf{The regions of rules and the `else region'.} Given a rule set with $K$ individual rules, we write the collection of rules as 
\[
    M = \{S_1, \ldots, S_K\},
\]
and denote their corresponding regions by $\Gamma_1, \ldots, \Gamma_K \subseteq \mathcal{X}$.

For any non-empty index set $I \subseteq [K]$, we define
\[
    \Gamma_I^{\cup} := \bigcup_{i \in I} \Gamma_i
    \quad\text{and}\quad
    \Gamma_I^{\cap} := \bigcap_{i \in I} \Gamma_i.
\]
We use these to distinguish between the region where at least one rule in $I$ applies ($\Gamma_I^{\cup}$) and the region where all rules in $I$ apply simultaneously ($\Gamma_I^{\cap}$).

Next, we introduce the `else region'
\[
    \Gamma_0 := \mathcal{X} \setminus \bigcup_{i=1}^K \Gamma_i,
\]
which covers all instances that are not covered by any rule in $M$. The associated `else rule' is denoted by $S_0$. 

\crev{Strictly speaking, $S_0$ is not a rule in the sense defined above, as no condition is explicitly specified; it can be interpreted as a default rule whose (implicit) condition is the negation of the disjunction of the conditions of $S_1, \ldots, S_K$, inducing the region $\Gamma_0$. We nevertheless refer to it as the `else rule' for convenience.}

\smallskip \noindent \textbf{The probabilistic model, case by case.}
\crevb{We now switch the perspective from the rules to a single instance in the data: we define a global probabilistic model for $M$ by assigning, for every $x \in \mathcal{X}$, a conditional model based on the regions that cover $x$. We denote the model as $P_\theta(Y \mid X = x)$, where $\theta$ collects all parameters: 
\[
\theta := \bigl(\{\vec{\beta}_{i}\}_{i=0}^K,\; \{\vec{\beta}_I\}_{I\subseteq[K],\,|I|\ge2}\bigr),
\]
and $\vec{\beta}_i$ and $\vec{\beta}_I$ \cdel{respectively denotes}\cnew{denote} the parameters associated with \cdel{region}\cnew{the regions} $\Gamma_i$ and $\Gamma_I^{\cup}$\cnew{, respectively}; we describe the model in detail in the following. }

\crevb{Specifically, we distinguish three cases, depending on how many rules cover $x$. Notably, throughout this section, we only consider regions induced by a rule with positive probability measure, as a rule that in practice will not cover any instances is not interesting.}

\paragraph{Covered by one rule only.} If $x$ is covered by a single rule $S_i$ only, that is,
\[
    x \in \Gamma_i \quad \text{and} \quad x \notin \Gamma_j, \;\; \forall j \neq i,
\]
we define
\begin{equation}
    P_\theta(Y = y \mid X = x)
    := P_{\vec{\beta}_{i}}(Y = y \mid X \in \Gamma_i),
\end{equation}
and we can write 
\[P_{\vec{\beta}_{i}}(Y = y \mid X \in \Gamma_i)= P_{\vec{\beta}_{i}}(Y = y)\] 
\cdel{in short}\cnew{for short} as the left-hand side does not depend on the specific value of $x$ as long as $x \in \Gamma_i$. 

\paragraph{Covered by multiple rules.} Consider the case where $x$ is covered by multiple rules $\{S_i\}_{i \in I}$ and not by any other rules, where the non-empty index set $I \subseteq [K]$ has $|I| \ge 2$; i.e., $x \in \Gamma_I^{\cap}$, and $\forall j \notin I, x \notin \Gamma_j$. We can then define 
the conditional distribution model of $Y$ as
\begin{equation}
    P_\theta(Y = y \mid X = x)
    := P_{\vec{\beta}_I}(Y = y \mid X \in \Gamma_I^{\cup});
\end{equation}
again, we can write 
\[P_{\vec{\beta}_I}(Y = y \mid X \in \Gamma_I^{\cup})= P_{\vec{\beta}_I}(Y = y)\] 
for short as above. 

To interpret the model definition: if $x$ is only covered by $\{S_i\}_{i \in I}$, instead of modeling the conditional distribution conditioned on the intersection $\Gamma_I^{\cap}$, we condition on the union $\Gamma_I^{\cup}$. \lincen{In doing so, we enrich the probabilistic semantics of rule sets by assigning the `logical OR' semantics to overlaps. We further explain this design choice in Section~\ref{subsec:semantic_overlap}, where we also discuss its advantages over conditioning on the intersection $\Gamma_I^{\cap}$. }

Further, unlike previous methods, we do not need any post-hoc methods to resolve the instances in the overlaps which may result in the interpretability issues that we described in the Introduction. 

Last, we will formally show in Section~\ref{subsec:formal_justification} that this ``union-for-the-intersection'' construction encourages overlap consistency, i.e., overlapping rules having similar probabilistic outputs. Thus, when an instance is covered by multiple rules yet their probabilistic outputs are similar, each rule independently remains an approximately faithful explanation \cdel{to}\cnew{of} the probabilistic prediction of that instance.

\paragraph{Covered by no rule.} If $x$ is not covered by any rule, that is, $x \in \Gamma_0$, we use the else rule:
\begin{equation}
    P_\theta(Y = y \mid X = x)
    := P_{\vec{\beta}_0}(Y = y \mid X \in \Gamma_0) = P_{\vec{\beta}_0}(Y=y),
    \qquad \forall x \in \Gamma_0,
\end{equation}
where $\cfix{\vec{\beta}_0}$ parameterizes the categorical distribution associated with the `else region'.

\input{fig_running_example}
\smallskip
\noindent\textbf{An illustrative example.}
\crevb{Figure~\ref{fig:running_example} illustrates the three cases with a two-dimensional example: two numeric features $X = (X_1, X_2)$ with $\mathcal{X}=[0,1]^2$, a binary target $Y$, and the rule set $M = \{S_1, S_2\}$ with
\[
S_1:\; X_1 < 0.6 \;\longrightarrow\; P_{\vec{\beta}_{1}}(Y),
\qquad
S_2:\; X_2 < 0.5 \;\longrightarrow\; P_{\vec{\beta}_{2}}(Y).
\]
An instance covered by $S_1$ only is modeled by the distribution conditioned on region $\Gamma_1$ (panel a), and an instance covered by no rule is modeled through the else region $\Gamma_0$ (panel c). An instance in the overlap of $S_1$ and $S_2$ is modeled by $P_{\vec{\beta}_{\{1,2\}}}(Y \mid X \in \Gamma_1 \cup \Gamma_2)$, the distribution conditioned on the \emph{union} (panel b), rather than by an estimate on the (often small) overlap itself or by a choice between the two rules. }

\subsubsection{Global parameterization and the model in one formula}
For every $x \in \mathcal{X}$, define the active rule index set as
\[
I(x) := \{i \in \{0, 1, ..., K\} : x \in \Gamma_i\},
\]
and the assigned region as
\begin{equation}\label{eq:Gx}
G(x) :=
\begin{cases}
\Gamma_0, & \text{if } I(x) = \{0\} \quad \text{(the `else' region),} \\[2pt]
\Gamma_i, & \text{if } I(x) = \{i\}, \,\, i \neq 0 \quad \text{(covered by one rule only),}\\[2pt]
\Gamma_{I(x)}^{\cup}, & \text{if } |I(x)| \ge 2 \quad \text{(covered by multiple rules)}.
\end{cases}
\end{equation}

Recall that $\theta$ collects all model parameters $\theta = \bigl(\{\vec{\beta}_{i}\}_{i=0}^K,\; \{\vec{\beta}_I\}_{I\subseteq[K],\,|I|\ge2}\bigr)$. We have now defined that $\vec{\beta}_{i}$ parameterizes the categorical distribution associated with region $\Gamma_i$ (including the else region $\Gamma_0$), and $\vec{\beta}_I$ parameterizes the distribution associated with the union region $\Gamma_I^\cup$.

In principle, $\theta$ contains a parameter vector $\vec{\beta}_I$ for every subset $I\subseteq[K]$ with $|I|\ge2$, which is exponential in $K$. 
In practice, however, only overlap configurations with at least one \cdel{instance in the training procedure}\cnew{training instance} require parameterization. 
Consequently, the number of active overlap parameters is typically small, yielding a sparse effective model.

Thus, we can let $\mathscr{G}$ denote the collection of regions for which the model specifies a categorical distribution:
\[
\mathscr{G}
:=
\{\Gamma_0,\Gamma_1,\ldots,\Gamma_K\}
\;\cup\;
\{\Gamma_I^\cup : I\subseteq[K],\,|I|\ge2,\,\Gamma_I^\cap\neq\emptyset\}.
\]
\crevb{Each instance $(x,y)$ is associated with exactly one region in $\mathscr{G}$}, namely the assigned region $G(x)$ defined in Equation~(\ref{eq:Gx}).

The model can then be compactly written as
\[
P_\theta(Y = y \mid X = x) := P_{\vec{\beta}_{I(x)}}(Y = y \mid X \in G(x)),
\]
which coincides with the three cases above (here we slightly abuse the notation by defining $\vec{\beta}_{i} = \vec{\beta}_{\{i\}}$). 
Notably, the single-rule case is in particular the instantiation for $I(x) = \{i\}$, for which $\Gamma_{\{i\}}^{\cup} = \Gamma_{\{i\}}^{\cap} = \Gamma_i$. 

\subsubsection{Likelihood and parameter estimation}\label{subsubsec:parameter_estimation}

\smallskip
\noindent\textbf{Likelihood.}
Assuming the data $D=(x^n,y^n)$ to be i.i.d., a rule set $M = \{S_1, ..., S_K\}$ with these $K$ fixed rules induces the conditional model
\[
P_{M,\theta}(y^n\mid x^n)
=
\prod_{G\in\mathscr G} \,\,
\prod_{(x,y):\,\crevb{G(x) = G}}
P_{\vec{\beta}_{I(x)}}(Y=y\mid X\in G),
\]
where each $P_{\vec{\beta}_{I(x)}}(Y=y\mid X\in G)$ is a categorical distribution parameterized by either $\vec{\beta}_{i}$ or $\vec{\beta}_I$. We explicitly include $M$ in the notation when denoting the model extended to $n$ samples, to emphasize that the conditional probability depends on the \cdel{collections}\cnew{collection} of rules. 

\paragraph{Parameter estimation.}
All categorical parameters are estimated as the empirical class proportions within their associated regions. That is, for all $\vec{\beta}_i, i \in \{0, ..., K\}$, we estimate \cnew{it }by the empirical class proportions obtained from the instances in $\Gamma_i$; for all $\vec{\beta}_I, I \subseteq [K]\cfix{,\, |I| \ge 2}$, we estimate it by the empirical class proportions obtained from the instances in $\Gamma_I^\cup$. Therefore, we write
\begin{equation} \label{eq:mle_1}
    \hat{{\beta}}_{i, k} = \frac{|\{(x, y) \in D: x \in \Gamma_i, y =k\}|}{|\{(x, y) \in D: x \in \Gamma_i\}|}, \quad \hat{\vec{\beta}}_{i} = \{\hat{{\beta}}_{i, k}\}_{k \in [ |\mathcal{Y}|]};
\end{equation}

\begin{equation}\label{eq:mle_2}
    \hat{{\beta}}_{I, k} = \frac{|\{(x, y) \in D: x \in \Gamma_I^\cup, y =k\}|}{|\{(x, y) \in D: x \in \Gamma_I^\cup\}|}, \quad \hat{\vec{\beta}}_{I} = \{\hat{{\beta}}_{I, k}\}_{k \in [ |\mathcal{Y}|]}.
\end{equation}
Notably, Equation~(\ref{eq:mle_1}) is a special case of Equation~(\ref{eq:mle_2}), yet we write it explicitly to distinguish the two cases of the instance being covered by one rule or multiple rules. 

We therefore write the estimated probability distribution as $P_{\hat{\vec{\beta}}_{i}}$ and $P_{\hat{\vec{\beta}}_{I}}$; when it is clear from the context, we also abbreviate it as $\hat{P}$.

\cfix{These estimators follow directly from what each parameter represents in the TURS model. Recall that, by the model definition in Section~\ref{subsubsec:local_models}, $\vec{\beta}_i$ parameterizes the class distribution associated with the region $\Gamma_i$, i.e., it models $\Ptrue(Y \mid X \in \Gamma_i)$, and $\vec{\beta}_I$ parameterizes the class distribution associated with the union $\Gamma_I^{\cup}$, i.e., it models $\Ptrue(Y \mid X \in \Gamma_I^{\cup})$. Hence, given a fixed rule set and a fixed training set, parameter estimation amounts to estimating these region-conditional class distributions, and the empirical class proportions among all training instances in a region are the maximum likelihood estimator of the categorical distribution of $Y$ conditioned on that region. As a consequence, every training instance in $\Gamma_i$ contributes to $\hat{\vec{\beta}}_i$, regardless of whether it is also covered by other rules, which is consistent with the rules being self-standing. Since each parameter is tied to its own region in this way, the likelihood above serves to evaluate how well a rule set fits the data under these estimates, and we do not treat $\theta$ as a free parameter vector to be optimized jointly; the two views coincide when no two rules cover a common training instance.}

\subsubsection{Prediction without training data access}
When a new instance $x'$ arrives, we determine its active set $I(x')$ and
predict the class probabilities using the associated
$P_{\hat{\vec{\beta}}_{I(x')}}$. 

However, one remaining technical detail is that $x'$
may be covered by a combination of rules that never co-occurred on any
training instance, in which case $\hat{\vec{\beta}}_{I(x')}$ was not
estimated during training. We next show that such predictions can still be
made without access to the training instances, provided a small set of
statistics is recorded during training.

\smallskip
\noindent\textbf{What to record during training.}
For each \emph{realized configuration}, i.e., each active set
$J$ \cfix{(either $J = \{0\}$ or a non-empty $J \subseteq [K]$)} shared by at least one training instance, we record its
per-class instance counts
\[
n_{J,k} := |\{(x,y)\in D:\, I(x)=J,\; y=k\}|, \quad k \in [|\mathcal{Y}|].
\]
Since the realized configurations partition the training instances, these
counts are simply the class counts within each cell of that partition. All
estimated parameters in Equations~(\ref{eq:mle_1})
and~(\ref{eq:mle_2}) are recoverable from these counts, as we illustrate
next.

\smallskip
\noindent\textbf{Two overlapping rules.}
Suppose $x'$ is covered by exactly two rules $S_i$ and $S_j$, whose
coverages were disjoint on the training sample, i.e.,
$c(S_i)\cap c(S_j)=\emptyset$, so that the configuration $\{i,j\}$ was never
realized. The estimate in Equation~(\ref{eq:mle_2}) prescribes the empirical
class proportions on $\Gamma_i \cup \Gamma_j$; as the training instances in
this union are exactly those covered by $S_i$ plus those covered by
$S_j$\cfix{, with no instance counted twice}, this yields the coverage-weighted average
\begin{equation}\label{eq:pred_two_rules}
\hat P(Y\mid X\in \Gamma_i\cup \Gamma_j)
=
\frac{|c(S_i)|\,\hat P(Y\mid X\in \Gamma_i)
     +|c(S_j)|\,\hat P(Y\mid X\in \Gamma_j)}
     {|c(S_i)|+|c(S_j)|},
\end{equation}
which is computable from the recorded counts alone.

\smallskip
\noindent\textbf{Three or more overlapping rules.}
Suppose $x'$ is covered by three rules $S_i$, $S_j$, and $S_l$, of which
only the pair $\{i,j\}$ overlapped during training. The two-rule case above
extends naturally by treating the overlapping pair as a single hypothetical
rule $S_{ij}$, with region $\Gamma_i \cup \Gamma_j$, estimated class
probabilities $\hat{\vec{\beta}}_{\{i,j\}}$, and coverage
$|c(S_i) \cup c(S_j)|$---all of which are available from training, as the
configuration $\{i,j\}$ was realized. Since the coverages of $S_{ij}$ and
$S_l$ are disjoint on the training sample, applying
Equation~(\ref{eq:pred_two_rules}) to the pair $(S_{ij}, S_l)$ yields
$\hat{\vec{\beta}}_{\{i,j,l\}}$, again without access to the training
instances.

\cdel{Notably, this argument naturally extends to any general case; thus, we conclude that we do not need access to the training set during prediction even if $x'$ is covered by a set of rules which never simultaneously cover any training instance.}

\smallskip
\noindent\textbf{The general case.}
In general, a closed formula is available regardless of how the
realized configurations overlap: for any \cfix{non-empty} $I \subseteq [K]$ (whether it was realized
during training or not)\cdel{:}\cnew{,} 
\begin{equation}\label{eq:pred_general}
\hat{\beta}_{I,k}
=
\frac{\sum_{J:\, J\cap I \neq \emptyset} n_{J,k}}
     {\sum_{J:\, J\cap I \neq \emptyset} \sum_{k'} n_{J,k'}}; 
\end{equation}
that is, since each instance in the training set has one unique active set, we can sum over the active \cdel{set}\cnew{sets} $J$. Further, for any training instance with the active set $J$, it lies in
$\Gamma_I^{\cup}$ if and only if $J$ satisfies
$J \cap I \neq \emptyset$. Thus, summing $n_{J, k}$ over $J$ with this condition counts the total number of training instances inside $\Gamma_I^{\cup}$ with class label $k$; as a consequence, the denominator is simply the total number of instances inside $\Gamma_I^{\cup}$. \cnew{Thus, we do not need access to the training set during prediction, even if $x'$ is covered by a set of rules that never simultaneously cover any training instance.}

\subsection{Semantic meaning and robustness of overlaps} \label{subsec:semantic_overlap}

We formulate the probabilistic model class for rule sets such that overlaps are modeled through the union rather than the intersection of the involved rule regions. While this choice differs from more conventional overlap treatments, it provides overlaps with a clear semantic interpretation and yields a more expressive and robust class of probabilistic rule models.

\paragraph{Logical OR semantics.}
Under our formulation, an instance that satisfies multiple rules is \emph{not} assigned to a single ``winning'' rule. All applicable rules are treated as simultaneously valid\crevb{, i.e., all covering rules are applied}. Conditioning on $\Gamma_I^{\cup}$ whenever $x \in \Gamma_I^{\cap}$ therefore gives the overlap a direct and explicit logical interpretation: prediction is made under the \emph{disjunction} (logical `OR') of the rule conditions, corresponding to the event $\{X \in \Gamma_I^{\cup}\}$ in the probability space. Hence, from the logical perspective, the overlap represents a region with higher ambiguity in comparison to the region described by a single rule. Further, taking the union in the probabilistic modeling makes such ambiguity concrete, as the union \cfix{$\Gamma_I^{\cup}$} is a superset of any single $\Gamma_i, i\in I$ and may contain more `impurity'. 

Notably, our model formulation never involves the conditional class probability conditioned on the intersection of overlapping rules---any form of $P(Y|X \in \Gamma_I^{\cap})$ is \emph{not} part of our model.

\paragraph{Why not the intersection?}
A natural alternative is to define the conditional distribution on the intersection $\Gamma_I^{\cap}$. However, this leads to four conceptual and practical drawbacks. First, despite the different notation, this effectively remains a disjoint modeling approach: the model class coincides with that of rule lists or implicitly ordered rule sets, since the intersection introduces an additional region that is separate from the two overlapping rules. 

Second, individual rules are no longer self-standing. The overlap becomes a sub-region within the region described by each rule, so the rule condition no longer faithfully describes the region for which its probability distribution is actually used. 

Third, intersections are typically much smaller than unions and therefore poorly supported by data, which encourages very small overlap regions and leads to unreliable estimates that are prone to overfitting; for instance, in Figure~\ref{fig:running_example}(b), the overlap is far smaller than the union of the two rule regions.

Fourth and finally, at prediction time an unseen instance may lie in an overlap that never occurred in the training data, in which case an intersection-based model becomes undefined or relies on ad-hoc fallback \cdel{behaviour}\cnew{behavior}.

By contrast, conditioning on $\Gamma_I^{\cup}$ preserves the self-standing nature of rules, avoids brittle micro-partitions, yields estimates that are supported by larger effective sample sizes, and generalizes naturally to previously unseen overlap configurations. The union-based semantics is therefore not an ad-hoc workaround, but a principled and robust way of providing overlaps with a meaningful probabilistic interpretation.

\subsection{A formal justification of our model formulation}
\label{subsec:formal_justification}

We develop population-level theory to justify that our overlap formulation 1) increases the
population conditional \cdel{KL}\cnew{Kullback--Leibler (KL) divergence} relative to an oracle partition model unless the overlapping rules fully agree, and 2) yields a quantitative
excess-KL lower bound when rules disagree and the overlap has non-negligible probability mass. \crevb{That is, the theory establishes that, at the population level, the proposed probabilistic formulation inherently encourages overlap consistency by incurring additional loss whenever overlapping rules disagree.}

\paragraph{Setup.}
Throughout this subsection we work with the true joint distribution $\Ptrue$ of $(X,Y)$, where
$Y$ takes values in a finite alphabet $\mathcal Y$ (binary is a special case).
We consider a rule set $M=\{S_1,S_2\}$ with associated \cnew{regions }$(\Gamma_1, \Gamma_2)$.
We assume $\Gamma_1\cup\Gamma_2=\mathcal X$ (no `else rule', only for simplicity in this section).

Define the three induced regions and their union:
\[
A:=\Gamma_1\setminus \Gamma_2,\qquad
B:=\Gamma_2\setminus \Gamma_1,\qquad
C:=\Gamma_1\cap \Gamma_2,\qquad
U:=\Gamma_1\cup\Gamma_2.
\]
Assume $\Ptrue(X\in A),\Ptrue(X\in B),\Ptrue(X\in C)>0$ so all conditional distributions below are well-defined.
\crev{Throughout, we write $R \,\dot\cup\, R'$ for the union of two \emph{disjoint} sets $R$ and $R'$; e.g., $\Gamma_1 = A \,\dot\cup\, C$ and $U = A \,\dot\cup\, B \,\dot\cup\, C$.}
\crevb{In the running example of Figure~\ref{fig:running_example}, $A$ and $B$ are the parts of $\Gamma_1$ and $\Gamma_2$ covered by only one rule, $C$ is their overlap (see panel (b)), and $U$ is the union.}
For any measurable $R\subseteq \mathcal X$ with $\Ptrue(X\in R)>0$, define the region-conditional label law
\[
p_R(\cdot)\ :=\ \Ptrue(Y=\cdot\mid X\in R).
\]
For each $x\in\mathcal X$, write the pointwise conditional label law
\[p_x(\cdot):=\Ptrue(Y=\cdot\mid X=x).\]
\cnew{Further, we use $H(\cdot)$ to denote the Shannon entropy and $D_{\mathrm{KL}}(\cdot\,\|\,\cdot)$ to denote the KL divergence.}

We begin by making explicit the population-level predictive behavior induced by a given rule set.
At this stage, we do not consider estimation or finite-sample effects; instead, we ask a purely conceptual question: given the true distribution 
$\Ptrue$, what conditional distribution does the TURS formulation assign to each region of the instance space? This definition makes precise how overlap is resolved at the population level and serves as the reference point for all subsequent risk decompositions.

\begin{definition}[Population-level TURS model (two-rule case)]
\label{def:pop_turs}
The population-level TURS model $Q_M(\cdot\mid X)$ induced by $\Ptrue$ is
\[
Q_M(\cdot\mid X)=
\begin{cases}
p_{\Gamma_1}(\cdot), & X\in A,\\
p_{\Gamma_2}(\cdot), & X\in B,\\
p_{U}(\cdot), & X\in C.
\end{cases}
\]
\end{definition}
Notably, we use $Q_M$, a different notation \cdel{than}\cnew{from} $P_\theta$ in Section~\ref{subsec:prob_formulation}, to emphasize that the analysis is based on the population-level distribution instead of the proposed model parameterized by $\theta$. 

Intuitively, outside the overlap, the model predicts according to the region-conditional distribution; on the other hand, within the overlap, the model predicts the union-conditional distribution, which aggregates information from both rules. The purpose of the following analysis is to understand the statistical consequences of this aggregation.

To evaluate different rule sets at the population level, we measure performance using the standard expected log-loss under the true distribution. 
\begin{definition}[Population risk (cross-entropy / conditional KL)]
\label{def:pop_risk}
Given $M$ and its induced $Q_M$, define the population cross-entropy risk
\[
\cfix{\poprisk}(M)\ :=\ \mathbb E_{(X,Y)\sim \Ptrue}\big[-\log Q_M(Y\mid X)\big].
\]
Then
\[
\cfix{\poprisk}(M)
=
\mathbb E_{X\sim \Ptrue}\Big[H\big(p_X\big)\Big]
+\mathbb E_{X\sim \Ptrue}\Big[D_{\mathrm{KL}}\big(p_X\ \|\ Q_M(\cdot\mid X)\big)\Big],
\]
so minimizing $\cfix{\poprisk}(M)$ over $M$ is equivalent to minimizing the conditional KL term.
\end{definition}

We now analyze how overlap affects the conditional KL term in Definition \ref{def:pop_risk}.
The goal is to make precise, at the population level, when and why resolving overlap by aggregation incurs an unavoidable excess KL penalty.
We proceed in three steps.
First, Lemma \ref{lem:kl_bias_variance} provides a basic identity that decomposes the conditional KL on a region into a local variability term and a mismatch term with respect to a fixed predictive distribution.
Second, Theorem \ref{thm:overlap_harms_full_rigorous} applies this identity to show that overlap can never reduce the population conditional KL relative to an oracle partition model, and yields an exact decomposition of the excess risk induced by overlap.
Finally, Theorem \ref{thm:disagreement_excessKL_crisp_fixed} gives a quantitative consequence: when overlapping rules disagree and the overlap has non-negligible probability mass, the induced excess KL is bounded away from zero under mild local homogeneity assumptions.
\crevb{To keep the presentation compact, all proofs of this section are deferred to Appendix~\ref{app:proofs}.}

Specifically, to analyze how the choice of predictive distribution affects the population KL,
it is useful to separate two sources of error on a region $R$. First, even if we use the optimal region-level predictor $p_R(\cdot)=P(Y\mid X\in R)$, there remains an unavoidable loss due to variability of the true conditional distribution $p_X$ within the region. Second, using a different predictive distribution $q$ introduces an additional mismatch term measuring how far $q$ is from $p_R$.

The following identity makes this separation explicit. It will later allow us to isolate the extra KL incurred when the TURS model
predicts with $p_{\Gamma_1}$, $p_{\Gamma_2}$, or $p_U$ instead of the region-specific distributions $p_A$, $p_B$, and $p_C$.

\begin{lemma}[KL identity on a region]
\label{lem:kl_bias_variance}
Let $R\subseteq\mathcal X$ with $\Ptrue(X\in R)>0$. For any fixed distribution $q$ on $\mathcal Y$,
\[
\mathbb E\!\left[D_{\mathrm{KL}}(p_X\|q)\,\middle|\,X\in R\right]
=
\mathbb E\!\left[D_{\mathrm{KL}}(p_X\|p_R)\,\middle|\,X\in R\right]
+
D_{\mathrm{KL}}(p_R\|q).
\]
\end{lemma}

Lemma~\ref{lem:kl_bias_variance} shows that the expected KL loss on any region can be decomposed into a local variability term and a mismatch term between the predictive distribution and the region-average distribution. We now apply this identity to the regions induced by the rule set.
Comparing the TURS predictor with an oracle partition model allows us to isolate exactly where additional KL arises when overlap is resolved through the union-based prediction.

\begin{theorem}[Overlap harms the population KL]
\label{thm:overlap_harms_full_rigorous}
Define $A,B,C,U$ as above and $p_R(\cdot)=\Ptrue(Y=\cdot\mid X\in R)$.
Let $Q_M$ be as in Definition~\ref{def:pop_turs}, and define the (partition) oracle
\[
Q_{\mathrm{part}}(\cdot\mid X)=
\begin{cases}
p_A(\cdot), & X\in A,\\
p_B(\cdot), & X\in B,\\
p_C(\cdot), & X\in C.
\end{cases}
\]
Then the population conditional KL risk admits the exact decomposition
\[
\begin{aligned}
\mathbb{E}_{X\sim \Ptrue}\!\Big[
D_{\mathrm{KL}}\!\big(p_X\,\|\,Q_M(\cdot\mid X)\big)
\Big]
&=
\mathbb{E}_{X\sim \Ptrue}\!\Big[
D_{\mathrm{KL}}\!\big(p_X\,\|\,Q_{\mathrm{part}}(\cdot\mid X)\big)
\Big] \\
&\quad
+\, \Ptrue(X\in A)\,
D_{\mathrm{KL}}\!\big(
p_A\,\|\,p_{\Gamma_1}
\big) \\
&\quad
+\, \Ptrue(X\in B)\,
D_{\mathrm{KL}}\!\big(
p_B\,\|\,p_{\Gamma_2}
\big) \\
&\quad
+\, \Ptrue(X\in C)\,
D_{\mathrm{KL}}\!\big(
p_C\,\|\,p_{U}
\big).
\end{aligned}
\]
In particular,
\[
\mathbb{E}_{X\sim \Ptrue}\!\Big[
D_{\mathrm{KL}}\!\big(p_X\,\|\,Q_M(\cdot\mid X)\big)
\Big]
\ \ge\
\mathbb{E}_{X\sim \Ptrue}\!\Big[
D_{\mathrm{KL}}\!\big(p_X\,\|\,Q_{\mathrm{part}}(\cdot\mid X)\big)
\Big].
\]
\end{theorem}

Theorem~\ref{thm:overlap_harms_full_rigorous} shows that overlap can only increase the population conditional KL relative to the oracle partition model. A natural question is when this additional cost disappears. The following corollary shows that equality can hold only in the degenerate case where the conditional label distributions in all involved regions coincide, which implies that the two rules are probabilistically identical.

\begin{corollary}[Equality forces full agreement]
\label{cor:full_agreement}
Assume $\Ptrue(X\in A),\Ptrue(X\in B),\Ptrue(X\in C)>0$.
If equality holds in Theorem~\ref{thm:overlap_harms_full_rigorous}, then
\[
p_A=p_B=p_C,
\qquad\text{and hence}\qquad
p_{\Gamma_1}=p_{\Gamma_2}.
\]
\end{corollary}

We next quantify how disagreement between $p_{\Gamma_1}$ and $p_{\Gamma_2}$ induces a positive KL contribution through the overlap region.

\smallskip \noindent \textbf{Intuition.}
Under the TURS formulation, instances in the overlap $C=\Gamma_1\cap\Gamma_2$ are predicted using the union-conditional distribution
$p_U = P(Y \mid X \in U)$, where $U=\Gamma_1\cup\Gamma_2$.
If the two rules are probabilistically consistent, i.e., $p_{\Gamma_1}\approx p_{\Gamma_2}$, then $p_U$ is close to either rule’s conditional distribution and therefore also close to the true overlap law $p_C$.
In this regime, resolving the overlap by aggregation is essentially harmless.

The situation changes when the two rules disagree.
Because $p_U$ aggregates information from both rule regions, disagreement between $p_{\Gamma_1}$ and $p_{\Gamma_2}$ propagates into the overlap prediction.
Consequently, if the overlap has non-negligible probability mass and the rules are locally homogeneous, the union-based prediction cannot simultaneously match the conditional distribution on all involved regions.

\paragraph{Implication.}
The theorem below makes this intuition precise.
When overlapping rules make incompatible probabilistic predictions, resolving the overlap through the union distribution $p_U$ necessarily incurs a strictly positive excess conditional KL risk.
The magnitude of this penalty depends on three factors: the size of the overlap, the degree of rule disagreement, and the local homogeneity of the rule regions.
In particular, the penalty vanishes when the rules are sufficiently consistent or the overlap becomes negligible.

This result provides a formal justification for our modeling choice.
Because disagreement in overlapping regions leads to unavoidable probabilistic loss, the likelihood induced by the TURS formulation naturally favors rule sets whose overlapping rules produce similar conditional distributions.
In this sense, the proposed probabilistic semantics acts as an implicit \emph{consistency regularizer}: overlaps are tolerated when rules agree, but become costly when they encode incompatible predictions.

\begin{theorem}[Rule disagreement yields excess KL with non-negligible overlap]
\label{thm:disagreement_excessKL_crisp_fixed}
Let $A,B,C,U$ be as above, and let $Q_M$ be the population-level TURS model in
Definition~\ref{def:pop_turs}. Assume:

\begin{enumerate}
\item[(H1)] (Non-negligible overlap mass) $\Ptrue(X\in C)\ge \cfix{\ovmass}$ for some $\cfix{\ovmass}>0$.
\item[(H2)] (Union mixes in region $B$) $\Ptrue(X\in B\mid X\in U)\ge \eta$ for some $\eta\in(0,1]$.
\item[(H3)] (Overlap is close to rule 1) $D_{\mathrm{KL}}(p_C\|p_{\Gamma_1})\le \alpha_C$.
\item[(H4)] (Rule disagreement on one label) There exists $y^\star\in\mathcal Y$ such that
\[
\bigl|\log p_{\Gamma_1}(y^\star)-\log p_{\Gamma_2}(y^\star)\bigr|\ge \sigma,
\qquad
m:=\min\{p_{\Gamma_1}(y^\star),p_{\Gamma_2}(y^\star)\}>0.
\]
\item[(H5)] (Rule 2 is locally not too heterogeneous on $B$) $D_{\mathrm{KL}}(p_B\|p_{\Gamma_2})\le \alpha_2$.
\item[(H6)] (Rule 1 is locally not too heterogeneous on $A$) $D_{\mathrm{KL}}(p_A\|p_{\Gamma_1})\le \alpha_1$.
\end{enumerate}

Then the \cfix{excess} population conditional KL risk of $Q_M$ \cfix{over the oracle partition model $Q_{\mathrm{part}}$ of Theorem~\ref{thm:overlap_harms_full_rigorous}} satisfies
\[
\mathbb E_{X\sim \Ptrue}\!\left[D_{\mathrm{KL}}\!\bigl(p_X\,\|\,Q_M(\cdot\mid X)\bigr)\right]
\cfix{\;-\;
\mathbb E_{X\sim \Ptrue}\!\left[D_{\mathrm{KL}}\!\bigl(p_X\,\|\,Q_{\mathrm{part}}(\cdot\mid X)\bigr)\right]}
\;\ge\;
\frac{1}{2}\cfix{\ovmass}\left(
\eta\,\Delta
-(1-\eta)\,\delta_A
\right)_+^2,
\]
where
\[
\Delta
:= m\,(e^\sigma-1)\;-\;\sqrt{2{\alpha_2}}\;-\;\sqrt{2{\alpha_C}},
\qquad
\delta_A
:= \sqrt{2{\alpha_1}}+\sqrt{2{\alpha_C}}\cfix{,}
\]
\cfix{and $(t)_+ := \max\{t, 0\}$. In particular, since the conditional KL risk of $Q_{\mathrm{part}}$ is nonnegative, the same lower bound holds for the population conditional KL risk of $Q_M$ itself.}
\end{theorem}

\smallskip
\noindent \textbf{Interpretations.}
The assumptions of the theorem have a natural statistical interpretation. \\
(H1) requires that the overlap region is not negligible, so its contribution to the population risk cannot be ignored. \\
(H2) ensures that the \cdel{overlap}\cnew{union} distribution $p_U$ meaningfully mixes information from at least one side (here, rule~2), so that disagreement can propagate into the overlap. \\
(H3), (H5), and (H6) are local homogeneity assumptions: they state that, within each rule region and near the overlap, the conditional label distribution does not vary too sharply, which is precisely the regime in which rule-based models are intended to operate. \\
Finally, (H4) captures genuine rule disagreement by requiring that the two rules make substantially different predictions on at least one label. \\
Under these conditions, the theorem shows that overlap cannot ``hide'' disagreement for free. If two rules disagree \cfix{sufficiently strongly relative to the local heterogeneity of their regions, i.e., $\eta\,\Delta > (1-\eta)\,\delta_A$,} and the overlap has non-negligible probability mass, then any model that resolves overlaps by predicting the union-conditional distribution $p_U$ must incur \cfix{an excess conditional KL risk that is bounded away from zero}. The lower bound makes this trade-off explicit: \cfix{it} grows with the size of the overlap and the degree of rule disagreement, and \cfix{becomes vacuous when the overlap is negligible, when the rules (approximately) agree, or when the disagreement is dominated by the local heterogeneity of the rule regions}. In this sense, overlap acts as a consistency-enforcing mechanism rather than a free smoothing device.

\smallskip
\noindent\textbf{Symmetric variant.}
If instead we assume $\Ptrue(X\in A\mid X\in U)\ge\eta$ and swap the roles of $(1,A,\alpha_1)$ and $(2,B,\alpha_2)$,
the same conclusion holds with indices $1$ and $2$ swapped.

\paragraph{Implications for rule sets with more than two rules.}
Although Theorem~\ref{thm:disagreement_excessKL_crisp_fixed} is stated for two rules, its implications extend naturally to rule sets with more than two rules.
In a general rule set, overlaps arise locally between subsets of rules, and the population risk decomposes additively over disjoint regions of the feature space.
The theorem therefore characterizes the unavoidable excess KL incurred by any such local overlap when the corresponding rules disagree and the overlap has non-negligible mass.
In this sense, the result should be interpreted as a local consistency principle: whenever two rules overlap and make incompatible predictions, the model must pay a positive KL penalty on \cfix{the regions involved in} that overlap, independent of how the rest of the rule set behaves.
Aggregating these local penalties across all overlapping regions provides a qualitative explanation for why overlap tends to be discouraged unless rules agree, even in more complex rule systems.

\section{Rule Set Learning as Probabilistic Model Selection}
\label{sec:model_selection}

Exploiting the formulation of rule sets as probabilistic models, we define the task of learning a rule set as a probabilistic model selection problem. Specifically, we use the minimum description length (MDL) principle for model selection. 

The MDL principle is one of the best off-the-shelf model selection methods and has been widely used in machine learning and data mining \citep{grunwald2019minimum, galbrun2022minimum}. Although rooted in information theory, it has been shown that MDL-based model selection can be regarded as an extension of Bayesian model selection \citep{grunwald2019minimum}. 

The principle of MDL-based model selection is to pick the model, such that the code length (in bits) needed to encode the data given the model, together with the model itself, is minimized. We begin by discussing Normalized Maximum Likelihood (NML) distributions \crev{for calculating the length of the code required to encode the data given the model. We then discuss the length of the code required to encode the model. Note that, under the MDL framework, we calculate code lengths only and never need actual codes~\citep{grunwald2007minimum}; consequently, code lengths are not necessarily integer-valued.}

\subsection{Normalized maximum likelihood distributions for rule sets}

The Kraft inequality establishes a connection between code length and probability. Based on this connection, the MDL principle can be viewed as assigning a single probability distribution---called a \emph{universal distribution}---to the data given a model~\citep{grunwald2019minimum}. 

In our setting, the model is a rule set $M$. However, recall that a rule set does not correspond to a single probability distribution but to a family of distributions, denoted by $\{P_{\ruleset,\theta}(y^n \mid x^n)\}_\theta$, where $\theta$ represents the model parameters. Thus, this universal distribution should be intuitively viewed as a `representative' distribution for the family $\{P_{\ruleset,\theta}(y^n \mid x^n)\}_\theta$.

Given the features $X^n=x^n$, we denote the universal distribution given a rule set $M$ by 
\[P_{\ruleset}(Y^n \mid X^n=x^n).\] 
The question now is how to choose such a representative distribution. 

The theoretically optimal choice in the MDL framework is the distribution that minimizes the worst-case \emph{regret}. The regret measures how much longer the code length of a distribution $P_{\ruleset}$ is compared to the best-fitting distribution in the model class for a given dataset.
Formally, this representative distribution is defined as
\begin{equation}
\label{eq:define_regret}
\arg \min_{P_{\ruleset}} \max_{y^n \in \mathcal{Y}^n} [
-\log_2 P_{\ruleset}(Y^n = y^n \mid X^n = x^n)
-
(
-\log_2 P_{M, \htheta(x^n, y^n)}(Y^n = y^n \mid X^n = x^n)
)].
\end{equation}
Here $\htheta(x^n,y^n)$ denotes the parameter estimator obtained from the data \cdel{$(X^n,Y^n)$}\cnew{$(x^n,y^n)$}\cfix{, as defined in Section~\ref{subsubsec:parameter_estimation}}; we write it explicitly as $\htheta(x^n,y^n)$ to emphasize its dependence on the values of both $X^n$ and $Y^n$. \cfix{As discussed there, each component of $\htheta$ is the maximum likelihood estimator of the class distribution of its associated region, and we use `best-fitting' in this sense throughout; the minimax argument below holds for this estimator.}

It has been shown that the unique solution to Equation~(\ref{eq:define_regret}) is the \emph{normalized maximum likelihood} (NML) distribution \citep{grunwald2007minimum}:
\begin{equation}
\label{eq:NML_defi}
P^{NML}_{\ruleset}(Y^n=y^n \mid X^n=x^n)
=
\frac{P_{\ruleset,\htheta(x^n,y^n)}(Y^n=y^n \mid X^n=x^n)}
{\sum_{\cfix{\tilde{y}^n} \in \mathcal{Y}^n} P_{\ruleset,\htheta(x^n,\cfix{\tilde{y}^n})}(Y^n=\cfix{\tilde{y}^n} \mid X^n=x^n)}.
\end{equation}
Intuitively, the NML distribution takes the likelihood of the data under the best-fitting parameters and normalizes it over all possible label sequences $y^n \in \mathcal{Y}^n$. This normalization ensures that the resulting distribution is valid~\citep{grunwald2019minimum}, i.e.,
\[
\sum_{y^n \in \mathcal{Y}^n} P^{NML}_{\ruleset}(Y^n=y^n \mid X^n=x^n)=1.
\]

\subsection{Approximating the NML distribution}

A crucial difficulty in using the NML distribution in practice is the computation of the normalizing term $\sum_{y^n \in \mathcal{Y}^n} P_{\cfix{\ruleset,}\htheta(x^n, y^n)}(Y^n=y^n|X^n=x^n)$. Efficient algorithms exist almost exclusively for exponential family models \citep{grunwald2019minimum}, hence we approximate the term by the product of the normalizing terms for the individual rules. 

\smallskip
\noindent \textbf{NML distribution for a single rule.}
For an individual rule $S \in \ruleset \cfix{{}\cup \{S_0\}}$ \cfix{(i.e., a rule in $\ruleset$ or the else rule)}, we introduce \cdel{an }additional notation and write all instances covered by $S$ as $(x^S, y^S)$, in which $y^S$ can be regarded as a realization of the random vector of length \crev{$|c(S)|$}, denoted as $Y^S$. $Y^S$ takes values in \crev{$\mathcal{Y}^{|c(S)|}$, the $|c(S)|$\cdel{-ary}\cnew{-fold}} Cartesian product of $\mathcal{Y}$. 

As we introduced in Section~\ref{subsubsec:single_rule}, the local model given a fixed single rule can be represented as $P_{\vec{\beta}_S}(Y)$. 
Consequently, following the definition of the NML distribution in Equation~(\ref{eq:NML_defi}), the NML distribution for the family $\{P_{{\vec{\beta}}_S}(Y^S |X^S=x^S)\}_{\vec{\beta}_S}$, which we denote as $P^{NML}_S(Y^S = y^{S}|X^S = x^S)$, equals
\begin{equation} \label{eq:local_nml}
    P^{NML}_S(Y^S = y^{S}|X^S = x^S) = \frac{P_{\hat{\vec{\beta}}_S(y^S)}(Y^S = y^S)}{\sum_{\cfix{\tilde{y}^S} \in \mathcal{Y}^{\crev{|c(S)|}}} P_{\hat{\vec{\beta}}_S(\cfix{\tilde{y}^S})}(Y^S = \cfix{\tilde{y}^S})},
\end{equation}
in which we explicitly write $\hat{\vec{\beta}}_S(y^S)$ to emphasize that the estimator $\hat{\vec{\beta}}_S$ depends on the specific values of $Y^S$. Also note that the specific values of $x^S$ do not affect the value of $P^{NML}_S(Y^S = y^{S}|X^S = x^S)$ as long as all $(x^S, y^S)$ are covered by rule $S$. 

As $P_{{\vec{\beta}}_S}(Y^S |X^S=x^S)$ is a categorical distribution, it has been shown~\citep{mononen:08:sub-lin-stoch-comp} that the normalizing term can be written as \crev{$\mathcal{R}(|c(S)|, |\mathcal{Y}|)$, a function of the rule's coverage $|c(S)|$} and the number of unique values that $Y$ can take, $|\mathcal{Y}|$, which does not depend on the specific values of $y^S$:
\begin{equation}
    \mathcal{R}(\crev{|c(S)|}, |\mathcal{Y}|) = \sum_{y^S \in \mathcal{Y}^{\crev{|c(S)|}}} P_{\hat{\vec{\beta}}_S(y^S)}(Y^S = y^S),
\end{equation}
and it can be efficiently calculated in sub-linear time~\citep{mononen:08:sub-lin-stoch-comp}.

\smallskip
\noindent \textbf{The approximate NML distribution.} We propose to approximate the normalizing term of the NML distribution for rule set model $P^{NML}_{\ruleset}$ as the product of the normalizing terms of $P^{NML}_S$ for all $S \in \ruleset \cfix{{}\cup \{S_0\}}$:
\begin{equation} \label{eq:nml_cl_data}
    P^{apprNML}_{\ruleset}(Y^n =y^n | X^n=x^n) = \frac{P_{\ruleset, \htheta(x^n, y^n)}(Y^n=y^n|X^n=x^n)}{\prod_{S \in \ruleset \cfix{{}\cup \{S_0\}}} \mathcal{R}(\crev{|c(S)|}, |\mathcal{Y}|)}.
\end{equation}
Note that the product \cfix{runs over the $K$ rules in $\ruleset = \{S_1, \ldots, S_K\}$ and, in addition, over} the ``else rule" $S_0$\cfix{, which is not an element of $\ruleset$}.

The rationale \cdel{of}\cnew{for} using the approximate-NML distribution is as follows. First, when the rule set contains no overlaps, the approximation coincides exactly with the NML distribution, as shown below.
\begin{proposition}\label{prop:apprnml_exact}
Given a rule set $\ruleset$ in which for any \cfix{two distinct rules} $S_i, S_j \in \ruleset$ \cfix{($i \neq j$)}, \crev{$c(S_i) \cap c(S_j) = \emptyset$}, then $P^{NML}_{\ruleset}(Y^n=y^n|X^n=x^n) = P^{apprNML}_{\ruleset}(Y^n=y^n|X^n=x^n)$.
\end{proposition}
\crevb{The proof, which follows by factorizing the normalizing term over the rules, is deferred to Appendix~\ref{app:proofs}.}

\noindent Second, when overlaps exist in $\ruleset$, approximate-NML puts a small extra penalty on overlaps, which is desirable to trade off overlap with goodness-of-fit. This effect arises in the approximation of the NML normalizing term: when we sum over all instances covered by each rule $S \in \ruleset$ to compute the normalizing terms, instances that lie in overlap regions contribute to multiple rule coverages and are therefore effectively counted multiple times in the regret term. Third, approximate-NML behaves asymptotically like the Bayesian information criterion (BIC), as shown in the next proposition.
\begin{proposition}\label{prop:apprnml_bic}
\cfix{Let $\ruleset$ contain $K$ rules, not counting the else rule $S_0$.} Under the mild assumption that \crev{$|c(S)|$} grows linearly with the sample size $n$ for all $S \in M \cfix{{}\cup \{S_0\}}$, then $\log \left(\prod_{S \in \ruleset \cfix{{}\cup \{S_0\}}} \mathcal{R}(\crev{|c(S)|}, |\mathcal{Y}|)\right) = \frac{\cfix{(K+1)}(|\mathcal{Y}| - 1)}{2} \log n + \mathcal{O}(1)$, where $\mathcal{O}(1)$ is bounded by a constant with respect to $n$.
\end{proposition}
\crevb{The proof is deferred to Appendix~\ref{app:proofs} as well.}

\subsection{Code length of \cnew{the }model} \label{subsec:code_model}
To obtain the final MDL-based score, we next describe how to calculate the code length of the model, denoted as $L(\ruleset)$. The code length needed to encode the model depends on the encoding scheme we choose. Given \cdel{the }Kraft's inequality~\citep{grunwald2007minimum}, this can be practically treated as defining a prior distribution over the model class. 

We describe the encoding scheme in a hierarchical manner due to the complexity of the model class. Specifically, we describe in order the code length needed to encode the number of rules and each individual rule. Further, for encoding an individual rule, we separately describe the code length needed for encoding 1) the number of literals, 2) which feature variables are used, 3) the ``operator" inside the literal (i.e., ``$\geq$" or \cnew{``}$<$\cnew{"}), and 4) the values corresponding to all feature variables in the literals. 

\subsubsection{Encoding the number of rules}
First, we encode the number of rules in the rule set, for which we use \cdel{the standard }Rissanen's \cdel{integer }universal code\cnew{ for integers}~\citep{rissanen1983universal}. The code length needed for encoding an integer $K$ is equal to 
$$L_{rissanen}(K) = c + \log_2(K) + \log_2(\log_2(K)) + \log_2(\log_2(\log_2(K))) + \ldots ;$$
the summation continues until a certain precision is reached (which we set as $10^{-3}$ in our implementation), and $c \approx \log_2(2.865)$ is a constant. \cfour{This code is defined for positive integers; for the empty rule set ($K = 0$), with which our algorithm is initialized, we set $L_{rissanen}(0) := 0$ by convention.}

\subsubsection{Encoding individual rules}
We next describe the code length needed to encode each individual rule. We start with encoding the feature variables in the rule condition. 

\smallskip
\noindent \textbf{Encoding the feature variables. }
For a given rule with $\cfix{\numlit}$ literals, we first encode \cdel{the }$\cfix{\numlit}$ itself (the number of literals) by a \emph{uniform code}: as $\cfix{\numlit}$'s range is bounded by the number of \cfix{feature variables (i.e., the columns of the dataset excluding the target)}, denoted by $K_{col}$ \cfix{(that is, $K_{col} = d$ in the notation of Section~\ref{subsubsec:single_rule})},  the code length needed to encode $\cfix{\numlit}$ is equal to
\begin{equation}
	L_{num\_literal} = \log_2 K_{col}. 
\end{equation}

As each literal contains one unique variable, given the number of literals $\cfix{\numlit}$, we further specify \cdel{which are these}\cnew{which} $\cfix{\numlit}$ variables \cnew{these are }among all $K_{col}$ variables, again with a uniform code. Thus, the code length needed to specify which these $\cfix{\numlit}$ variables are is equal to
\begin{equation}
	L_{which\_vars} = \log_2 {K_{col} \choose \cfix{\numlit}}.
\end{equation}

\smallskip \noindent
\textbf{Encoding the split values for numeric variables.}
For numeric variables, the literal involves either a single threshold or two thresholds: 1) $X \geq (\text{or} <)$ $v$, and 2) $v_1  \leq X < v_2$. As a result, we first need to encode which form the literal \cdel{is}\cnew{has} (i.e., whether it contains one or two split values), which costs $L_{form} = 1$ bit. 

When the literal contains one split value, we also need $L_{operator} = 1$ bit to specify the operator between `$\geq$' and `$<$'. 

Next, to encode the values $v$ [or $(v_1, v_2)$], we need to know in advance the domain of $v$ [or $(v_1, v_2)$], which are chosen as quantiles in our algorithm implementation. 

\ly{The number of candidate values (quantiles) for each numerical feature is a hyper-parameter, which we argue should be chosen based on the data: it should be large enough to retain sufficient predictive information, while at the same time respecting computational constraints and prior knowledge about what constitutes an interpretable rule in practice.}

\ly{In contrast to regularization hyper-parameters that primarily control model capacity through abstract penalties, this hyper-parameter has a concrete and intuitive meaning. It directly specifies the maximal granularity at which numerical features are discretized, and can therefore be set using domain knowledge about what differences in feature values are practically meaningful. In particular, specifying the number of candidate split values expresses a prior on the finest resolution at which adjacent intervals should still be distinguishable and interpretable by a domain expert. Finally, this hyper-parameter can be tuned if necessary. In cases where the chosen granularity is too fine relative to the available sample size, the effective model complexity may be reduced by selecting a smaller number of candidate values. In our experiments, however, we find that a single, coarse setting of this hyper-parameter across datasets is often sufficient, suggesting that the method is not overly sensitive to its precise value.}

Once we fix the number of candidate cut points (quantiles), we further denote the number of candidate cut points \emph{after excluding those that result in a rule with coverage equal to 0} as $K_{split}$\footnote{This follows from the standard interpretation of encoding as a message transmission process between a sender and a receiver who share knowledge of the model class, the data $x^n$, and the encoding scheme. Once the literals already encoded in the rule are known, both parties can determine which candidate split values would lead to a rule with zero coverage and therefore do not need to be considered in the encoding space.}. 

To put these terms together: depending on whether the literal contains one or two split values, we can encode the operator and value(s) in the literal, denoted as $L_{value\_op}$, as

\begin{equation}
L_{value\_op} = 
\begin{cases}
 L_{form} + L_{operator} + \log_2 {K_{split}}  & \text{if one split value, } \\
  L_{form} + \log_2 {K_{split} \choose 2}  & \text{if two split values}.
\end{cases}
\end{equation}
Note that encoding the operator, i.e., ``$\geq$" or ``$<$", is needed when there is just one single threshold, which costs 1 bit. 

\smallskip \noindent
\textbf{Encoding the split values for categorical variables.}
For categorical variables with $\mathcal{L}$ levels, encoding a subset of $l$ levels requires $L_{value\_op} =\log_2 \mathcal{L} + \log_2 {\mathcal{L} \choose l}$ bits; the former term, $\log_2 \mathcal{L}$, is needed for encoding the number $l$ itself, and the latter one is the code length needed to specify these $l$ levels from $\mathcal{L}$ in total. 

For simplicity, in our implementation we assume all categorical features are one-hot encoded, and hence $L_{value\_op} = 1$. In practice, one-hot variables that represent subsets of categories can be created if needed. 

\smallskip \noindent
\textbf{Code length needed for encoding a single rule.}
Based on the components discussed so far, \crev{the code length needed for encoding} an individual rule $S$, denoted as $L(S)$, is equal to
\begin{equation}
	L(S) = L_{num\_literal} + L_{which\_vars} + \sum L_{value\_op}, 
\end{equation}
in which the term $\sum L_{value\_op}$ denotes the summation of the code length needed to encode the operator and value for each single literal. 

Note that $2^{-L(S)}$ can be interpreted as a prior probability mass for $S$ among all possible individual rules~\citep{grunwald2007minimum}. Moreover, because of the way we determine $K_{split}$ (i.e., by excluding those candidate cut points that lead to rules with coverage equal to 0), the code length needed to encode a single rule \emph{does} depend on the order of encoding each literal in the condition of the rule. This turns out to be desirable because of our algorithmic approach, which will be described in Section~\ref{sec:alg}.

\subsubsection{Encoding the rule set}
Based on the code length needed for single rules, we can now define the code length needed to encode the whole rule set. Given a rule set $\ruleset$ with $K$ rules, \crev{the total length of the code needed to encode $\ruleset$} is
\begin{equation} \label{eq:cl_model}
	L(M) = L_{rissanen}(K) + \sum_{i = 1}^K L(S_i) - \log_2(K!), 
\end{equation}
in which the last term is needed to eliminate the redundancy caused by the fact that the order of the rules in a rule set does not matter. 

To explain the rationale \cdel{of}\cnew{for} introducing the term $-\log_2(K!)$ more precisely, consider the prior probability of each rule denoted as $P(S_i) = 2^{-L(S_i)}$. Then, the prior probability of the set of rules $\{S_1, ..., S_K\}$, conditioned on the fixed $K$, can be defined as 
\begin{equation}
	P(\{S_1, ..., S_K\}) = \sum \prod_{i = 1}^K P(S_i) = (K!) \prod_{i = 1}^K P(S_i), 
\end{equation}
in which the sum goes over all permutations of $\{S_1, ..., S_K\}$. Thus, we have $L(M) = L_{rissanen}(K) - \log_2 P(\{S_1, ..., S_K\})$, which connects the definition of $L(M)$ to the prior probability of $M$ and hence justifies the introduction of the term $-\log_2(K!)$ in Equation~(\ref{eq:cl_model}). 

\smallskip \noindent \textbf{Discussion: no need to encode the class probability explicitly.} The NML encoding we adopt belongs to the so-called one-part MDL framework. In one-part MDL, the parameters (i.e., the class distribution parameters in our formulation) are not encoded explicitly as part of the model description. Instead, they are implicitly accounted for through the regret term of the NML distribution~\citep{grunwald2007minimum}.

\smallskip \noindent \textbf{Discussion: comparison to encoding schemes in previous rule learning methods.} Encoding rule conditions within an MDL framework has been considered in several previous rule learning systems. Early approaches such as RIPPER~\citep{cohen1995ripper} and C4.5~\citep{quinlan1995mdl} employ MDL-inspired encodings primarily as heuristics for pruning rules or trees, rather than as a principled model selection criterion for the final model. Their encoding schemes are largely intuitive and do not explicitly enforce the correspondence between code lengths and prior probabilities implied by Kraft's inequality, which may lead to redundant encoding where the implied prior mass sums to substantially less than one. 

In contrast, in our formulation MDL directly defines the objective used to select the rule set. Our encoding follows the general guidelines of the modern MDL framework, where model codes correspond to probability distributions and hierarchical structures are encoded using non-informative priors at each level~\citep{grunwald2007minimum, grunwald2019minimum}. Similar principles are adopted by more recent MDL-based rule learning approaches~\citep{proencca2020interpretable, proencca2022robust, schuster2024data}, although the specific encoding schemes differ depending on the model class and design choices.

\subsection{MDL-based model selection}
After describing the approximate normalized maximum likelihood distributions and the code length (number of bits) needed to specify a model in the model class, we can now formulate the task of learning truly unordered rule sets as a model selection problem. That is, our goal is to search for the rule set, denoted as $M^*$, among all possible rule sets $\mathcal{M}$, such that
\begin{equation} \label{eq:global_score}
\begin{split}
    M^* &= \arg \min_{M \cnew{{}\in \mathcal{M}}} L\left((x^n, y^n), M\right)\cnew{, \quad \text{where}} \\
    \cnew{L\left((x^n, y^n), M\right)} &:= L(M) -\log_2  P^{apprNML}_{\ruleset}(Y^n =y^n | X^n=x^n),
\end{split}
\end{equation}
in which $P^{apprNML}_{\ruleset}(Y^n =y^n | X^n=x^n)$ is defined in Equation~(\ref{eq:nml_cl_data}) and $L(M)$ in Equation~(\ref{eq:cl_model}). 

We refer to the proposed \cdel{optimization}\cnew{objective} function $L((x^n, y^n), M)$ as our model selection criterion; for a fixed model $M$ and a (training) dataset $(x^n, y^n)$, we refer to the value of $L((x^n, y^n), M)$ as the \emph{MDL-based score for the rule set model $M$}.

\section{Learning Truly Unordered Rule Sets from Data} \label{sec:alg}
Given the combinatorial nature of the search space, learning rule sets from data is an extremely difficult task. Notably, although some existing algorithms can obtain provably optimal rule lists~\citep{angelino2017learning} and decision trees~\citep{hu2019optimal}, their branch-and-bound approaches are not applicable to learning TURS models \cdel{due to}\cnew{for} the following reasons. First, our model class (and hence also search space) is different \cdel{than}\cnew{from} that of rule lists and decision trees, since our TURS model allows for overlaps of rules. Second, the output of the TURS model is probabilistic while the optimal \cdel{trees/lists}\cnew{tree/list} algorithms learn rule-based models with non-probabilistic (or just binary) output. Third, our model selection criterion, although requiring no hyper-parameter for regularization, does not allow efficient search for the global optimum, like most existing MDL-based approaches~\citep{galbrun2022minimum}. Hence, we cannot easily apply the branch-and-bound approaches as employed by optimal tree/list algorithms. 

As for rule set methods, traditional algorithms focus on defining heuristics that try to characterize the ``quality" of individual rules in different ways, often without a global optimization score~\citep{furnkranz2012foundations, furnkranz2005roc}. In addition, recently proposed ones mostly rely on randomized techniques: DRS~\citep{zhang2020diverseRuleSets} is based on a heuristic randomized algorithm, IDS~\citep{lakkaraju2016interpretable} on stochastic local search, BRS~\citep{wang2017bayesian} on simulated annealing, and CG~\citep{dash2018boolean} on (randomized) integer programming. However, BRS and CG are only suitable for binary \cdel{target}\cnew{targets} and non-probabilistic rules, while DRS and IDS turn out to have unsatisfactory predictive performance (as shown in Section~\ref{sec:exp}).  

Therefore, we develop a heuristic algorithm for iteratively learning single rules with extensive innovations in comparison to traditional heuristic algorithms. 

\begin{algorithm} [b!]
\caption{Iteratively Learning a Rule Set} \label{alg:rule_set}
\begin{algorithmic}[1]
\REQUIRE{dataset $D=(x^n, y^n)$}
\ENSURE{rule set $M$}
\STATE Initialize $M$ \COMMENT{Initialize an empty rule set.}
\WHILE{TRUE}
    \STATE $S \leftarrow$ Learn-Single-Rule$(M, D)$ \COMMENT{Described in Algorithm~\ref{alg:find_next_rule}}
    \IF{$L(D, M \cup \{S\}) < L(D, M)$}
        \STATE $M \leftarrow M \cup \{S\}$ \COMMENT{The ``else-rule" updates accordingly}
    \ELSE
        \RETURN rule set $M$
    \ENDIF
\ENDWHILE
\end{algorithmic}
\end{algorithm}

In the following, we start by describing the process of iteratively learning a rule set, followed by discussing the heuristic \cdel{of}\cnew{for} defining the ``best" single rule given the current status of the rule set. Then, we discuss how to learn a single rule, in which we introduce a \crevb{patience-diversified} beam search algorithm, together with a novel look-ahead strategy that we propose based on the analogy between the MDL principle and hypothesis testing~\cite[Chapter 14.3]{grunwald2007minimum}, which we hence name \emph{MDL-based local testing}.

\subsection{Learning a rule set}
\label{subsec:single_rule}
The procedure for learning a rule set iteratively, one rule at a time, is summarized in Algorithm~\ref{alg:rule_set}. The algorithm begins with an empty rule set, in which all instances are covered by the ``else-rule'' (Line 1). At each iteration, we select a candidate rule using a separate heuristic, termed the \emph{learning speed score} (Line 3). 

This heuristic prioritizes the rule whose inclusion is expected to reduce the MDL-based model selection criterion most efficiently, measured as the decrease in score \emph{per \cnew{additionally }covered instance}. The definition and rationale of the learning speed score are presented in detail below. 

The selected rule is added to the rule set only if its inclusion reduces the MDL-based model selection criterion defined in Equation~(\ref{eq:global_score}) (Lines 4--5). The process repeats until no candidate rule can be found that further improves the model selection criterion (Lines 2--9).

\subsection{Heuristic score for a single rule}
Consider the search space of all possible rule sets. 
Adding a single rule to the rule set can be viewed as a step that moves the model to a new point in the search space.
Since adding a rule that covers no previously uncovered instances \crev{leaves the coverage of the rule set unchanged while still increasing the code length of the model, we only consider candidate rules that cover at least one previously uncovered instance (moreover, the learning speed score introduced below would be undefined for such a rule, as its denominator would equal zero). Hence,} each valid step induces a monotonic increase in the \emph{coverage} of the rule set (excluding the else-rule).

We therefore introduce a heuristic that selects the next rule corresponding to the steepest descent relative to the increase in rule set coverage. Specifically, the ``best'' rule is defined as the one that maximizes the reduction in the MDL-based score per additionally covered instance. We refer to this heuristic as the \emph{learning speed score}.

Formally, given a rule set denoted as $M$, the learning speed score for a single rule $S$ to be added to $M$ is defined as 
\begin{equation}\label{eq:learning_speed}
    r(S) = \frac{L\left((x^n, y^n), M\right) - L\left((x^n, y^n), M \cup \{S\}\right)}{|c(M \cup \{S\})| - |c(M)|}, 
\end{equation}
in which $M \cup \{S\}$ denotes the rule set obtained by adding the single rule $S$ to $M$. Further, $|c(M)|$ and $|c(M \cup \{S\})|$ \cdel{respectively denotes}\cnew{denote} the coverage before and after adding $S$ to the rule set ${M}$\cnew{, respectively} (excluding the else-rule) --- precisely, for a set of rules $M$ and a given dataset $D$, $c(M)$ is defined as 
\[
c(M):= \{(x, y) \in D: \exists S \in M, x \in \Gamma_S\},
\]
and $|c(M)|$ represents the cardinality of $c(M)$. 

\ly{Note that $r(S)$ is not monotonic with respect to rule refinement: extending a rule by adding literals does not necessarily increase $r(S)$, as any reduction in empirical loss is explicitly regulated by the MDL-based model complexity penalty.}

We next discuss how to search for the next best rule that optimizes $r(S)$.

\subsection{Learning a single rule} \label{subsec:rule_grow}
\input{alg_nex_rule}

We describe the procedure for learning a single rule using beam search augmented with two components. The first component is an MDL-based local test that evaluates rule growth by explicitly modeling the instances excluded when a literal is added. The second is a beam selection strategy that diversifies candidates according to their relative coverage reduction. These components are introduced to address two limitations frequently observed in greedy rule learning: evaluation based solely on the rule obtained after adding a literal, and the tendency to select literals that induce large coverage reductions early in the search. We first describe the rule construction procedure step by step, and then summarize the overall algorithmic structure corresponding to the pseudo-code in Algorithm~\ref{alg:find_next_rule}.

\subsubsection{Step-by-Step Rule Growth}
\smallskip \noindent \textbf{Step 1: Initialization.} We begin with an empty rule $S$, whose condition contains no literals and therefore covers all instances. As defined in Section~\ref{section:rulesetModel}, a rule condition is expressed as a conjunction of literals, where each literal takes the form $\{X_i \in \cfix{\litset_i}\}$. Here, $\cfix{\litset_i}$ denotes an interval for quantitative variables or a subset of categorical levels.

\smallskip \noindent \textbf{Step 2: Candidate Enumeration.}
Given a rule $S$, a candidate growth result $S'$ is generated by appending a single literal. 

For continuous-valued variables $X_i$, quantiles are selected as candidate split points and combined with operators ($\geq$, $<$). The number of quantiles determines the search granularity and is treated as a hyper-parameter that depends on the task at hand, as previously discussed in Section~\ref{subsec:code_model}. For categorical variables, we assume one-hot encoding for simplicity, yielding literals of the form $(X_i = 1)$ or $(X_i = 0)$.

Enumerating literals across all variables defines the search space of candidate rule growth results. 

\smallskip \noindent \textbf{Step 3: Beam Search with Diversified Shrinkage Rate (\crevb{`Patience Diversification'}).}
\label{subsec:diverse_beam_search}
The learning speed score $r(S')$ measures the per-instance improvement obtained by adding a rule to a rule set. Directly optimizing $r(S')$ during rule growth, however, can be overly greedy. In particular, it tends to favor literals that induce large reductions in coverage. Although such updates are not inherently undesirable and can yield strong immediate improvements, they may eliminate large portions of the search space too aggressively. Consequently, many potentially useful rule candidates remain unexplored, increasing the risk of becoming trapped in suboptimal local optima. Effective rule growth therefore requires balancing improvements in $r(S')$ against the extent to which adding a literal reduces rule coverage.

\cdel{Let $c(S)$ denote}\cnew{Recall that $|c(S)|$ denotes} the coverage of rule $S$. For each candidate $S'$, we compute the ratio $\frac{|c(S')|}{|c(S)|}$. Given beam width $W$, candidates are partitioned into $W$ clusters:

\begin{equation}
\{S'\}_w =
\left\{
S' :
\frac{|c(S')|}{|c(S)|} \in 
\left[\frac{w-1}{W}, \frac{w}{W}\right)
\right\},
\quad w = 1, \ldots, W.
\end{equation}

Within each cluster, we select the candidate maximizing $r(S')$, subject to passing the MDL-based local test.

Unlike standard beam search, which retains the globally top-$W$ candidates by score, this strategy enforces diversity over coverage reduction. Candidates associated with small and large shrinkage ratios are therefore considered simultaneously. This mechanism is related to the notion of patience introduced in PRIM~\citep{friedman1999bump}, here incorporated into beam search through explicit clustering\crevb{: the beam is diversified over the degree of patience of its candidates. We hence refer to our algorithm as \emph{patience-diversified beam search}.}

\smallskip \noindent \textbf{Step 4: MDL-based Local Test.} 
\label{subsec:local_local testing}
When a rule $S$ is extended by adding a literal, producing a growth result $S'$, some instances covered by $S$ may no longer be covered by $S'$. These instances form the left-out subset $c(S) \setminus c(S')$; we write $S_l$ for the corresponding hypothetical rule whose coverage is exactly this subset, i.e.,
\[
c(S_l) = c(S) \setminus c(S').
\]

In many rule learning approaches, candidate evaluation focuses solely on the rule $S'$. The instances in \crevb{$c(S_l)$} are implicitly assumed to be handled by rules learned later. This may favor rule growth steps that improve the fit of $S'$ while excluding instances whose modeling cost is deferred.

To account for this effect, we introduce an MDL-based local test that explicitly evaluates the trade-off induced by extending $S$ to $S'$. A candidate $S'$ is retained only if
\begin{equation} \label{eq:local_local testing}
-\log_2 P_S^{NML}(y^S|x^S) 
>
-\log_2 P_{S'}^{NML}(y^{S'}|x^{S'}) 
- \log_2 P_{S_l}^{NML}(y^{S_l}|x^{S_l}) 
+ L_{split},
\end{equation}
where $P^{NML}$ denotes the NML distribution defined in Equation~(\ref{eq:local_nml}), and $L_{split}$ is the code length required to encode the literal responsible for splitting $S$ into $(S', S_l)$. This encoding specifies the variable involved and the corresponding threshold or categorical subset using the uniform coding scheme (as discussed in Section~\ref{subsec:code_model}).

Intuitively, this test compares two local modeling choices: representing all instances covered by $S$ using a single rule, or representing the depth-one partition $(S', S_l)$. Growth is accepted only when the partition yields a shorter description length.

By explicitly incorporating the description length of $S_l$, the test evaluates whether extending $S$ to $S'$ improves the local representation of the data covered by $S$. In particular, it discourages rule growth steps that achieve a better fit on $S'$ while leaving a subset of instances that is costly to encode.

As shown empirically in \crevb{Section~\cnew{\ref{subsec:exp_local_constraint}}}, omitting this test leads to behavior in which the learning speed score favors rule growth steps that exclude instances that are difficult to model later, which may result in rules involving irrelevant variables. 

In summary, the MDL-based local test serves as a filter that evaluates rule growth steps by incorporating the immediate fit of the refined rule together with the \emph{anticipated} modeling cost of the left-out instances.

\smallskip \noindent \textbf{Step 5: Obtaining Length-2 Rules.}
Steps 1--4 produce up to $W$ length-1 candidate rules, selected via the MDL-based local test and the diversified shrinkage beam strategy. Each of these rules is then treated as a parent rule for further growth.

For every length-1 candidate, we repeat the same procedure: candidate literals are enumerated, filtered using the MDL-based local test, and selected according to the diversified shrinkage criterion. This process generates up to $W$ length-2 candidates per parent rule.

\smallskip \noindent \textbf{Step 6: Beam Consolidation.} Pooling candidates from $W$ parent rules produces up to $W^2$ candidates. To maintain beam width $W$, we sort candidates by coverage, partition them into $W$ equally sized groups, and select the rule with the highest $r(\cdot)$ score from each group. The selected rules form the beam for the next iteration. These $W$ rules serve as the parent rules for the rule growth in the next iteration. 

\smallskip \noindent \textbf{Step 7: Recursive Growth, Stopping, and Rule Selection.}
Length-3 and higher-length candidate rules are generated by repeatedly applying Steps 2--6. At each iteration, rules in the current beam are extended, filtered using the MDL-based local test, and selected using the diversified shrinkage strategy, followed by beam consolidation.

We adopt a patience-based stopping rule with budget $K_{stop}=5$. The search terminates when the best candidate in the current beam does not improve upon the best candidate in the previous beam for $K_{stop}$ consecutive iterations.

Upon termination, we return the rule with the largest learning speed score $r(\cdot)$ among all explored candidates.

\subsubsection{Algorithm Description} \label{subsec:alg_whole}
We now summarize the complete procedure for learning a single rule, with pseudo-code provided in Algorithm~\ref{alg:find_next_rule}. Given a rule set $M$, dataset $(x^n, y^n)$, \cdel{and }beam width $W$,\cnew{ and patience budget $K_{stop}$,} the algorithm initializes \textbf{[Line 2]} with an empty rule whose condition contains no literals. At each iteration, a beam is initialized \textbf{[Line 4]}, and every rule selected for the next iteration is treated as a parent rule \textbf{[Line 5]}. Candidate rules are generated by enumerating literals that can be appended \textbf{[Line 6]}. The candidates are then partitioned according to their relative coverage reduction \textbf{[Line 8]}, filtered using the MDL-based local test \textbf{[Line 9]}, and evaluated using the learning speed score $r(\cdot)$ \textbf{[Line 10]}. The best candidate from each coverage cluster is retained to form a diversified beam. Candidates generated from all parent rules are subsequently consolidated to maintain beam width $W$ \textbf{[Line 18]}. This process is repeated to generate rules of increasing length. The search terminates according to the patience-based stopping criterion \textbf{[Line 14]}, after which the rule with the largest learning speed score among all explored candidates is returned \textbf{[Line 16]}.

\section{Experiments} \label{sec:exp}
We extensively benchmark our \crevb{patience-diversified} beam search algorithm and study the truly unordered rule sets (TURS) model learned from data from the following perspectives:
\begin{enumerate}
     \vspace{-0.2cm}\item Does the TURS model learned from data achieve classification performance that is on par with or superior to other rule-based methods, particularly rule set methods that allow (implicit) rule ordering?
     \vspace{-0.2cm}\item Does the TURS model learned from data exhibit strong interpretability properties, including low model complexity and reliable class probability estimates from individual rules?
     \vspace{-0.2cm}\item While our theoretical results suggest that TURS inherently encourages overlap consistency, to what extent are these overlaps empirically consistent?
\end{enumerate}
\input{data_info_table.tex} 

\subsection{Setup}
\textbf{Datasets.} We conduct extensive experiments on 31 datasets, summarized in Table~\ref{table:data_info}. The multi-class datasets are drawn from the UCI repository, while the binary-class datasets are obtained from both the UCI repository~\citep{Dua:2019:UCI} and the ADBench GitHub repository~\citep{han2022adbench}, marked as \textit{italic} in Table~\ref{table:data_info}. The latter is a benchmark toolbox for anomaly detection (including imbalanced classification).

\smallskip
\noindent
\textbf{Competitors.} We compare against a wide range of methods, summarized as follows. First, we include \emph{unordered} CN2~\citep{clark1991cn2Improve}, which adopts a one-versus-rest strategy. Since CN2 does not impose an implicit order among rules, it is conceptually the closest competitor to our method. Second, we compare with DRS~\citep{zhang2020diverseRuleSets} and IDS~\citep{lakkaraju2016interpretable}, which, to the best of our knowledge, are the only two multi-class rule set methods that do not rely on first learning class-specific rules via a one-versus-rest strategy. Notably, both methods incorporate overlap-related properties in their optimization objectives: DRS minimizes the size of overlaps, while IDS optimizes a linear combination of seven scores, one of which explicitly penalizes overlap size. Third, we include CLASSY, an ordered rule list method, as it employs a similar MDL-based model selection approach. Fourth, given the conceptual relationship between the MDL principle and Bayesian modeling, we also compare with BRS~\citep{wang2017bayesian}, a representative Bayesian framework method that adopts a non-heuristic simulated annealing procedure. Finally, we include RIPPER~\citep{cohen1995ripper}, CART~\citep{breiman1984classification}, and C4.5 decision trees~\citep{quinlan2014c4}, due to their widespread use in practice.

\smallskip \noindent{\textbf{Implementations and hyper-parameters.}} For TURS, we implemented the algorithm in Python. As for hyper-parameters, we set the beam width to 10\cdel{ and}\cnew{,} the number of candidate cut points for numeric features to 20\cnew{, and the patience budget $K_{stop}$ to 5} as the default values\footnote{We observe that increasing the number of candidate cut points for numeric features to 100, as well as the beam width to 20, does not yield significant differences in predictive performance.}. 
For competitor algorithms, we use CN2 from Orange~\citep{JMLR:demsar13a}, IDS from a third-party implementation with proven scalability~\citep{filip2019pyids}, RIPPER and C4.5 from \cdel{Weka}\cnew{WEKA}~\citep{hall2009weka} and its R wrapper, CART from \cdel{Scikit-Learn}\cnew{scikit-learn}~\citep{scikit-learn}, and DRS, BRS, and CLASSY from the original authors’ implementations. Competitor configurations follow the defaults reported in the respective papers and/or the original authors' implementations. We make our code for TURS publicly available for reproducibility\footnote{\url{https://github.com/ylincen/TURS2}.}.

All reported results are based on five-fold stratified cross-validation, unless stated otherwise.
\input{table_roc_auc}

\begin{figure}[ht]
\centering
\includegraphics[width=.8\textwidth]{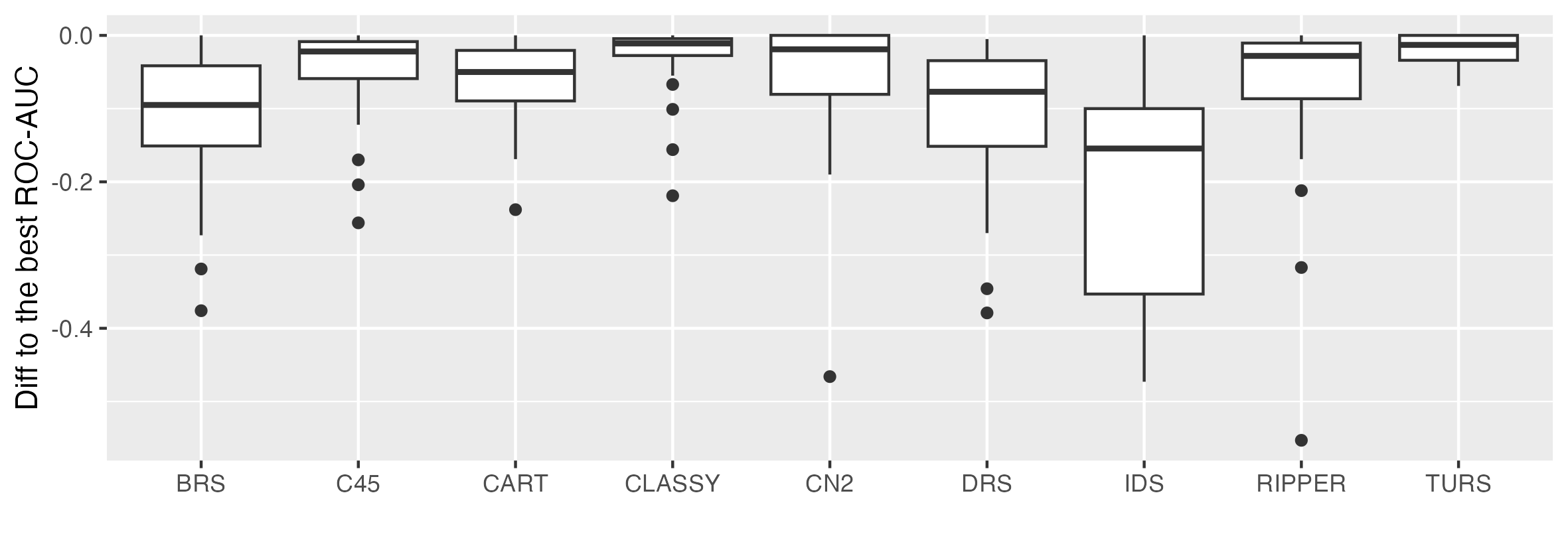}	
\caption{For each algorithm, we calculate for every individual dataset the difference between its ROC-AUC score and the best ROC-AUC score. The differences to the best ROC-AUC scores for each algorithm are illustrated by a box-plot.} \label{fig:diff_to_best_auc}
\end{figure}

\subsection{Classification performance} \label{subsec:classifier_perf}
We evaluate the classification performance of TURS using average ROC-AUC scores from five-fold \emph{stratified} cross-validation. Results across all datasets are reported in Table~\ref{table:roc_auc}. For multi-class classification, we adopt the one-versus-rest ROC-AUC and report the macro-average, which assigns equal weight to all classes and thus reflects performance on minority classes.

BRS~\citep{wang2017bayesian} is applicable only to binary datasets. Additionally, we were unable to obtain results for DRS on four datasets because its implementation cannot handle datasets with a very large number of features\footnote{Specifically, the implementation converts binary vectors into integers using \texttt{NumPy}, which does not support arbitrarily large integers.}. For IDS, results could not be obtained on one dataset due to exceeding the time limit of 10 hours for a single fold.

Overall, TURS demonstrates strong and competitive classification performance. It achieves the highest ROC-AUC on 11 out of 31 datasets, including 5 out of 12 multi-class datasets. For each dataset, the best ROC-AUC among all methods is highlighted in bold in Table~\ref{table:roc_auc}. 

To contextualize performance differences, we also report in parentheses the difference between the ROC-AUC of TURS and the best ROC-AUC achieved on the corresponding dataset. Figure~\ref{fig:diff_to_best_auc} summarizes these gaps using box plots. The results indicate that TURS consistently exhibits small gaps to the best-performing method, suggesting stable performance with limited degradation relative to the dataset-specific optimum.

Among rule-set-based methods (CN2, DRS, IDS, and TURS), TURS substantially outperforms DRS and IDS. Since both DRS and IDS explicitly aim to minimize overlap size, these findings suggest that overly restrictive overlap minimization may lead to suboptimal classification performance. CN2 remains competitive, achieving the best ROC-AUC on 11 datasets as well (Table~\ref{table:roc_auc}). However, as shown in Figure~\ref{fig:diff_to_best_auc}, CN2 generally exhibits larger gaps to the best ROC-AUC than TURS.

\subsection{Interpretability}
We study the interpretability of rule-based models in two aspects: the model complexity and the reliability of the probabilistic outputs of rules. The former is a proxy for the cognitive workload that is required for comprehending the decision logic of the model, and the latter measures the trustworthiness and generalizability of the rules when they are used as explanations. 

\subsubsection{Model complexity} \label{subsec:model_complexity}
We next study whether TURS leads to complex models in practice. Since one of the main goals of rule-based models is interpretability, it is important to assess whether the flexibility of TURS comes at the cost of increased model complexity.

\smallskip \noindent \textbf{Total number of literals.} We first measure model complexity using the total number of literals, defined as the sum of the lengths of all rules in a learned model. For decision trees, each root-to-leaf path is treated as a rule. This measure directly reflects the cognitive effort required from a domain expert to inspect and understand the model, and is therefore a natural proxy for interpretability. The total number of literals is reported in Table~\ref{table:num_literal}.

When comparing model complexity, it is important to account for classification performance. A substantially simpler model is of limited interest if it achieves markedly worse predictive accuracy. Therefore, for each dataset, we mark \crev{with an asterisk ($^{\ast}$)} those competitor models whose ROC-AUC scores are more than $0.1$ lower than that of TURS. These models are excluded from the comparison of minimal complexity. Among the remaining models, the smallest total number of literals for each dataset is highlighted in bold.

Excluding models with substantially inferior ROC-AUC, TURS produces the simplest model on 16 out of the 31 datasets. This indicates that allowing controlled overlaps does not systematically inflate model complexity; on the contrary, TURS frequently yields models that are competitive or superior in terms of complexity while maintaining strong predictive performance. As shown in Table\cnew{s~\ref{table:roc_auc} and}~\ref{table:num_literal}, TURS achieves strong ROC-AUC scores while maintaining a small total number of literals. 

\smallskip \noindent \textbf{Number of rules and average rule lengths.} To further disentangle the sources of complexity, we additionally report the number of rules and the average rule length in Tables~\ref{table:num_rules} and~\ref{table:avg_rule_length}, respectively. These results reveal that the compactness of TURS models typically arises from a combination of a moderate number of rules and relatively short rule lengths, rather than from aggressively pruning either dimension alone. In terms of structural simplicity, TURS produces the smallest number of rules on 13 out of the 31 datasets (Table~\ref{table:num_rules}). 

Further, the average rule lengths produced by TURS are also consistently small (\crev{Table~\ref{table:avg_rule_length}}), suggesting that individual rules remain easy to interpret in practice. This is notable given that TURS produces probabilistic outputs, unlike BRS, RIPPER, and C4.5, and does so without relying on explicit or implicit rule ordering, unlike CLASSY, DRS, and IDS. 
Despite these additional modeling capabilities, TURS maintains rule lengths that remain manageable for domain experts in practice, aligning with established findings on human working memory capacity~\citep{cowan2012working}.

In contrast, decision-tree-based methods such as CART and C4.5 tend to produce substantially longer rules when tree paths are interpreted as rules. This empirical observation supports our conjecture that decision trees often include internal nodes that primarily serve structural purposes, rather than contributing directly to predictive accuracy, leading to unnecessarily lengthy rule representations.

In summary, these results suggest that the overlap mechanism in TURS does not lead to unnecessarily complex models. Instead, it provides additional modeling flexibility that can be exploited to balance predictive performance and interpretability, often resulting in concise rule sets that remain competitive across a wide range of datasets.

\begin{figure}[ht]\centering
\includegraphics[width=0.9\textwidth]{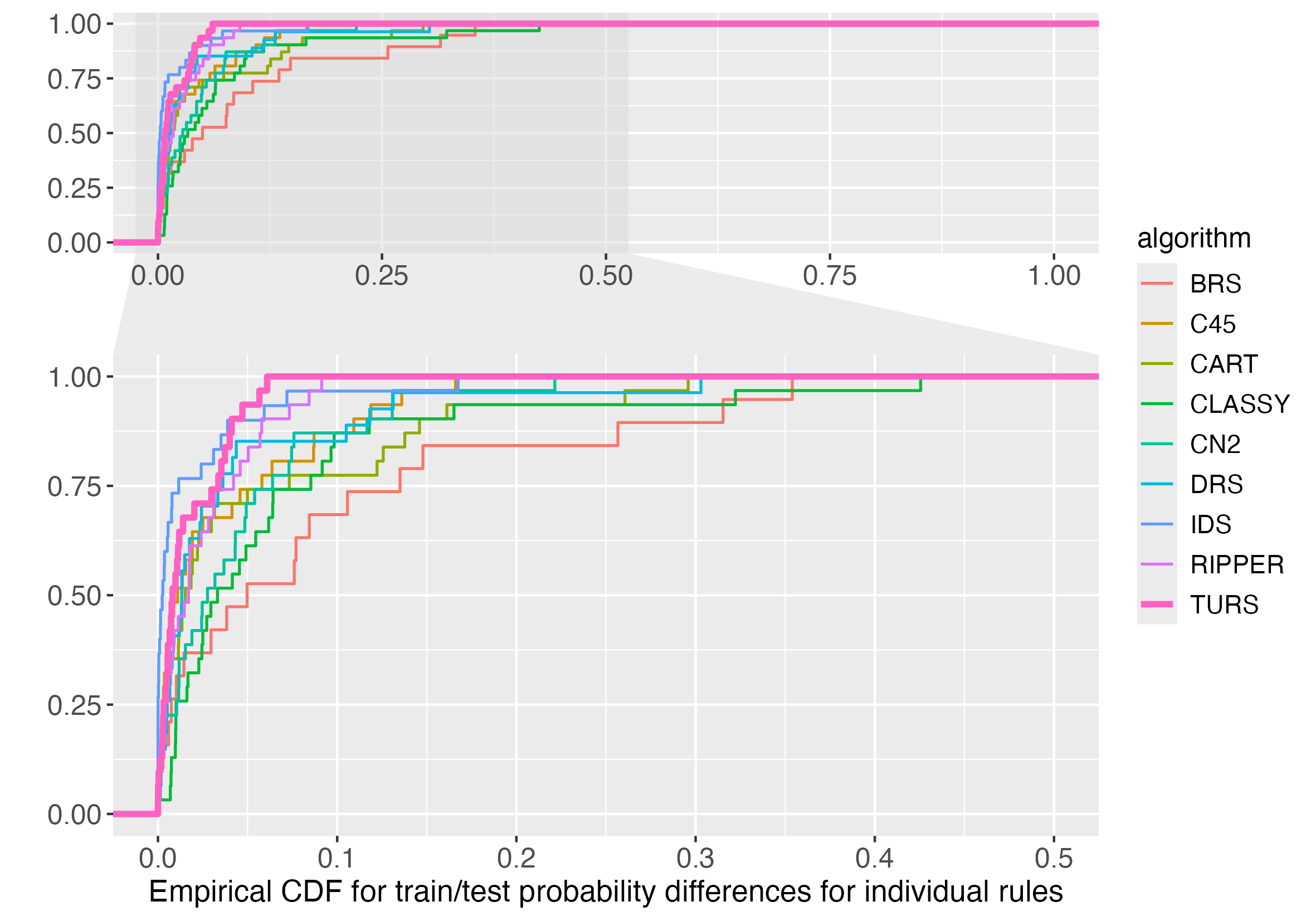}
\caption{The weighted average of the differences between the class probability estimates of each individual rule on the training and test sets, shown as an empirical cumulative distribution function (CDF). The weights are defined by the coverage of each rule on the training set. The zoomed panel emphasizes the range $[0,0.5]$ to improve visibility where most differences are concentrated.}
\label{fig:train_test_diff}
\end{figure}

\subsubsection{Reliability of rules' probabilistic outputs}
While rule-based models are commonly regarded as intrinsically explainable, only rules whose probability estimates generalize well can be considered trustworthy explanations. We therefore examine the discrepancy between individual rules’ probability estimates on the training and test sets.

Specifically, given a rule set induced from a dataset, we compute each rule’s class probability estimates separately on the training and test sets \cfix{as the empirical class proportions among the instances that the rule covers (i.e., the maximum likelihood estimator of that rule's categorical distribution)}. We then report the weighted average of the absolute differences between these estimates across all rules, where weights correspond to the coverage of each rule on the training set.

Formally, consider a rule set with $K$ rules, $M = \{S_1, \ldots, S_K\}$ together with the else rule $S_0$. Let $(\mathbf{p}_{\cfix{0}}, \ldots, \mathbf{p}_K)$ and $(\mathbf{q}_{\cfix{0}}, \ldots, \mathbf{q}_K)$ denote the class probability vectors estimated from the training and test sets, respectively. Each probability vector has length $\cfix{|\mathcal{Y}|}$ ($\cfix{|\mathcal{Y}|} = 2$ for binary targets and $\cfix{|\mathcal{Y}|} > 2$ for multi-class targets). We quantify the generalization of rule probabilities as
\begin{equation}
  g = \frac{1}{\sum_{j} |c(S_j)|} \sum_{j} |c(S_j)|
      \left( \sum_{\cfix{k}} \frac{1}{\cfix{|\mathcal{Y}|}} \left| p_{j\cfix{k}} - q_{j\cfix{k}} \right| \right),
\end{equation}
where $p_{j\cfix{k}}$ and $q_{j\cfix{k}}$ denote the $\cfix{k}$-th elements of $\mathbf{p}_j$ and $\mathbf{q}_j$, respectively\cfix{, and $j$ ranges over $0, 1, \ldots, K$}. Each rule is evaluated independently when computing the $g$-score; consequently, rule overlaps do not affect this measure. We use the absolute probability difference here because it provides a direct and easily interpretable measure of how much the predicted class probabilities change between the training and test sets.

We compute this score for all algorithms and datasets, averaging results over five-fold stratified cross-validation. The distributions of $g$-scores are shown as empirical cumulative distribution functions (ECDFs) in Figure~\ref{fig:train_test_diff}. Curves positioned closer to the upper-left correspond to smaller discrepancies between training and test probability estimates. Under this criterion, TURS (bold curve) consistently achieves the most favorable distribution of $g$-scores among all competing methods, with IDS being the only notable exception. 

However, this apparent advantage of IDS must be interpreted in light of its substantially inferior predictive performance (Table~\ref{table:roc_auc} and Section~\ref{subsec:classifier_perf}). This behavior is largely explained by the tendency of IDS to induce rules with excessively large coverage, resulting in coarse and insufficiently discriminative decision structures. Although such broadly covering rules naturally yield more stable probability estimates across training and test sets, they are symptomatic of underfitting and fail to capture the finer-grained patterns necessary for accurate classification.

Overall, TURS provides \cdel{the}\cnew{a} more desirable balance between probability estimate stability and predictive performance than IDS\cdel{ does}. Consequently, the rules learned by TURS offer more reliable and trustworthy probabilistic interpretations compared to those produced by the other \cdel{eight }tree- and rule-based competitor models\cnew{, with IDS being the exception discussed above}.

\input{table_model_complexity2}

\input{table_num_rules}
\input{table_avg_rule_length}

\subsection{Empirical overlap consistency}

Our probabilistic formulation assigns overlapping regions an explicit semantic role by resolving overlaps through the union-based conditional distribution. As discussed in Section\cnew{s}~\ref{subsec:semantic_overlap} and~\ref{subsec:formal_justification}, this construction implicitly encourages overlap consistency: overlapping rules are expected to produce similar probabilistic predictions, since substantial disagreement would incur a likelihood penalty under the union-based model.

We next empirically investigate the extent to which this expected overlap consistency manifests in practice. Specifically, we first analyze the probabilistic predictions produced by different rules when a single instance is covered by multiple rules. We then assess whether overlap consistency is sufficiently strong to support prediction by randomly selecting one of the applicable rules.

\subsubsection{Probabilistic disagreement of overlaps}
\label{subsubsec:overlap_disagreement}

\begin{figure}[ht]
    \centering
    \includegraphics[width=\linewidth]{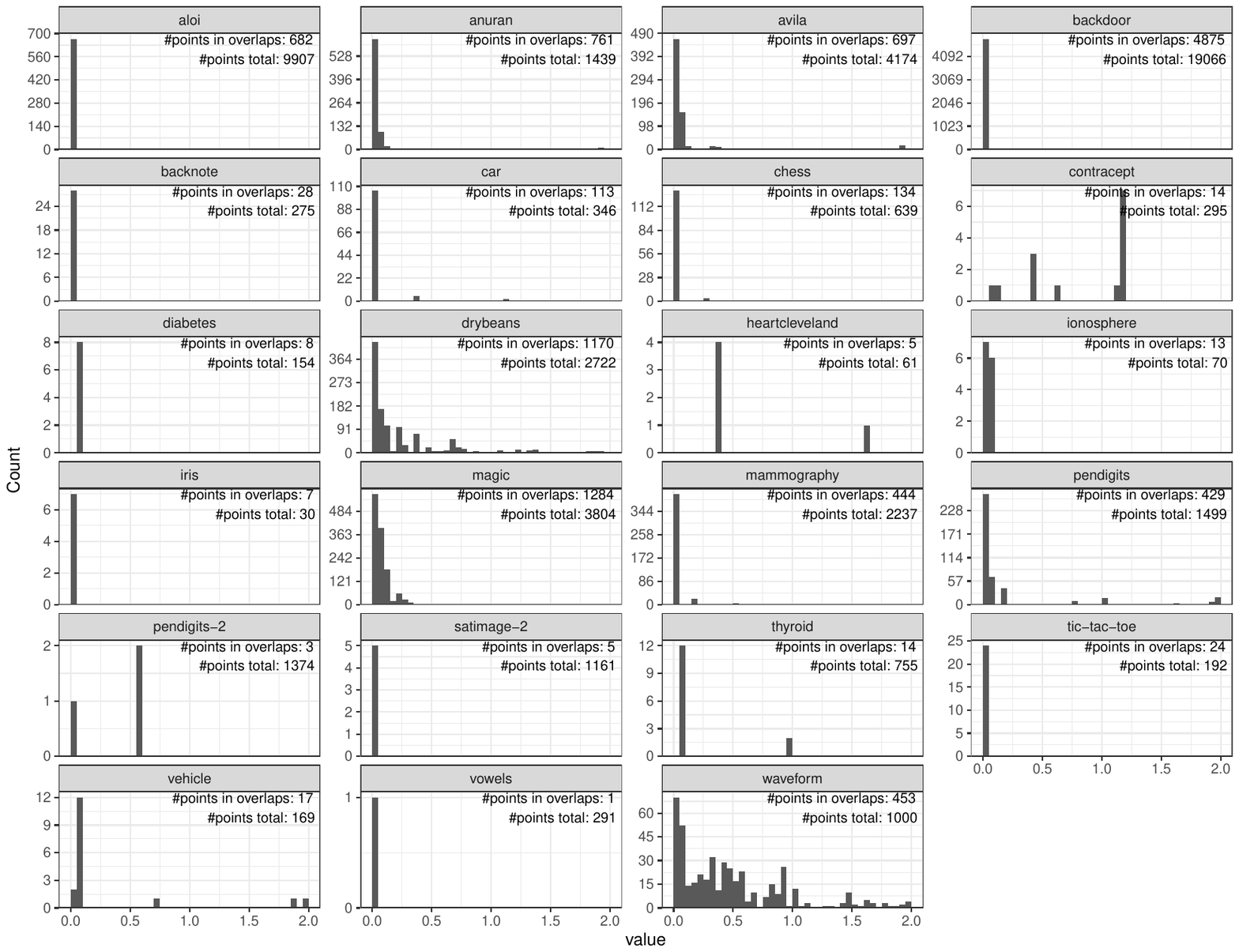}
    \caption{Histograms of the maximum pairwise $\ell_1$-distance between probability estimates of overlapping rules, computed on the test set of a single cross-validation fold.}
    \label{fig:l1_overlap_disagreement}
\end{figure}

To quantify probabilistic agreement in overlapping regions, we analyze, for each dataset, the disagreement between overlapping rules learned by TURS. For each test instance covered by multiple rules, we compute the maximum pairwise $\ell_1$-distance between the categorical probability distributions associated with these rules. This statistic captures the strongest probabilistic disagreement among all rules applicable to an instance, thus providing a conservative measure of overlap inconsistency. Figure~\ref{fig:l1_overlap_disagreement} reports histograms of this quantity for each dataset, along with the number of instances located in overlapping regions.

Results are shown for 23 out of the 31 datasets. The remaining 8 datasets exhibit no overlaps in the learned TURS models and are therefore omitted. This observation highlights that overlap is not an inherent property of TURS: when the data admit a clean partition of the instance space, TURS naturally learns overlap-free models. Overlaps arise only when they improve the overall model fit.

For most datasets, overlap consistency is high. The $\ell_1$-distances are sharply concentrated near zero, indicating that overlapping rules typically produce very similar probabilistic predictions. This suggests that overlaps learned by TURS predominantly occur between probabilistically compatible rules rather than arbitrary rule combinations. In these cases, overlaps may be interpreted as providing multiple, complementary descriptions of essentially the same local conditional distribution.

A small subset of datasets---most notably {waveform}, {magic}, and {drybeans}---exhibits comparatively higher disagreement levels. Despite this increased disagreement, TURS achieves top or near-top ROC-AUC performance on these datasets, with differences of at most $0.013$ relative to the best-performing method. This indicates that moderate overlap disagreement does not necessarily translate into degraded predictive accuracy. A plausible explanation is that these datasets exhibit more heterogeneous class-conditional structure, for which strict partitioning may be suboptimal under finite data. Enforcing stronger homogeneity constraints would either over-fragment the rule set or yield poorly supported estimates.

Beyond predictive accuracy, these results also inform the interpretability of TURS models. When overlap consistency is high, overlapping rules provide convergent rather than conflicting explanations. In this sense, overlaps need not undermine comprehensibility; instead, they may enrich interpretation by offering multiple consistent rule-based perspectives on the same probabilistic prediction. Empirically, TURS predominantly learns overlaps exhibiting this desirable behavior.

This observation is particularly informative when considered in the context of existing rule-based models. In previously proposed rule learning approaches, overlaps between rules are typically addressed through post-hoc conflict-resolution mechanisms, such as explicit or implicit rule ordering and weighted aggregation. While these strategies can be effective for prediction, they may introduce dependencies on factors external to the individual rules. Consequently, the interpretation of a rule’s prediction may depend on how such conflicts are resolved, rather than solely on the rule’s condition and its estimated probabilities.

In contrast, the high degree of overlap consistency observed in TURS suggests that overlapping rules generally yield compatible probabilistic predictions. This reduces reliance on additional conflict-resolution layers and enables a more direct semantic interpretation of individual rules, even in the presence of overlaps. Although overlaps are often regarded as a potential source of ambiguity in rule-based systems---motivating methods such as IDS~\citep{lakkaraju2016interpretable} and DRS~\citep{zhang2020diverseRuleSets} to explicitly penalize overlap size---our results indicate that, under our proposed formulation, overlaps can instead reflect coherent and interpretable model behavior.

\begin{figure}[ht]
    \centering
    \includegraphics[width=0.8\linewidth]{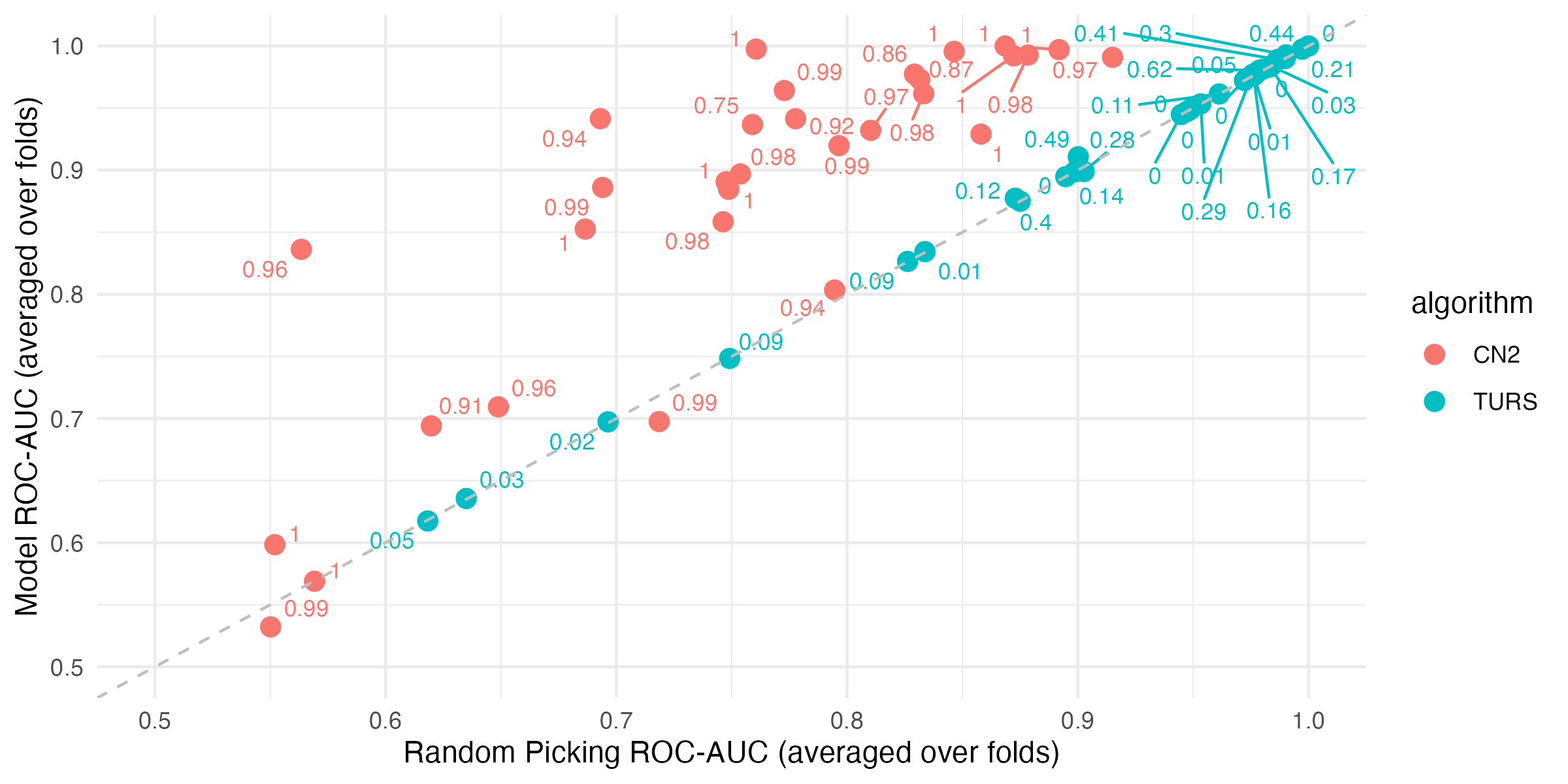}
    \caption{ROC-AUC obtained by TURS \cnew{and CN2 }versus ROC-AUC obtained using random picking for overlaps. The annotated numbers denote the fractions of instances covered by the overlaps.}
    \label{fig:random_picking}
\end{figure}
\subsubsection{Prediction with random picking for overlaps}
\label{subsec:exp_rp}
To assess whether overlap consistency is sufficiently strong for prediction in practice, we compare TURS with a baseline based on \emph{random picking}. Under this strategy, when an unseen instance is covered by multiple rules, one applicable rule is selected uniformly at random, and its probability estimates are used for prediction. If overlapping rules produce similar probabilistic outputs, predictive performance should remain close to that of the union-based TURS prediction.

Figure~\ref{fig:random_picking} compares the ROC-AUC scores obtained by TURS and by random picking, averaged over five-fold stratified cross-validation. For each fold, the ROC-AUC of the random-picking strategy is averaged over 10 independent random selections to reduce the effect of randomness and avoid results that depend on a particularly favorable random choice. Each point corresponds to a dataset, with annotations indicating the fraction of instances covered by multiple rules. For comparison, the corresponding results for CN2 are also shown.

Across datasets, ROC-AUC differences between TURS and random picking are negligible, remaining below $0.01$ even for datasets with substantial overlap. This demonstrates that overlap consistency learned by TURS is sufficiently strong: predictions are largely insensitive to which rule is selected within overlapping regions. In contrast, CN2 exhibits noticeably larger discrepancies, indicating greater variability among overlapping rule predictions and stronger dependence on conflict-resolution mechanisms.

These findings also carry interpretability implications. Since random picking performs comparably to union-based prediction, individual rules participating in overlaps may be regarded as providing stable local probabilistic explanations. Overlaps in TURS therefore do not introduce substantial ambiguity; rather, they offer multiple, mutually consistent rule-based perspectives on the same prediction.

\subsection{Runtime}

\begin{figure}[b!]
	\centering
	\includegraphics[width=\textwidth]{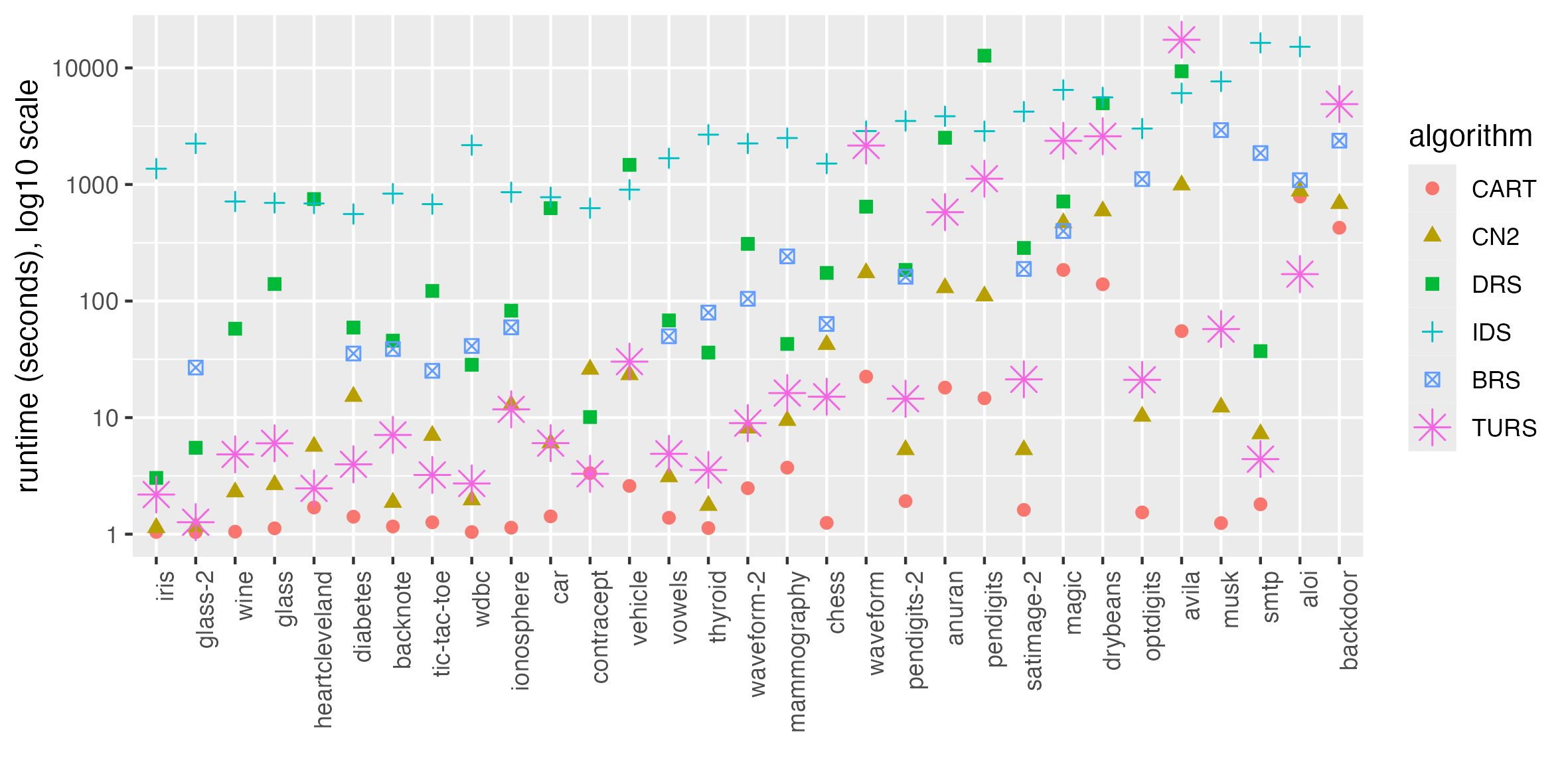}
	\caption{Average runtime (in seconds) for five rule set methods \cnew{and CART }across datasets. The y-axis is shown on a $\log_{10}(\cdot)$ scale. Datasets on the x-axis are ordered by the number of rows $\times$ (number of columns + number of classes), from smallest to largest.}
	\label{fig:runtime}
\end{figure}

Finally, we report the runtime of TURS alongside competing rule set methods. We focus on the comparison to rule set learners, as decision tree and rule list methods are implemented in highly optimized software libraries and operate over more structurally constrained search spaces, typically resulting in substantially lower runtimes. Nevertheless, as a practical point of reference, we also report the runtime of CART, a widely used and highly efficient decision tree method. 

Figure~\ref{fig:runtime} summarizes the average runtime (in seconds) obtained under five-fold cross-validation. The $\log_{10}$ scaling enables clear visualization of differences spanning multiple orders of magnitude.

Overall, TURS \cdel{outperforms IDS, BRS, and DRS in computational efficiency, while remaining on par with CN2}\cnew{is faster than IDS, BRS, and DRS on the large majority of datasets, the exceptions being a few of the larger ones (avila, magic, waveform, and backdoor); compared to CN2, TURS has comparable runtimes on small datasets and is slower on larger ones}. Despite optimizing a probabilistic objective over a flexible, unordered rule space, TURS maintains practical runtime behavior.\cdel{ These results show that the strong predictive performance of TURS is achieved without sacrificing computational efficiency.}

\section{Ablation Studies} \label{sec:ablation}
We conduct ablation studies to investigate the empirical effect of the learning speed score, the \crevb{patience-diversified} beam search, and the MDL-based local testing.
\begin{figure} [ht]
    \centering
    \includegraphics[width=\linewidth]{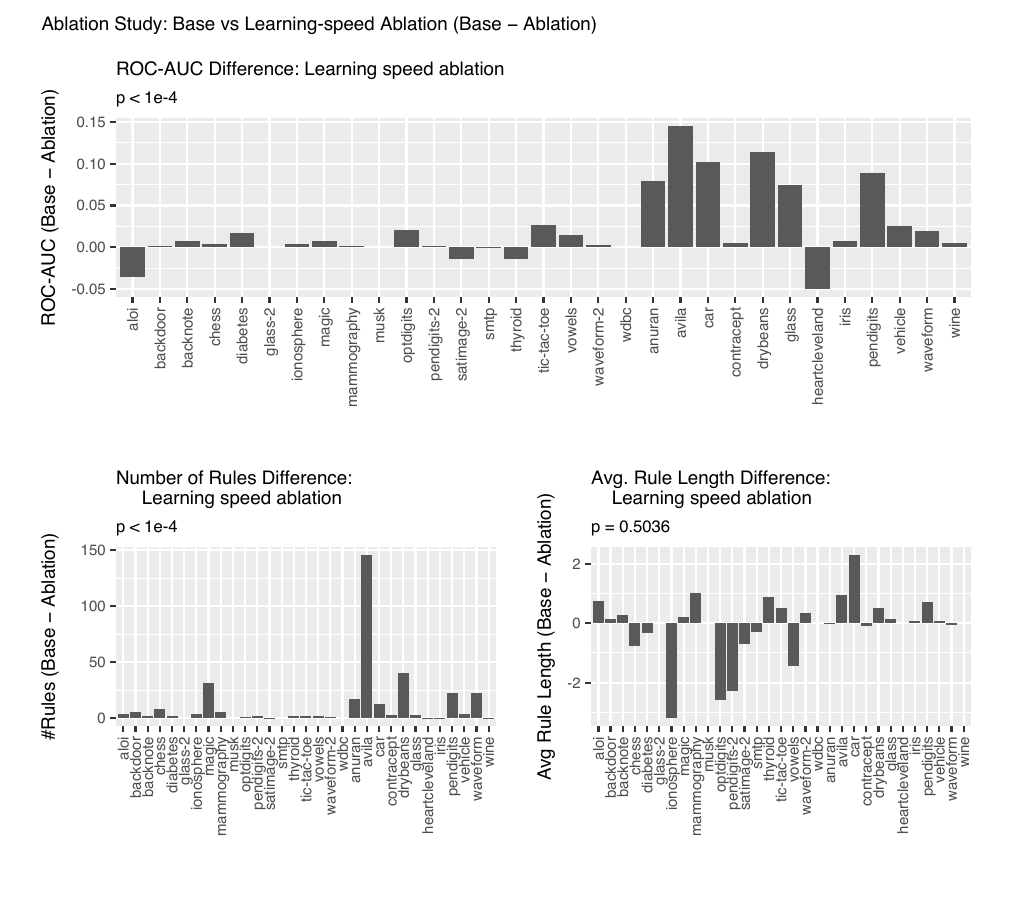}
    \vspace{-10mm}
    \caption{Ablation study comparing the base TURS model with the learning-speed-score ablation. Bars show differences between the base and ablated models across datasets. Positive values indicate higher values for the base model (higher ROC-AUC, higher model complexity). Note that the three plots use different y-axis ranges, and that the ROC-AUC difference plot does not show the full ROC-AUC range $(0,1)$.}
    \label{fig:ablation_learning_speed}
\end{figure}
\subsection{Ablation I: learning speed score}
To assess the role of the learning speed score, we remove it from both stages where it is used in TURS: 1) during rule growth, when searching for the next literal to extend a rule, and 2) when selecting which candidate rule is added to the rule set. In the ablated variant, candidate refinements and candidate rules are no longer ranked by the MDL decrease normalized by \cdel{coverage reduction}\cnew{the number of additionally covered instances}; instead, they are selected purely based on the absolute decrease in the global MDL-based model criterion (Equation~\cnew{(}\ref{eq:global_score}\cnew{)}). Consequently, the search becomes strictly greedy with respect to total MDL decrease, without controlling for how \cdel{much coverage is sacrificed}\cnew{many additional instances are covered} to obtain that decrease.

The effect on predictive performance is clear and statistically decisive. As shown in Figure~\ref{fig:ablation_learning_speed}, the base model outperforms the ablated variant on the large majority of datasets, often by a substantial margin (e.g., {avila}, {car}, {drybeans}, and {pendigits}). A one-sided Wilcoxon signed-rank test yields $p < 10^{-4}$, establishing a highly significant improvement in ROC-AUC when the learning speed score is used. 

Regarding model complexity, enabling the learning speed score significantly increases the number of rules ($p < 10^{-4}$, two-sided Wilcoxon test), while leaving the average rule length statistically unchanged ($p = 0.5036$). Notably, the datasets exhibiting the largest increases in rule count largely coincide with those showing the strongest ROC-AUC improvements.
\subsection{Ablation II: \crevb{patience-diversified} beam search} \label{subsec:exp_patience}
\begin{figure}
    \centering
    \includegraphics[width=0.95\textwidth]{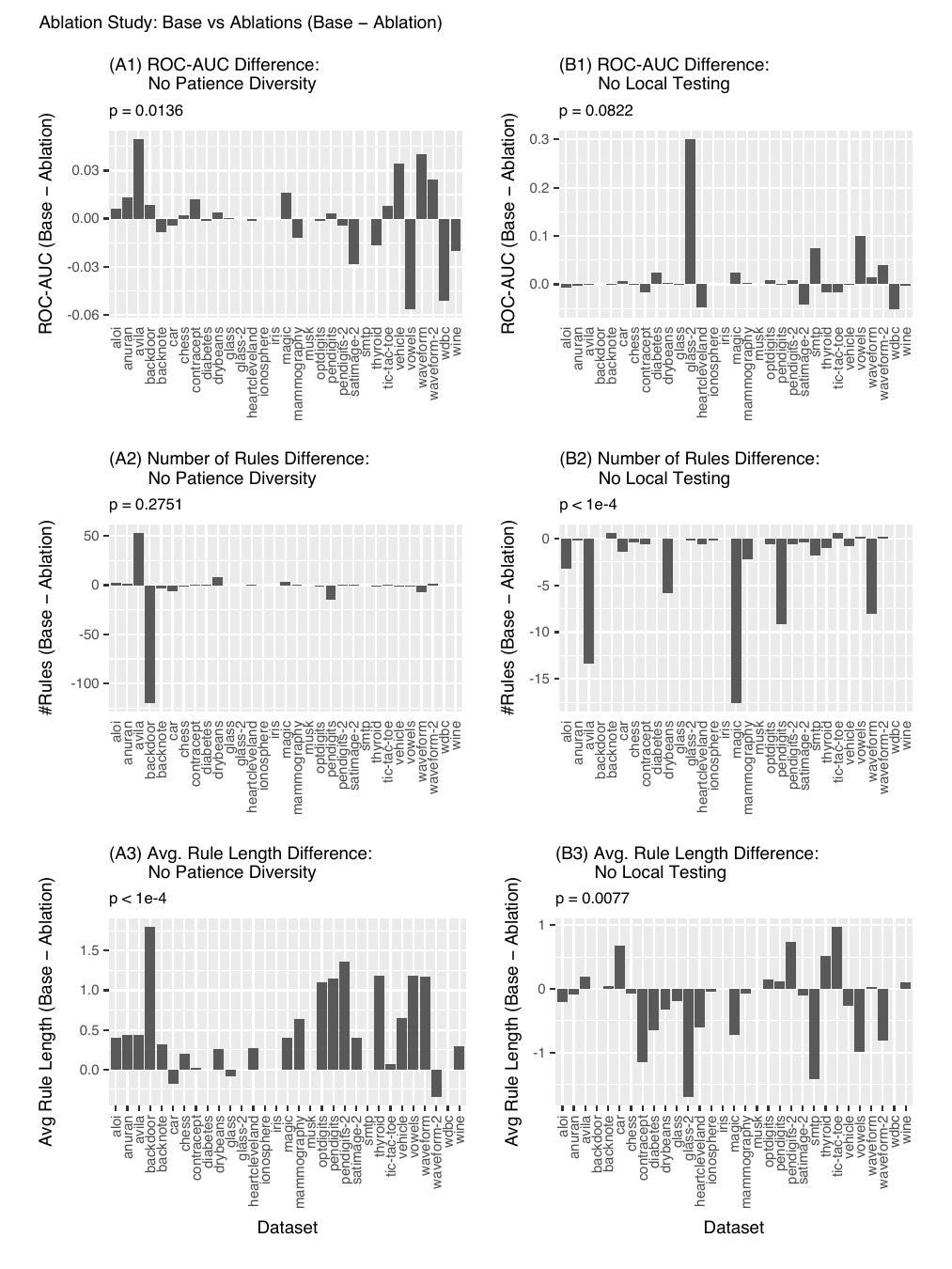}
    \vspace{-10mm}
    \caption{Ablation study comparing the base TURS model with variants. Bars show differences between the base and ablated models across datasets. Positive values indicate higher values for the base model (higher ROC-AUC, higher model complexity). Note that the plots use different y-axis ranges, and that the ROC-AUC difference plot does not show the full ROC-AUC range $(0,1)$.}
    \label{fig:ablation_all}
\end{figure}

We assess the contribution of \crevb{patience-diversified beam search} (Section~\ref{subsec:diverse_beam_search}) by replacing it with a standard beam search strategy. In this ablated variant, beam updates retain the globally top-$W$ candidates ranked by the learning speed score $r(S')$, without partitioning candidates according to their relative coverage reduction $\frac{|c(S')|}{|c(S)|}$. Consequently, the clustering step that enforces diversity over shrinkage ratios is removed, while the MDL-based local test (Equation~\cnew{(}\ref{eq:local_local testing}\cnew{)}) remains unchanged.

Figure~\ref{fig:ablation_all}-A1 reports the difference in ROC-AUC between the base model and this ablated variant (positive values indicate that \crevb{patience diversification} improves ROC-AUC). Overall, enabling \crevb{patience diversification} yields a statistically significant improvement in test ROC-AUC. A one-sided Wilcoxon signed-rank test results in a $p$-value of approximately $0.01$. The \cdel{few}\cnew{two} counterexamples—datasets where removing \crevb{patience diversification} improves ROC-AUC by more than $0.03$—are \cdel{all}\cnew{both} relatively small: {vowels} (1{,}456 instances)\cdel{,}\cnew{ and} {wdbc} (367 instances)\cdel{, and {wine} (178 instances)}. This pattern suggests that smaller sample sizes may benefit from a greedier search strategy, whereas, in general, \crevb{patience diversification} provides a more reliable improvement.

In terms of model complexity, \cdel{disabling}\cnew{enabling} \crevb{patience diversification} \cdel{produces a slight simplification}\cnew{yields slightly longer rules}. While the change in the number of rules (Figure~\ref{fig:ablation_all}-A2) is not statistically significant ($p \approx 0.28$, two-sided Wilcoxon test), the average rule length (Figure~\ref{fig:ablation_all}-A3) increases by less than $1$ on average and by up to approximately $1.75$ at maximum. This effect is consistent with the role of \crevb{patience diversification}: by discouraging overly aggressive coverage reductions, the search explores refinements more gradually, which can yield slightly longer rules. Given that TURS models typically consist of short rules, this increase remains modest and reflects the expected trade-off between rule compactness and improved predictive performance.

\subsection{Ablation III: MDL-based local testing}\label{subsec:exp_local_constraint}
We investigate the effect of the MDL-based local testing heuristic from two complementary perspectives. First, we show that without local testing, the algorithm can fail even on a simple synthetic dataset by learning spurious rules that describe only subsets of the true rule region. Second, we demonstrate that similar failure modes also manifest on real-world benchmark datasets.

\begin{table}[ht]
\centering
\scalebox{0.85}{
\begin{tabular}{llllll}
  \hline
Local testing & \# rules & rule length & ROC-AUC & MDL-based score & train/test prob. diff. \\ 
  \hline
No & 12.48($\pm$1.56) & 5.597($\pm$0.42) & 0.722($\pm$0.02) & 2191.189($\pm$65.91) & 0.049($\pm$0.01) \\ 
  Yes & 1($\pm$0) & 1($\pm$0) & 0.724($\pm$0.01) & 2050.087($\pm$68.88) & 0.007($\pm$0) \\ 
   \hline\end{tabular}
   }
\caption{Results of ablation study on local testing. We report the mean ($\pm$ standard deviation) over $100$ repetitions. } \label{table:ablation_simu}
\end{table}
\begin{figure}[ht]
\centering
	\includegraphics[width=0.7\textwidth]{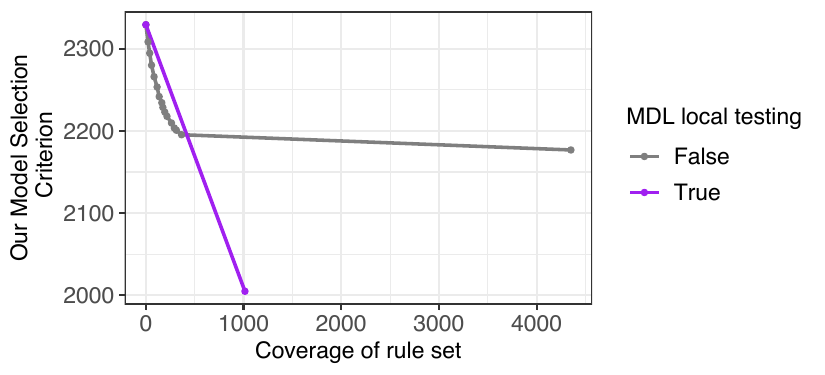}
	\caption{The process of adding rules to the rule set, with and without the local testing heuristic, using the first  dataset among the 100 simulated datasets. Each point represents the status after a single rule is added, with the x-axis representing the coverage of the (potentially incomplete) rule set after adding this rule, and the y-axis representing the MDL-based score. } \label{fig:ablation_simulation}
\end{figure}
\smallskip \noindent \textbf{Failure on a simple synthetic dataset.}
Recall that the MDL-based local test is designed to evaluate the ``potential'' of instances left uncovered when growing a rule. From an optimization perspective, it serves as a look-ahead mechanism that helps avoid poor local minima when optimizing the global MDL-based model selection criterion in Equation~(\ref{eq:global_score}).

We illustrate that without local testing, the algorithm may fail to recover the ground-truth rule set even in a highly controlled setting. Instead, it learns a more complex model and converges to an inferior optimum of the MDL objective.

Consider a simulated dataset generated from a ground-truth rule set consisting of a single rule. The feature vector is $X=(X_1,\ldots,X_{50})$, where all variables are binary. We sample $X_1\sim\mathrm{Ber}(0.2)$ and $X_i\sim\mathrm{Ber}(0.5)$ for $i=2,\ldots,50$. The binary target variable $Y$ is generated according to
\[
Y\mid X_1=1\sim\mathrm{Ber}(0.7), \qquad
Y\mid X_1=0\sim\mathrm{Ber}(0.95).
\]
Thus, $X_1$ alone defines the true underlying rule.

We generate $100$ datasets of size $5{,}000$ and run TURS with and without local testing. As shown in Table~\ref{table:ablation_simu}, disabling local testing leads to a substantially worse MDL-based score, even though the ROC-AUC remains similar. Importantly, the ground-truth model is consistently recovered only when local testing is enabled. Without it, the algorithm learns multiple longer rules involving irrelevant variables, and hence only describes part of what the ground-truth rule describes.

This effect is illustrated in Figure~\ref{fig:ablation_simulation}, which shows the rule-growing process for a single run. Each point corresponds to the rule set after adding one rule, with the x-axis representing coverage and the y-axis the MDL-based score. Although the variant without local testing initially exhibits steeper improvements (the gray curve), it ultimately converges to a worse solution. Without local testing, uncovered instances are effectively ignored during rule growth, leading to suboptimal global optimization.

\smallskip \noindent \textbf{Impact on real-world datasets.}
We next examine whether the failure mode observed in the synthetic setting extends to real-world benchmark datasets. Figures~\ref{fig:ablation_all}-B1 to B3 report the differences between the base TURS model and a variant without local testing.

As shown in Figure~\ref{fig:ablation_all}-B1, removing local testing leads to a \cdel{modest degradation}\cnew{small decrease} in ROC-AUC across datasets\cnew{, which is not statistically significant at the $0.05$ level} (one-sided Wilcoxon test $p\approx0.08$). In contrast, the effect on model complexity is substantial: disabling local testing significantly increases the number of rules (Figure~\ref{fig:ablation_all}-B2, $p<10^{-4}$), while also increasing the average rule length (Figure~\ref{fig:ablation_all}-B3, $p\approx0.008$).

In summary, these results indicate that MDL-based local testing plays a crucial role in guiding the search toward compact and well-generalizing rule sets. While predictive performance may not always deteriorate dramatically without local testing, the resulting models are consistently more complex and less well-regularized, reflecting the same underlying optimization pathology observed in the synthetic example.

\section{Conclusion and Discussion}
We studied the problem of learning truly unordered rule sets from data. In contrast to existing rule set methods that rely on post-hoc conflict resolution mechanisms, we proposed a probabilistic formulation in which \cdel{rule overlaps are permitted only when overlapping rules exhibit similar probabilistic outputs}\cnew{overlaps between rules with dissimilar probabilistic outputs are penalized, so that overlapping rules tend to exhibit similar probabilistic outputs}. Building upon this principle, we introduced the TURS model, which assigns explicit probabilistic semantics to rule overlaps while preserving the self-standing nature of individual rules. Rather than resolving conflicts through explicit or implicit rule ordering, TURS models overlaps directly within the likelihood. Our theoretical analysis provides a formal justification for this formulation, showing that the TURS semantics naturally encourages overlap consistency, since inconsistent overlaps incur an unavoidable increase in the population conditional KL divergence (equivalently, expected log-loss).

We formulated the learning of TURS as a probabilistic model selection problem under the MDL principle. The resulting criterion balances goodness-of-fit and model complexity without requiring regularization hyper-parameters. To optimize this objective, we developed a beam search algorithm augmented with \crevb{patience diversification} and an MDL-based local testing mechanism, which regulate search dynamics by balancing learning speed improvements against coverage reduction.

Empirically, TURS demonstrates highly competitive predictive performance across datasets, \cdel{most frequently achieving the highest ROC-AUC among the compared methods}\cnew{achieving the highest ROC-AUC on 11 out of the 31 datasets}. Moreover, even on datasets where TURS does not attain the top ROC-AUC, \cdel{it consistently exhibits among the smallest gaps}\cnew{its gap} to the best-performing method \cnew{is consistently among the smallest}, indicating strong and stable performance across diverse learning scenarios.

Beyond predictive performance, TURS exhibits strong interpretability characteristics. Interpretability of rule-based models can involve multiple aspects, including whether a model can be comprehended at all, the cognitive effort required to understand it, and other factors affecting human interpretation beyond model complexity, such as cognitive plausibility and the influence of cognitive biases \citep{furnkranz2020cognitive,kliegr2021review}. In this work, we focus on a specific interpretability issue arising from rule ordering, where understanding the prediction associated with a rule may require examining higher-ranked rules. In addition to addressing this conceptual aspect, the learned models also exhibit favorable interpretability properties in practice. The model complexity of the learned models remains compact, while individual rules provide reliable class probability estimates that generalize well to unseen data. This suggests that the probabilistic semantics and MDL-based model selection criterion jointly promote rule sets that balance descriptive accuracy with structural simplicity.

Finally, the learned models exhibit strong overlap consistency. Predictions for instances covered by multiple rules remain stable, and randomly selecting a covering rule results in negligible performance differences. This behavior empirically justifies our way of formalizing rule sets as probabilistic models. 

For future work, we aim to explore TURS as a building block for interactive rule learning with humans in the loop, where self-standing rules can be inspected and modified independently. Extensions to regression with uncertainty, feature construction, subgroup discovery, and explaining black-box models represent promising directions.

\newpage
\appendix

\section{Comparison to Our Previous Work} \label{app:diff}
This paper substantially extends our previous work~\citep{yang2022truly}. We summarize the key improvements in Table~\ref{table:sum_diff} and describe \cnew{them }in detail as follows: compared to~\cdel{\mbox{\citep{yang2022truly}}}\cnew{\citet{yang2022truly}}, this work introduces 1) a fully specified probabilistic model, 2) a principled MDL encoding scheme, 3) a redesigned rule-learning algorithm, and 4) a substantially expanded empirical evaluation. 

\begin{table}[h]
\centering
\begin{tabular}{p{3.2cm}p{5.2cm}p{5.2cm}}
\toprule
\textbf{Aspect} & \textbf{Previous work~\citep{yang2022truly}} & \textbf{This paper} \\
\midrule
Probabilistic model & 1) No probabilistic model formally defined; 2) intuitive probabilistic interpretation; 3) overlaps motivated by uncertainty and exceptionalness & Formally defined probabilistic TURS model without rule ordering with theoretical justification \\
\midrule
Overlap handling & Separate treatment of nested and non-nested overlaps & Unified overlap modeling framework \\
\midrule
Model selection & MDL-based model selection criterion but the code length is excluded to achieve better \cdel{practice}\cnew{practical} performance & MDL-based criterion with \cdel{non-redundant}\cnew{a principled} encoding with the code length of \cnew{the }model included \\
\midrule
Search heuristic & FOIL-like MDL compression gain & Learning speed score (MDL reduction per additionally covered instance) \\
\midrule
Rule learning algorithm & Diverse beam search; surrogate CART-based evaluation for uncovered instances & Redesigned beam search with MDL-based local testing and \crevb{patience-diversified} beam strategy \\
\midrule
Experimental evaluation & Predictive performance and model complexity on 10 datasets & 1) Large-scale predictive study (31 datasets); 2) empirical overlap consistency; 3) interpretability with both model complexity and the generalizability of individual rules' outputs; 4) efficiency (runtimes); 5) large-scale ablation studies \\
\bottomrule
\end{tabular}
\caption{Summary of differences between the previous and present work.} \label{table:sum_diff}
\end{table}

\smallskip \noindent \textbf{Probabilistic formalization and overlap modeling.}
Our previous work introduced truly unordered rule sets primarily through an intuitive perspective, where rule overlaps were interpreted as expressing predictive uncertainty arising from rule interactions, and these intuitions were incorporated implicitly into the design of rule quality measures and the model selection criterion. While this view motivated the algorithmic framework, the probabilistic semantics of the model were not fully specified.

In contrast, the present work formulates TURS as a fully defined probabilistic model that assigns predictions without introducing any rule ordering. Overlaps are handled directly within the probabilistic model rather than through heuristic interpretations, yielding a coherent description of rule interactions. Importantly, this formalization is supported by theoretical analysis \cdel{establishing conditions under which overlapping rules improve or degrade predictive risk}\cnew{showing that overlaps never reduce the population risk relative to the corresponding partition model, and quantifying the excess risk incurred when overlapping rules disagree}. Additionally, we unify the treatment of nested and non-nested overlaps within a single modeling framework (in previous work, these two cases were handled separately). This eliminates the need for overlap-type-specific mechanisms, reduces computational overhead, and reflects the empirical observation that nested overlaps occur only rarely in the learned models.

\smallskip \noindent \textbf{Model selection criterion.}
Interpreting the MDL criterion as penalized maximum likelihood, the penalty consists of (1) the code length of the model and (2) the regret term. In our previous implementation, we excluded the model code length term, as this often yielded better predictive performance, albeit at the cost of increased model complexity. This exclusion was partly motivated by known limitations of model encoding schemes, which may not fully capture redundancy in the rule set model class~\citep{quinlan2014c4}.

In the present work, we introduce a principled encoding scheme tailored to probabilistic rule-based models. In contrast to how MDL penalties were applied in classical rule learning approaches such as C4.5~\citep{quinlan1995mdl} and RIPPER~\citep{cohen1995ripper}, our encoding is explicitly constructed \cdel{to be non-redundant, in the sense that it induces a complete prior distribution over the considered model class satisfying Kraft's inequality}\cnew{such that code lengths correspond to prior probabilities over the considered model class via Kraft's inequality}. This ensures that model complexity is penalized in a theoretically coherent manner rather than through heuristic or partially redundant encodings. Combined with the improved optimization strategy, this allows us to retain the model code length term, yielding simpler rule sets without sacrificing predictive accuracy.

\smallskip \noindent \textbf{Algorithmic heuristics.}
During a follow-up real-world case study in health care~\citep{yang2025interpretable}, we found that the prediction performance of our previous method was occasionally unsatisfactory. A closer inspection suggested that the main bottleneck was the heuristics used to guide the combinatorial search for growing rules.

In the earlier work, rule growth was driven by an MDL-based FOIL-like compression gain motivated by FOIL's information gain~\citep{furnkranz2012foundations}. In hindsight, this heuristic exhibits limitations in the multi-class setting: when generalized beyond binary targets, FOIL-style information gain introduces an entropy-reduction bias, systematically favoring refinements that reduce empirical entropy relative to the previous step. This behavior can be harmful on noisy data and/or under strong class imbalance, as it may encourage overly specific rules that improve short-term purity but degrade generalization.

To address this limitation, the present work adopts a different search principle centered on the \emph{learning speed score}. Specifically, adding a rule is viewed as a step in the rule-set search space, and the next rule is selected by maximizing the decrease in the global MDL-based score per additionally covered instance (Equation~(\ref{eq:learning_speed})). This reformulation better aligns rule selection with steepest descent behavior in the MDL objective with respect to coverage growth.
Building on this principle, we redesign the rule-learning algorithm using a beam search strategy that integrates two key mechanisms. First, we incorporate an \emph{MDL-based local testing} criterion that explicitly evaluates the modeling trade-off induced by adding a literal. The test compares describing the data covered by a rule using a single region versus a depth-one split into the candidate rule and the left-out subset (Equation~(\ref{eq:local_local testing})), following the analogy between MDL and hypothesis testing (e.g.,~\citealp[Ch.~14.3]{grunwald2007minimum}). Second, we introduce a \crevb{\emph{patience-diversified beam search}} strategy that diversifies candidate rules according to their relative coverage reduction, inspired by the notion of patience in PRIM~\citep{friedman1999bump}. Together, these mechanisms explicitly account for the modeling cost associated with left-out instances and regulate the effect of coverage shrinkage during rule growth, yielding a more stable and globally aligned search procedure.

\smallskip \noindent \textbf{Experimental scope.}
The experimental focus has been substantially broadened. While the previous paper primarily demonstrated that truly unordered rule sets can achieve predictive performance comparable to methods relying on ad-hoc conflict resolution\cdel{ only}, the present work investigates:
\begin{enumerate}
\item Predictive performance across a large and diverse collection of datasets;
\item The empirical degree of overlap consistency, measured via (1) random rule selection within overlapping regions and (2) the $\ell_1$ distance between probabilistic outputs of overlapping rules;
\item Interpretability in terms of (1) model complexity and (2) the generalization behavior of rule-level probabilistic estimates;
\item Algorithm efficiency measured via runtimes;
\item A comprehensive ablation study analyzing the effect of key algorithmic components.
\end{enumerate}

\smallskip \noindent \textbf{Summary of differences.}
Compared to our previous work~\citep{yang2022truly}, the present paper introduces several substantial advancements. Most importantly, we replace the earlier intuition-driven interpretation of overlaps with a fully specified probabilistic model of truly unordered rule sets, supported by theoretical analysis. We further propose a principled\cdel{, non-redundant} MDL encoding scheme, redesign the rule-learning algorithm around the learning speed score with MDL-based local testing and \crevb{patience-diversified} beam search, and significantly broaden the empirical evaluation, including predictive performance, overlap consistency, interpretability, efficiency, and ablation studies.

\section{Proofs} \label{app:proofs}
\crevb{This appendix collects the proofs of all formal results of Sections~\ref{subsec:formal_justification} and~\ref{sec:model_selection}.}

\begin{proofof}{Lemma~\ref{lem:kl_bias_variance}}
For any $x\in R$ and any $y\in\mathcal Y$, write
\[
\log\frac{p_x(y)}{q(y)}
=
\log\frac{p_x(y)}{p_R(y)}
+
\log\frac{p_R(y)}{q(y)}.
\]
Substituting into the definition of the KL divergence gives
\[
D_{\mathrm{KL}}(p_x\|q)
=
D_{\mathrm{KL}}(p_x\|p_R)
+
\sum_{y\in\mathcal Y} p_x(y)\log\frac{p_R(y)}{q(y)}.
\]
Taking conditional expectation with respect to $X$ given $X\in R$ yields
\[
\mathbb E\!\left[D_{\mathrm{KL}}(p_X\|q)\,\middle|\,X\in R\right]
=
\mathbb E\!\left[D_{\mathrm{KL}}(p_X\|p_R)\,\middle|\,X\in R\right]
+
\sum_{y\in\mathcal Y}
\mathbb E\!\left[p_X(y)\,\middle|\,X\in R\right]
\log\frac{p_R(y)}{q(y)}.
\]
Since $\mathbb E[p_X(y)\mid X\in R]=\cnew{\Ptrue}(Y=y\mid X\in R)=p_R(y)$, the second term equals
\[
\sum_{y\in\mathcal Y} p_R(y)\log\frac{p_R(y)}{q(y)}
=
D_{\mathrm{KL}}(p_R\|q),
\]
which proves the claim.
\end{proofof}

\begin{proofof}{Theorem~\ref{thm:overlap_harms_full_rigorous}}
Write $q_R$ for the distribution used by $Q_M$ on region $R$
(i.e., $q_A=p_{\Gamma_1}$, $q_B=p_{\Gamma_2}$, $q_C=p_U$), so that
$Q_M(\cdot\mid X)=q_R$ and $Q_{\mathrm{part}}(\cdot\mid X)=p_R$ for $X\in R$.
Since $\{A,B,C\}$ is a partition of $\mathcal X$,
\[
\mathbb E\!\left[
D_{\mathrm{KL}}\!\big(p_X\|Q_M(\cdot\mid X)\big)
\right]
=
\sum_{R\in\{A,B,C\}}
\Ptrue(X\in R)\,
\mathbb E\!\left[
D_{\mathrm{KL}}(p_X\|q_R)\,\middle|\,X\in R
\right].
\]
For each region $R\in\{A,B,C\}$, Lemma~\ref{lem:kl_bias_variance} applied with
$q=q_R$ gives
\[
\mathbb E\!\left[
D_{\mathrm{KL}}(p_X\|q_R)\,\middle|\,X\in R
\right]
=
\mathbb E\!\left[
D_{\mathrm{KL}}(p_X\|p_R)\,\middle|\,X\in R
\right]
+
D_{\mathrm{KL}}(p_R\|q_R).
\]
Combining these two identities yields
\[
\begin{aligned}
\mathbb E\!\left[
D_{\mathrm{KL}}\!\big(p_X\|Q_M(\cdot\mid X)\big)
\right]
&=
\sum_{R\in\{A,B,C\}}
\Ptrue(X\in R)\,
\mathbb E\!\left[
D_{\mathrm{KL}}(p_X\|p_R)\,\middle|\,X\in R
\right] \\
&\quad+
\sum_{R\in\{A,B,C\}}
\Ptrue(X\in R)\,
D_{\mathrm{KL}}(p_R\|q_R).
\end{aligned}
\]
The first sum equals
$\mathbb E\!\left[
D_{\mathrm{KL}}\!\big(p_X\|Q_{\mathrm{part}}(\cdot\mid X)\big)
\right]$,
and the second sum expands to the stated excess-risk terms.
The inequality follows since each KL divergence is nonnegative.
\end{proofof}

\begin{proofof}{Corollary~\ref{cor:full_agreement}}
Equality in Theorem~\ref{thm:overlap_harms_full_rigorous} forces
\[
D_{\mathrm{KL}}(p_A\|p_{\Gamma_1})
=
D_{\mathrm{KL}}(p_B\|p_{\Gamma_2})
=
D_{\mathrm{KL}}(p_C\|p_U)
=
0,
\]
since each term is nonnegative and each region has positive probability.
Thus $p_A=p_{\Gamma_1}$, $p_B=p_{\Gamma_2}$, and $p_C=p_U$.
Since $\Gamma_1=A\;\dot\cup\;C$ with both parts having positive mass,
\[
p_{\Gamma_1}=\cfix{\mixw_A} p_A+\cfix{\mixw_C} p_C
\qquad\text{with }\cfix{\mixw_C := \Ptrue(X \in C \mid X \in \Gamma_1)}\in(0,1)\cfix{\text{ and }\mixw_A := 1-\mixw_C},
\]
which implies $p_C=p_A$.
Similarly, $\Gamma_2=B\;\dot\cup\;C$ implies $p_C=p_B$.
Thus $p_A=p_B=p_C$, and consequently $p_{\Gamma_1}=p_{\Gamma_2}$.
\end{proofof}

\begin{proofof}{Theorem~\ref{thm:disagreement_excessKL_crisp_fixed}}
Write $\mu_R:=\Ptrue(X\in R)$ and $\lambda_R:=\Ptrue(X\in R\mid X\in U)$ for $R\in\{A,B,C\}$.

\medskip \noindent \emph{Step 1: reduce to the overlap region.}
\cfix{By the exact decomposition in Theorem~\ref{thm:overlap_harms_full_rigorous} and nonnegativity of KL,
\[
\mathbb E\!\left[D_{\mathrm{KL}}\!\bigl(p_X\|Q_M(\cdot\mid X)\bigr)\right]
-
\mathbb E\!\left[D_{\mathrm{KL}}\!\bigl(p_X\|Q_{\mathrm{part}}(\cdot\mid X)\bigr)\right]
\ge
\mu_C\,D_{\mathrm{KL}}(p_C\|p_U),
\]
where $\mu_C\,D_{\mathrm{KL}}(p_C\|p_U)$ is the last term of that decomposition. Both expectations on the left-hand side are finite: for each region $R\in\{A,B,C\}$, we have $\mathbb E[D_{\mathrm{KL}}(p_X\|p_R)\mid X\in R]\le \log|\mathcal Y|$, and the distribution that $Q_M$ uses on $R$ is a mixture that puts positive weight on $p_R$, so that the KL divergence of $p_R$ from it is finite as well.}
Using (H1),
\begin{equation}
\label{eq:reduceC_fixed}
\mathbb E\!\left[D_{\mathrm{KL}}\!\bigl(p_X\|Q_M(\cdot\mid X)\bigr)\right]
\cfix{\;-\;
\mathbb E\!\left[D_{\mathrm{KL}}\!\bigl(p_X\|Q_{\mathrm{part}}(\cdot\mid X)\bigr)\right]}
\ge
\cfix{\ovmass}\,D_{\mathrm{KL}}(p_C\|p_U).
\end{equation}

\medskip \noindent \emph{Step 2: lower bound $D_{\mathrm{KL}}(p_C\|p_U)$ via total variation.}
We use the convention
\[
\|p-q\|_{\mathrm{TV}}
:=
\sum_{y\in\mathcal Y} \frac{1}{2}|p(y)-q(y)|.
\]
With this definition, Pinsker’s inequality \citep{cover2006elements} gives
\[
D_{\mathrm{KL}}(p_C\|p_U)
\;\ge\;
2\|p_C-p_U\|_{\mathrm{TV}}^2.
\]
Moreover, for any $y\in\mathcal Y$,
\[
\|p_C-p_U\|_{\mathrm{TV}}
\;\ge\;
\frac{1}{2}|p_C(y)-p_U(y)|.
\]
In particular, with the $y^\star$ defined in (H4),
\[
D_{\mathrm{KL}}(p_C\|p_U)
\;\ge\;
\frac{1}{2}|p_C(y^\star)-p_U(y^\star)|^2.
\]

\medskip \noindent \emph{Step 3: use the mixture structure of $p_U$ without dropping terms.}
Since $U=A\;\dot\cup\;B\;\dot\cup\;C$,
\[
p_U
=
\lambda_A p_A+\lambda_B p_B+\lambda_C p_C,
\qquad\lambda_C=1-\lambda_A-\lambda_B.
\]
Thus
\[
p_U(y^\star)-p_C(y^\star)
=
\lambda_A\bigl(p_A(y^\star)-p_C(y^\star)\bigr)
+
\lambda_B\bigl(p_B(y^\star)-p_C(y^\star)\bigr),
\]
and the reverse triangle inequality yields
\[
|p_U(y^\star)-p_C(y^\star)|
\ge
\lambda_B|p_B(y^\star)-p_C(y^\star)|
-\lambda_A|p_A(y^\star)-p_C(y^\star)|.
\]
By (H2), $\lambda_B\ge\eta$, hence $\lambda_A\le 1-\eta$, giving
\begin{equation}
\label{eq:mixture_lb_fixed}
|p_U(y^\star)-p_C(y^\star)|
\ge
\eta\,|p_B(y^\star)-p_C(y^\star)|
-(1-\eta)\,|p_A(y^\star)-p_C(y^\star)|.
\end{equation}

\medskip \noindent \emph{Step 4: lower bound $|p_B(y^\star)-p_C(y^\star)|$ by disagreement.}
By \cnew{the }triangle inequality,
\[
|p_B(y^\star)-p_C(y^\star)|
\ge
|p_{\Gamma_2}(y^\star)-p_{\Gamma_1}(y^\star)|
-|p_B(y^\star)-p_{\Gamma_2}(y^\star)|
-|p_C(y^\star)-p_{\Gamma_1}(y^\star)|.
\]
\cfix{By (H4), the larger of $p_{\Gamma_1}(y^\star)$ and $p_{\Gamma_2}(y^\star)$ is at least $e^{\sigma}$ times the smaller one, and the smaller one equals $m$; hence, whichever of the two is larger,
$|p_{\Gamma_2}(y^\star)-p_{\Gamma_1}(y^\star)|\ge (e^\sigma-1)\,m = m(e^\sigma-1)$.}
Pinsker plus (H5) gives
$|p_B(y^\star)-p_{\Gamma_2}(y^\star)|\le \sqrt{2\alpha_2}$,
and Pinsker plus (H3) gives
$|p_C(y^\star)-p_{\Gamma_1}(y^\star)|\le \sqrt{2\alpha_C}$.
Therefore
\[
|p_B(y^\star)-p_C(y^\star)|
\ge
m(e^\sigma-1)-\sqrt{2{\alpha_2}}-\sqrt{2{\alpha_C}}
=\Delta.
\]

\medskip \noindent \emph{Step 5: upper bound $|p_A(y^\star)-p_C(y^\star)|$ using (H6) and (H3).}
By \cnew{the }triangle inequality and Pinsker,
\[
|p_A(y^\star)-p_C(y^\star)|
\le
|p_A(y^\star)-p_{\Gamma_1}(y^\star)|
+
|p_C(y^\star)-p_{\Gamma_1}(y^\star)|
\le
\sqrt{2{\alpha_1}}+\sqrt{2{\alpha_C}}
=\delta_A.
\]

\medskip \noindent \emph{Step 6: conclude.}
Combining the bounds from Steps~3--5 yields
\[
|p_U(y^\star)-p_C(y^\star)|
\ge
\eta\,\Delta-(1-\eta)\,\delta_A.
\]
Therefore, by Step~2,
\[
D_{\mathrm{KL}}(p_C\|p_U)
\ge
\tfrac12\big(\eta\,\Delta-(1-\eta)\,\delta_A\big)_+^2.
\]
Combining with \eqref{eq:reduceC_fixed} gives the stated bound.
\end{proofof}

\begin{proofof}{Proposition~\ref{prop:apprnml_exact}}
The numerators are the same, and hence we only need to show that the denominators are the same. 

To simplify the notation, we denote $\hat{P}_S := P_{\hat{\vec{\beta}}_S}$. \cfix{Let $M = \{S_1, \ldots, S_K\}$ and let $S_0$ be the else rule. As no two rules in $M$ cover a common instance, the coverages $c(S_0), c(S_1), \ldots, c(S_K)$ partition the dataset, so that $y^n$ decomposes into the blocks $y^{S_0}, y^{S_1}, \ldots, y^{S_K}$, and}
{\small
\begin{equation}
\begin{split}
		& \sum_{y^n \in \mathcal{Y}^n} P_{M, \hat{\theta}(x^n, y^n)}(y^n|x^n)
        = \sum_{y^n} \prod_{S\in M \cfix{{}\cup \{S_0\}}} \hat{P}_S(\crev{y^S|x^S}) \\
		& = \sum_{y^n} \hat{P}_{S_{\cfix{0}}}(y^{S_{\cfix{0}}}|x^{S_{\cfix{0}}}) \cdot \hat{P}_{S_{\cfix{1}}}(y^{S_{\cfix{1}}}|x^{S_{\cfix{1}}}) \ldots \cdot \hat{P}_{S_K}(y^{S_{K}}|x^{S_{K}}) \\
		& = \sum_{y^{S_{\cfix{0}}}} \ldots \sum_{y^{S_{K}}}
        \left(\hat{P}_{S_{\cfix{0}}}(y^{S_{\cfix{0}}}|x^{S_{\cfix{0}}}) \cdot  \hat{P}_{S_{\cfix{1}}}(y^{S_{\cfix{1}}}|x^{S_{\cfix{1}}}) \ldots \cdot \hat{P}_{S_K}(y^{S_{K}}|x^{S_{K}}) \right)\\
		& = \Bigg(\sum_{y^{S_{\cfix{0}}}} \ldots \sum_{y^{S_{K-1}}}
        \hat{P}_{S_{\cfix{0}}}(y^{S_{\cfix{0}}}|x^{S_{\cfix{0}}}) \cdot \ldots \cdot \hat{P}_{S_{K-1}} (y^{S_{K-1}}|x^{S_{K-1}}) \Bigg)
        \left(\sum_{y^{S_{K}}}  \hat{P}_{S_K}(y^{S_{K}}|x^{S_{K}})\right) \\
		& \ldots \\
		& = \left(\sum_{y^{S_{\cfix{0}}}}  \hat{P}_{S_{\cfix{0}}}(y^{S_{\cfix{0}}}|x^{S_{\cfix{0}}})\right)
        \cdot \ldots \cdot
        \left(\sum_{y^{S_{K}}}  \hat{P}_{S_K}(y^{S_{K}}|x^{S_{K}})\right)\\
		& = \prod_{S \in M \cfix{{}\cup \{S_0\}}} \sum_{y^{S}}  \hat{P}_{S}(y^{S}|x^{S}) \\
		& = \prod_{S \in M \cfix{{}\cup \{S_0\}}} \mathcal{R}(\crev{|c(S)|}, |\mathcal{Y}|),
\end{split}
\end{equation}
}
which completes the proof.
\end{proofof}

\begin{proofof}{Proposition~\ref{prop:apprnml_bic}}
	The proof directly follows from Theorem~3 of~\cdel{\mbox{\citep{silander2008factorized}}}\cnew{\citet{silander2008factorized}}. Firstly, it has been proven that $\log \mathcal{R}(\crev{|c(S)|}, |\mathcal{Y}|) = \frac{|\mathcal{Y}| - 1}{2} \log \crev{|c(S)|} + \mathcal{O}(1)$~\citep{rissanen1996fisher}. Next, under the mild assumption that \crev{$|c(S)|$} grows linearly with the full sample size $n$, we have $\log \crev{|c(S)|} = \log ((\gamma + o(1))n) = \log n + \mathcal{O}(1)$\cnew{ for some constant $\gamma > 0$}. Hence, \cfix{as the product runs over the $K$ rules in $M$ and the else rule $S_0$,} $\log \prod_{S \in M \cfix{{}\cup \{S_0\}}} \mathcal{R}(\crev{|c(S)|}, |\mathcal{Y}|) = \sum_{S \cfix{{}\in M \cup \{S_0\}}}\log \mathcal{R}(\crev{|c(S)|}, |\mathcal{Y}|)$ $= \frac{\cfix{(K+1)}(|\mathcal{Y}| - 1)}{2} \log n + \mathcal{O}(1)$, which completes the proof.
\end{proofof}

\newpage
\vskip 0.2in
\bibliography{main}

\end{document}

%% file: fig_intro_example.tex
\begin{figure}[t!]
\centering
\fbox{%
\begin{minipage}{0.95\textwidth}
\vspace{2pt}
\textbf{Three probabilistic rules for flight-delay prediction}
\smallskip

{\small
\setlength{\tabcolsep}{3pt}
\begin{tabular}{@{}lll@{}}
$S_1$: & \textbf{IF} season = winter \textbf{AND} precipitation = snow & \textbf{THEN} $P(\text{delayed}) = 0.72$\\[2pt]
$S_2$: & \textbf{IF} departure 17:00--21:00 \textbf{AND} airport is a large hub & \textbf{THEN} $P(\text{delayed}) = 0.35$\\[2pt]
$S_3$: & \textbf{IF} aircraft's previous flight late \textbf{AND} turnaround $< 40$ min & \textbf{THEN} $P(\text{delayed}) = 0.68$\\[2pt]
& \textbf{ELSE} & \phantom{\textbf{THEN}} $P(\text{delayed}) = 0.15$
\end{tabular}}

\medskip
\textbf{One concrete flight.} A flight departs from a large hub airport at 18:30 on December 20th, in snowfall. It satisfies the conditions of \emph{both} $S_1$ and $S_2$, which give very different probabilities ($0.72$ vs.\ $0.35$).

\end{minipage}}
\caption{\crevb{Three probabilistic rules and a flight covered by two of them, with conflicting probabilities. 
}}
\label{fig:intro_example}
\end{figure}

%% file: fig_running_example.tex
\begin{figure}[t]
\centering
\newcommand{\panelbase}{%
  \draw[thick] (0,0) rectangle (1,1);
  \draw[gray!70] (0.6,0) -- (0.6,1);
  \draw[gray!70, dashed] (0,0.5) -- (1,0.5);
}
\begin{tikzpicture}[scale=3.0]
  \fill[blue!16] (0,0) rectangle (0.6,1);
  \panelbase
  \draw[blue!55!black, thick] (0,0) rectangle (0.6,1);
  \fill (0.3,0.8) circle (0.02); \node[above right, font=\scriptsize] at (0.3,0.8) {$x$};
  \node[below, font=\scriptsize, gray] at (0.6,0) {$0.6$};
  \node[left, font=\scriptsize, gray] at (0,0.5) {$0.5$};
  \node[below, font=\footnotesize, align=center] at (0.5,-0.14) {(a) $x$ covered by $S_1$ only:\\ modeled by $P_{\vec{\beta}_1}(Y \mid X \in \Gamma_1)$};
\end{tikzpicture}\hfill
\begin{tikzpicture}[scale=3.0]
  \fill[violet!14] (0,0) rectangle (0.6,1);
  \fill[violet!14] (0,0) rectangle (1,0.5);
  \panelbase
  \draw[violet!70!black, thick] (0,1) -- (0.6,1) -- (0.6,0.5) -- (1,0.5) -- (1,0) -- (0,0) -- cycle;
  \fill (0.3,0.25) circle (0.02); \node[above right, font=\scriptsize] at (0.3,0.25) {$x$};
  \node[below, font=\footnotesize, align=center] at (0.5,-0.14) {(b) $x$ covered by $S_1$ and $S_2$ (i.e., \cfix{$\Gamma_1 \cap \Gamma_2$}):\\ modeled by $P_{\vec{\beta}_{\{1,2\}}}(Y \mid X \in \Gamma_1 \cup \Gamma_2)$};
\end{tikzpicture}\hfill
\begin{tikzpicture}[scale=3.0]
  \fill[gray!22] (0.6,0.5) rectangle (1,1);
  \panelbase
  \fill (0.82,0.78) circle (0.02); \node[above right, font=\scriptsize] at (0.82,0.78) {$x$};
  \node[below, font=\footnotesize, align=center] at (0.5,-0.14) {(c) $x$ covered by no rule:\\ modeled by $P_{\vec{\beta}_0}(Y \mid X \in \Gamma_0)$};
\end{tikzpicture}
\caption{\crevb{Illustrative example: rule set $M=\{S_1, S_2\}$ over $\mathcal{X}=[0,1]^2$, with $S_1: X_1 < 0.6 \rightarrow P_{ \vec{\beta}_1}(Y)$ (region $\Gamma_1$, left of the vertical line) and $S_2: X_2 < 0.5 \rightarrow P_{\vec{\beta}_2}(Y)$ (region $\Gamma_2$, below the horizontal line). Each instance $x$ is modeled by the distribution conditioned on exactly one region (shaded): (a) the \emph{whole} region of the covering rule when $x$ is covered by a single rule, (b) the \emph{union} of the covering rules' regions when $x$ lies in an overlap, and (c) the else region $\Gamma_0$ when $x$ is covered by no rule.}}
\label{fig:running_example}
\end{figure}

%% file: alg_nex_rule.tex
\begin{algorithm}[!b]
\caption{Learn a single rule} \label{alg:find_next_rule}
\begin{algorithmic}[1]
\REQUIRE Rule set ${M}$, dataset $(x^n, y^n)$, beam width $W$\cnew{, patience budget $K_{stop}$};
\ENSURE The next rule $S^*$
\STATE $\text{all\_candidate\_rules} \leftarrow [\,\,\, ]$
\STATE $\text{rules\_for\_next\_iter} \leftarrow [ \emptyset ]$ \COMMENT{Initialize the rule with an ``empty" condition}
\WHILE{TRUE}
 \STATE $\text{beam} \leftarrow [\,\,\,]$ \COMMENT{Initialize the beam for the beam search}
	\FOR{$S$ in $\text{rules\_for\_next\_iter}$}
		\STATE $\text{rule\_candidates} \leftarrow \text{generate\_candidates}(S)$ \COMMENT{Generate candidates by enumerating literals that can be appended to the rule}
    \FOR{$w \,\,\, \text{in} \,\,\, 1:W$}
    \STATE $\text{candidates\_cluster\_w} \leftarrow \{S' \in \text{rule\_candidates}$: $\frac{|c(S')|}{|c(S)|}\in [\frac{w-1}{W}, \frac{w}{W}) \} $ \COMMENT{Cluster the candidates, which is defined in Step 3} 
    \STATE $\text{candidates\_cluster\_w} \leftarrow \text{MDL\_local\_testing(candidates\_cluster\_w)}$ \COMMENT{Filter by the MDL local test, which is defined in Step 4}
    \STATE $\text{top\_candidate\_in\_cluster} \leftarrow \arg\max_{S' \in \text{candidates\_cluster\_w}} r(S')$ \COMMENT{Pick the rule within the cluster that maximizes the $r(.)$ score}
    \STATE $\text{beam.append(top\_candidate\_in\_cluster)}$ 
    \ENDFOR    		
	\ENDFOR	\IF{$\text{stopping\_criterion\_is\_met}$}
    \STATE{}\COMMENT{As explained in Step 7, the stopping criterion is defined as the $K_{stop}$-th time in a row that the best candidate in this beam has worse (lower) $r(.)$ score than that of the previous beam}
    \RETURN \text{the rule $S^*$ that has the highest} r(.) \text{in all\_candidate\_rules}
	\ELSE
            \STATE $\text{beam }\leftarrow \text{$W$ candidates selected in beam}$ \COMMENT{reduce the number of rules from (potentially) $W^2$ to $W$, as explained in Step 6}
		\STATE $\text{all\_candidate\_rules.append(beam)}$
        \STATE $\text{rules\_for\_next\_iter} \leftarrow \text{beam}$
	\ENDIF
\ENDWHILE
\end{algorithmic}
\end{algorithm}

%% file: data_info_table.tex
\begin{table}[ht]
\centering
\scalebox{0.95}{
\begin{tabular}{l|rrrrr}
  \hline
Data & \# rows & \# columns & \# classes & max. class prob. & min. class prob. \\ 
  \hline
  \textit{aloi} & 49534 &   28 &    2 & 0.970 & 0.030 \\ 
  \textit{backdoor} & 95329 &  197 &    2 & 0.976 & 0.024 \\ 
  backnote & 1372 &    5 &    2 & 0.555 & 0.445 \\ 
  chess & 3196 &   37 &    2 & 0.522 & 0.478 \\ 
  diabetes &  768 &    9 &    2 & 0.651 & 0.349 \\ 
  \textit{glass-2} &  214 &    8 &    2 & 0.958 & 0.042 \\ 
  ionosphere &  351 &   35 &    2 & 0.641 & 0.359 \\ 
  magic & 19020 &   11 &    2 & 0.648 & 0.352 \\ 
  \textit{mammography} & 11183 &    7 &    2 & 0.977 & 0.023 \\ 
  \textit{musk} & 3062 &  167 &    2 & 0.968 & 0.032 \\ 
  optdigits & 5216 &   65 &    2 & 0.971 & 0.029 \\ 
  pendigits-2 & 6870 &   17 &    2 & 0.977 & 0.023 \\ 
  \textit{satimage-2} & 5803 &   37 &    2 & 0.988 & 0.012 \\ 
  \textit{smtp} & 95156 &    4 &    2 & 1.000 & 0.000 \\ 
  \textit{thyroid} & 3772 &    7 &    2 & 0.975 & 0.025 \\ 
  tic-tac-toe &  958 &   10 &    2 & 0.653 & 0.347 \\ 
  \textit{vowels} & 1456 &   13 &    2 & 0.966 & 0.034 \\ 
  \textit{waveform-2} & 3443 &   22 &    2 & 0.971 & 0.029 \\ 
  \textit{wdbc} &  367 &   31 &    2 & 0.973 & 0.027 \\ 
  \hline
  
  anuran & 7195 &   24 &    4 & 0.614 & 0.009 \\ 
  avila & 20867 &   11 &   12 & 0.411 & 0.001 \\ 
  car & 1728 &    7 &    4 & 0.700 & 0.038 \\ 
  contracept & 1473 &   10 &    3 & 0.427 & 0.226 \\ 
  drybeans & 13611 &   17 &    7 & 0.261 & 0.038 \\ 
  glass &  214 &   11 &    6 & 0.355 & 0.042 \\ 
  heartcleveland &  303 &   14 &    5 & 0.541 & 0.043 \\ 
  iris &  150 &    5 &    3 & 0.333 & 0.333 \\ 
  pendigits & 7494 &   17 &   10 & 0.104 & 0.096 \\ 
  vehicle &  846 &   19 &    4 & 0.258 & 0.235 \\ 
  waveform & 5000 &   22 &    3 & 0.339 & 0.329 \\ 
  wine &  178 &   14 &    3 & 0.399 & 0.270 \\ 
   \hline
\end{tabular}
}
\caption{Datasets for binary (top) and multi-class classification (bottom), publicly available on the UCI repository~\citep{Dua:2019:UCI} and the ADBench Python package~\citep{han2022adbench}; datasets from the latter are marked in \textit{Italic}. We use the maximum and minimum of the marginal class probabilities to indicate the degree of class imbalance. } 
\label{table:data_info}
\end{table}

%% file: table_roc_auc.tex
\begin{table}[ht]
\scalebox{0.85}{
\centering
\begin{tabular}{l|cccccccc|l}
  \hline
Data & BRS & C45 & CART & CLASSY & Ripper & CN2 & DRS & IDS & TURS\tiny{(diff to best)} \\ 
  \hline
aloi & 0.519 & 0.398 & 0.621 & \textbf{0.654} & 0.485 & 0.569 & --- & 0.509 & 0.617 \tiny{(-0.037)} \\ 
  backdoor & 0.917 & 0.990 & 0.979 & 0.996 & 0.976 & \textbf{0.997} & --- & --- & \textbf{0.997} \\ 
  backnote & 0.957 & 0.987 & 0.983 & 0.990 & 0.982 & \textbf{0.993} & 0.988 & 0.776 & 0.983 \tiny{(-0.01)} \\ 
  chess & 0.957 & \textbf{0.998} & 0.995 & 0.992 & 0.995 & 0.532 & 0.809 & 0.677 & 0.990 \tiny{(-0.008)} \\ 
  diabetes & 0.725 & 0.710 & 0.667 & 0.737 & 0.641 & 0.709 & 0.727 & 0.594 & \textbf{0.748} \\ 
  glass-2 & 0.676 & 0.890 & 0.790 & 0.730 & 0.793 & 0.941 & 0.926 & 0.912 & \textbf{0.949} \\ 
  ionosphere & 0.802 & 0.882 & 0.851 & 0.886 & 0.911 & \textbf{0.941} & 0.712 & 0.786 & 0.899 \tiny{(-0.042)} \\ 
  magic & 0.767 & 0.869 & 0.799 & \textbf{0.888} & 0.819 & 0.698 & 0.774 & 0.506 & 0.875 \tiny{(-0.013)} \\ 
  mammography & 0.644 & 0.817 & 0.730 & 0.890 & 0.582 & 0.891 & 0.857 & 0.535 & \textbf{0.899} \\ 
  musk & \textbf{1.000} & 0.995 & \textbf{1.000} & \textbf{1.000} & \textbf{1.000} & \textbf{1.000} & --- & \textbf{1.000} & \textbf{1.000} \\ 
  optdigits & 0.897 & 0.959 & 0.942 & 0.986 & 0.966 & \textbf{0.992} & --- & 0.960 & 0.977 \tiny{(-0.015)} \\ 
  pendigits-2 & 0.938 & 0.986 & 0.964 & 0.974 & 0.973 & \textbf{0.996} & 0.948 & 0.914 & 0.984 \tiny{(-0.012)} \\ 
  satimage-2 & 0.922 & 0.914 & 0.915 & 0.929 & \textbf{0.964} & \textbf{0.964} & 0.699 & 0.867 & 0.895 \tiny{(-0.069)} \\ 
  smtp & 0.596 & 0.930 & 0.965 & 0.905 & 0.950 & 0.853 & 0.889 & 0.879 & \textbf{0.972} \\ 
  thyroid & 0.897 & 0.972 & 0.950 & 0.983 & 0.989 & \textbf{0.998} & 0.921 & 0.960 & 0.976 \tiny{(-0.022)} \\ 
  tic-tac-toe & \textbf{1.000} & 0.878 & 0.918 & 0.978 & 0.972 & 0.932 & 0.992 & 0.599 & 0.961 \tiny{(-0.039)} \\ 
  vowels & 0.854 & 0.693 & 0.773 & 0.796 & 0.758 & \textbf{0.897} & 0.813 & 0.748 & 0.834 \tiny{(-0.063)} \\ 
  waveform-2 & 0.567 & 0.716 & 0.648 & 0.847 & 0.333 & \textbf{0.886} & 0.540 & 0.774 & 0.826 \tiny{(-0.06)} \\ 
  wdbc & 0.836 & \textbf{0.999} & 0.896 & 0.843 & 0.899 & 0.836 & 0.620 & 0.942 & 0.947 \tiny{(-0.052)} \\ 
  \hline
  anuran & --- & 0.995 & 0.944 & 0.968 & \textbf{0.996} & 0.962 & 0.945 & 0.609 & 0.980 \tiny{(-0.016)} \\ 
  avila & --- & \textbf{0.999} & 0.977 & 0.987 & 0.993 & 0.920 & 0.729 & 0.618 & 0.977 \tiny{(-0.022)} \\ 
  car & --- & 0.956 & 0.939 & \textbf{0.978} & 0.931 & 0.885 & 0.935 & 0.831 & 0.977 \tiny{(-0.001)} \\ 
  contracept & --- & \textbf{0.680} & 0.597 & 0.653 & 0.607 & 0.598 & 0.598 & 0.551 & 0.636 \tiny{(-0.044)} \\ 
  drybeans & --- & 0.970 & 0.943 & 0.977 & 0.979 & 0.929 & 0.975 & 0.591 & \textbf{0.989} \\ 
  glass & --- & 0.970 & \textbf{0.984} & 0.975 & 0.940 & 0.937 & 0.926 & 0.793 & 0.953 \tiny{(-0.031)} \\ 
  heartcleveland & --- & 0.603 & 0.572 & \textbf{0.721} & 0.509 & 0.694 & 0.611 & 0.513 & 0.697 \tiny{(-0.024)} \\ 
  iris & --- & 0.960 & 0.975 & 0.970 & 0.962 & 0.977 & 0.954 & 0.812 & \textbf{0.981} \\ 
  pendigits & --- & 0.982 & 0.974 & 0.986 & 0.983 & 0.991 & 0.967 & 0.520 & \textbf{0.993} \\ 
  vehicle & --- & 0.856 & 0.789 & 0.870 & 0.859 & 0.858 & 0.764 & 0.586 & \textbf{0.877} \\ 
  waveform & --- & 0.842 & 0.814 & 0.910 & 0.880 & 0.803 & 0.654 & 0.516 & \textbf{0.911} \\ 
  wine & --- & 0.937 & 0.906 & 0.960 & 0.937 & \textbf{0.973} & 0.909 & 0.864 & 0.945 \tiny{(-0.028)} \\ 
   \hline
\end{tabular}
}
\caption{Average ROC-AUC scores obtained using cross-validation. BRS can only be applied to binary datasets, and DRS and IDS fail to get results on a few datasets, denoted as ``---" in the table. \crev{The best ROC-AUC for each dataset is shown in bold, and the difference between the best ROC-AUC for each dataset and the ROC-AUC of TURS for the same dataset is shown in parentheses.}} \label{table:roc_auc}
\end{table}

%% file: table_model_complexity2.tex
\begin{table}[ht]
\small
\centering
\begin{tabular}{lrrrrrrrrr}
  \hline
Data & BRS & C45 & CART & CLASSY & RIPPER & CN2 & DRS & IDS & TURS \\ 
  \hline
aloi & \textbf{{3}} & 2659.1\rlap{\textcolor{black}{$^{\ast}$}} & 26952.8 & 52.3 & 26.2\rlap{\textcolor{black}{$^{\ast}$}} & 2121 & --- & 14\rlap{\textcolor{black}{$^{\ast}$}} & 41.7 \\ 
  backdoor & \textbf{{13}} & 701.5 & 2460.5 & 72.6 & 101.5 & 264.6 & --- & --- & 50.7 \\ 
  backnote & 35.8 & 79.6 & 116.7 & 22.7 & 22.3 & 42.7 & 54 & 13.2\rlap{\textcolor{black}{$^{\ast}$}} & \textbf{{16.2}} \\ 
  chess & \textbf{{19.2}} & 250.2 & 340.5 & 33 & 57.9 & 298.7\rlap{\textcolor{black}{$^{\ast}$}} & 54.2\rlap{\textcolor{black}{$^{\ast}$}} & 14.5\rlap{\textcolor{black}{$^{\ast}$}} & 43.7 \\ 
  diabetes & 15.6 & 107.6 & 700.5 & \textbf{{5}} & 6.6\rlap{\textcolor{black}{$^{\ast}$}} & 170 & 82.5 & 13.2\rlap{\textcolor{black}{$^{\ast}$}} & 6 \\ 
  glass-2 & 10.8\rlap{\textcolor{black}{$^{\ast}$}} & 19.7 & 6.7\rlap{\textcolor{black}{$^{\ast}$}} & 3\rlap{\textcolor{black}{$^{\ast}$}} & 5.9\rlap{\textcolor{black}{$^{\ast}$}} & 3.9 & 37.2 & 15.5 & \textbf{{1}} \\ 
  ionosphere & 31.2 & 59 & 86.6 & 5.6 & 12.5 & 31 & 440.7\rlap{\textcolor{black}{$^{\ast}$}} & 13.2\rlap{\textcolor{black}{$^{\ast}$}} & \textbf{{4.4}} \\ 
  magic & 43.5\rlap{\textcolor{black}{$^{\ast}$}} & 3211.8 & 20053.1 & 228.4 & \textbf{{87.1}} & 3358.3\rlap{\textcolor{black}{$^{\ast}$}} & 44.7\rlap{\textcolor{black}{$^{\ast}$}} & 15.1\rlap{\textcolor{black}{$^{\ast}$}} & 154.4 \\ 
  mammography & 3\rlap{\textcolor{black}{$^{\ast}$}} & 273.7 & 1126.8\rlap{\textcolor{black}{$^{\ast}$}} & 38.7 & 30.5\rlap{\textcolor{black}{$^{\ast}$}} & 203.4 & \textbf{{24.4}} & 14\rlap{\textcolor{black}{$^{\ast}$}} & 30.2 \\ 
  musk & 6 & 4 & \textbf{{2}} & \textbf{{2}} & \textbf{{2}} & 4.5 & --- & 9.3 & \textbf{{2}} \\ 
  optdigits & 54.5 & 46.6 & 78.7 & 9.6 & 15.8 & 21.6 & --- & 11.2 & \textbf{{6.8}} \\ 
  pendigits-2 & 14.8 & 48.3 & 66.9 & \textbf{{9.5}} & 12.7 & 23.9 & 136.5 & 14.1 & 11.3 \\ 
  satimage-2 & 15 & 14.8 & 17.8 & 7.9 & 9 & 11.9 & 524.8\rlap{\textcolor{black}{$^{\ast}$}} & 12.4 & \textbf{{3.5}} \\ 
  smtp & 3\rlap{\textcolor{black}{$^{\ast}$}} & 20.9 & 34.4 & 5.6 & 7 & 20.1\rlap{\textcolor{black}{$^{\ast}$}} & 16.7 & 13.9 & \textbf{{3}} \\ 
  thyroid & \textbf{{3}} & 18.9 & 28.1 & 16.4 & 4.2 & 13.7 & 73.1 & 13.4 & 7.5 \\ 
  tic-tac-toe & \textbf{{25.2}} & 411.9 & 410.8 & 29.2 & 42.3 & 83.3 & 101.7 & 13.5\rlap{\textcolor{black}{$^{\ast}$}} & 37.6 \\ 
  vowels & 25.2 & 58.9\rlap{\textcolor{black}{$^{\ast}$}} & 147.8 & 14.1 & 13.5 & 27.9 & 141.9 & 14.3 & \textbf{{7.6}} \\ 
  waveform-2 & 4.2\rlap{\textcolor{black}{$^{\ast}$}} & 165.6\rlap{\textcolor{black}{$^{\ast}$}} & 406.3\rlap{\textcolor{black}{$^{\ast}$}} & 19.3 & 20.5\rlap{\textcolor{black}{$^{\ast}$}} & 84.5 & 522.2\rlap{\textcolor{black}{$^{\ast}$}} & 11.2 & \textbf{{9.7}} \\ 
  wdbc & 9\rlap{\textcolor{black}{$^{\ast}$}} & 8 & 2.5 & 3.2\rlap{\textcolor{black}{$^{\ast}$}} & \textbf{{2}} & 6.2\rlap{\textcolor{black}{$^{\ast}$}} & 337.6\rlap{\textcolor{black}{$^{\ast}$}} & 7.6 & \textbf{{2}} \\ 
  \hline
  anuran & --- & 80.9 & 1284.3 & 90.7 & \textbf{{11}} & 270.4 & 346.1 & 11.4\rlap{\textcolor{black}{$^{\ast}$}} & 93 \\ 
  avila & --- & 3460.4 & 7795.9 & 939.8 & 707.2 & 1302 & 181.3\rlap{\textcolor{black}{$^{\ast}$}} & 14\rlap{\textcolor{black}{$^{\ast}$}} & \textbf{{687}} \\ 
  car & --- & 659.8 & 776.1 & \textbf{{56.4}} & 171.7 & 73.3 & 307.6 & 13.5\rlap{\textcolor{black}{$^{\ast}$}} & 73.7 \\ 
  contracept & --- & 1268.1 & 5536.2 & 12 & 14.7 & 222.8 & \textbf{{5}} & 14 & 10 \\ 
  drybeans & --- & 2481.6 & 7039.4 & \textbf{{105.8}} & 153.3 & 1287.1 & 202.3 & 13.2\rlap{\textcolor{black}{$^{\ast}$}} & 168.7 \\ 
  glass & --- & 24.8 & 20 & 5.2 & 14 & 8.3 & 171 & 11.4\rlap{\textcolor{black}{$^{\ast}$}} & \textbf{{5}} \\ 
  heartcleveland & --- & 245.7 & 410.2\rlap{\textcolor{black}{$^{\ast}$}} & 6.6 & 5.5\rlap{\textcolor{black}{$^{\ast}$}} & 46.1 & 522.2 & 14\rlap{\textcolor{black}{$^{\ast}$}} & \textbf{{5.3}} \\ 
  iris & --- & 15.2 & 13 & 2.8 & 5.9 & 9.9 & 26.7 & 13\rlap{\textcolor{black}{$^{\ast}$}} & \textbf{{2.3}} \\ 
  pendigits & --- & 1274.9 & 1689.7 & \textbf{{142.3}} & 260.9 & 437.4 & 561.7 & 13.8\rlap{\textcolor{black}{$^{\ast}$}} & 151.1 \\ 
  vehicle & --- & 560.6 & 760.3 & 25.4 & 46.9 & 180.3 & 741.1\rlap{\textcolor{black}{$^{\ast}$}} & 12.4\rlap{\textcolor{black}{$^{\ast}$}} & \textbf{{21.5}} \\ 
  waveform & --- & 2782.9 & 3520 & \textbf{{125.7}} & 143 & 981.1\rlap{\textcolor{black}{$^{\ast}$}} & 75.6\rlap{\textcolor{black}{$^{\ast}$}} & 12.8\rlap{\textcolor{black}{$^{\ast}$}} & 128.5 \\ 
  wine & --- & 22 & 15.5 & 5 & 8.9 & 9.9 & 103.8 & 14.5 & \textbf{{4.8}} \\ 
   \hline
\end{tabular}
\caption{Total number of literals in the rule set, rule list, or decision tree, averaged over the five cross-validation folds. \crev{Entries marked with $^{\ast}$ indicate} models whose ROC-AUC is more than $0.1$ lower than that of TURS on the same dataset; bold entries denote the smallest total number of literals among the remaining models. }
\label{table:num_literal}
\end{table}

%% file: table_num_rules.tex
\begin{table}[ht]
\small
\centering
\begin{tabular}{lrrrrrrrrr}
  \hline
Data & BRS & C45 & CART & CLASSY & RIPPER & CN2 & DRS & IDS & TURS \\ 
  \hline
aloi & \textbf{{1}} & 157.2\rlap{\textcolor{black}{$^{\ast}$}} & 1056 & 12.4 & 4.2\rlap{\textcolor{black}{$^{\ast}$}} & 468.8 & --- & 7\rlap{\textcolor{black}{$^{\ast}$}} & 10 \\ 
  backdoor & \textbf{{4.6}} & 61.2 & 130.4 & 19 & 19 & 76.6 & --- & --- & 17.6 \\ 
  backnote & 12 & 14.8 & 23.4 & 8 & \textbf{{6.4}} & 13.6 & 20.4 & 7\rlap{\textcolor{black}{$^{\ast}$}} & 7 \\ 
  chess & \textbf{{6.8}} & 28.6 & 41.2 & 9.4 & 14.2 & 83.4\rlap{\textcolor{black}{$^{\ast}$}} & 12\rlap{\textcolor{black}{$^{\ast}$}} & 7.2\rlap{\textcolor{black}{$^{\ast}$}} & 13.4 \\ 
  diabetes & 5.2 & 18 & 87.2 & \textbf{{2.6}} & 2\rlap{\textcolor{black}{$^{\ast}$}} & 31.2 & 21.4 & 6.6\rlap{\textcolor{black}{$^{\ast}$}} & 3.4 \\ 
  glass-2 & 3.6\rlap{\textcolor{black}{$^{\ast}$}} & 5.2 & 3.6\rlap{\textcolor{black}{$^{\ast}$}} & 1\rlap{\textcolor{black}{$^{\ast}$}} & 1.8\rlap{\textcolor{black}{$^{\ast}$}} & 2.8 & 12.6 & 8.2 & \textbf{{1}} \\ 
  ionosphere & 10.4 & 10.4 & 16.8 & 4.6 & 5 & 9 & 33.6\rlap{\textcolor{black}{$^{\ast}$}} & 6.6\rlap{\textcolor{black}{$^{\ast}$}} & \textbf{{4.4}} \\ 
  magic & 14.6\rlap{\textcolor{black}{$^{\ast}$}} & 269.8 & 1300.8 & 63 & \textbf{{20.4}} & 477.6\rlap{\textcolor{black}{$^{\ast}$}} & 12\rlap{\textcolor{black}{$^{\ast}$}} & 7.6\rlap{\textcolor{black}{$^{\ast}$}} & 41.6 \\ 
  mammography & 1\rlap{\textcolor{black}{$^{\ast}$}} & 36.2 & 116.2\rlap{\textcolor{black}{$^{\ast}$}} & 10.2 & 6.4\rlap{\textcolor{black}{$^{\ast}$}} & 47 & \textbf{{7.8}} & 7\rlap{\textcolor{black}{$^{\ast}$}} & 9 \\ 
  musk & 2 & 2 & 2 & \textbf{{1}} & \textbf{{1}} & 3 & --- & 7 & \textbf{{1}} \\ 
  optdigits & 18.4 & 9.8 & 17.8 & 2 & 3.4 & 8.2 & --- & 5.6 & \textbf{{1.8}} \\ 
  pendigits-2 & 5 & 10.2 & 16 & \textbf{{2.2}} & 3.2 & 9.2 & 25 & 7.2 & 3.4 \\ 
  satimage-2 & 5 & 4.4 & 6 & 2.4 & 3.6 & 5.2 & 32.8\rlap{\textcolor{black}{$^{\ast}$}} & 6.2 & \textbf{{1.4}} \\ 
  smtp & 1\rlap{\textcolor{black}{$^{\ast}$}} & 5.6 & 8.4 & 2.4 & \textbf{{2}} & 9\rlap{\textcolor{black}{$^{\ast}$}} & 7.8 & 5.2 & \textbf{{2}} \\ 
  thyroid & \textbf{{1}} & 5 & 8.4 & 4.4 & 1.4 & 5.4 & 17.8 & 7.2 & 2.6 \\ 
  tic-tac-toe & \textbf{{8.4}} & 75 & 60.4 & 9.6 & 9.8 & 26.8 & 26 & 6.8\rlap{\textcolor{black}{$^{\ast}$}} & 9.2 \\ 
  vowels & 8.4 & 11.4\rlap{\textcolor{black}{$^{\ast}$}} & 22 & 3.2 & 3.6 & 9.2 & 31.4 & 7.6 & \textbf{{2.4}} \\ 
  waveform-2 & 1.4\rlap{\textcolor{black}{$^{\ast}$}} & 21.6\rlap{\textcolor{black}{$^{\ast}$}} & 56.2\rlap{\textcolor{black}{$^{\ast}$}} & 4.8 & 3.4\rlap{\textcolor{black}{$^{\ast}$}} & 23.6 & 54.8\rlap{\textcolor{black}{$^{\ast}$}} & 5.8 & \textbf{{3.4}} \\ 
  wdbc & 3\rlap{\textcolor{black}{$^{\ast}$}} & 3 & 2.2 & 1\rlap{\textcolor{black}{$^{\ast}$}} & \textbf{{1}} & 3.2\rlap{\textcolor{black}{$^{\ast}$}} & 29.6\rlap{\textcolor{black}{$^{\ast}$}} & 7.6 & \textbf{{1}} \\ 
  \hline
  anuran & --- & 15 & 147.2 & 23.4 & \textbf{{4}} & 60.2 & 38.4 & 5.7\rlap{\textcolor{black}{$^{\ast}$}} & 25.6 \\ 
  avila & --- & 279.6 & 594.8 & 224 & \textbf{{141.2}} & 358.4 & 29.4\rlap{\textcolor{black}{$^{\ast}$}} & 7\rlap{\textcolor{black}{$^{\ast}$}} & 156.4 \\ 
  car & --- & 117.2 & 82.2 & \textbf{{11.6}} & 35.6 & 31.2 & 66.2 & 7\rlap{\textcolor{black}{$^{\ast}$}} & 14.8 \\ 
  contracept & --- & 116.2 & 428 & 4.4 & \textbf{{3.4}} & 65.8 & 5 & 7 & 4.8 \\ 
  drybeans & --- & 220 & 577 & 37.4 & 34.6 & 212.8 & \textbf{{30.8}} & 6.6\rlap{\textcolor{black}{$^{\ast}$}} & 48 \\ 
  glass & --- & 6 & 6 & \textbf{{4}} & 5 & \textbf{{4}} & 33.8 & 6.1\rlap{\textcolor{black}{$^{\ast}$}} & 4.4 \\ 
  heartcleveland & --- & 42.4 & 59.8\rlap{\textcolor{black}{$^{\ast}$}} & 2.6 & 1.6\rlap{\textcolor{black}{$^{\ast}$}} & 10.2 & 87 & 7\rlap{\textcolor{black}{$^{\ast}$}} & \textbf{{2.4}} \\ 
  iris & --- & 4.4 & 5 & \textbf{{2}} & 2.6 & 4 & 11.2 & 6.5\rlap{\textcolor{black}{$^{\ast}$}} & 2.2 \\ 
  pendigits & --- & 132.8 & 186.6 & 40 & 58.4 & 92 & 98.8 & 6.9\rlap{\textcolor{black}{$^{\ast}$}} & \textbf{{32.8}} \\ 
  vehicle & --- & 65 & 94.8 & 9.8 & 13.2 & 32.8 & 96.6\rlap{\textcolor{black}{$^{\ast}$}} & 6.2\rlap{\textcolor{black}{$^{\ast}$}} & \textbf{{7.6}} \\ 
  waveform & --- & 255.4 & 349.8 & 31.2 & \textbf{{25.4}} & 149.4\rlap{\textcolor{black}{$^{\ast}$}} & 14.8\rlap{\textcolor{black}{$^{\ast}$}} & 6.4\rlap{\textcolor{black}{$^{\ast}$}} & 32 \\ 
  wine & --- & 6 & 5.8 & \textbf{{2}} & 3.2 & 4 & 26.8 & 7.3 & \textbf{{2}} \\ 
   \hline
\end{tabular}
\caption{The number of rules in the learned models. For decision trees, each root-to-leaf path is treated as a rule. \crev{Entries marked with $^{\ast}$ indicate} models whose ROC-AUC is more than $0.1$ lower than that of TURS on the same dataset; bold entries denote the smallest number of rules among the remaining models.}\label{table:num_rules}
\end{table}

%% file: table_avg_rule_length.tex
\begin{table}[ht]
\small
\centering
\begin{tabular}{lrrrrrrrrr}
  \hline
Data & BRS & C45 & CART & CLASSY & RIPPER & CN2 & DRS & IDS & TURS \\ 
  \hline
aloi & \textbf{{3}} & 16.9\rlap{\textcolor{black}{$^{\ast}$}} & 25.5 & 4.2 & 6.2\rlap{\textcolor{black}{$^{\ast}$}} & 4.5 & --- & 2\rlap{\textcolor{black}{$^{\ast}$}} & 4.2 \\ 
  backdoor & \textbf{{2.8}} & 11.5 & 18.9 & 3.8 & 5.3 & 3.5 & --- & --- & 2.9 \\ 
  backnote & 3 & 5.4 & 5 & 2.8 & 3.5 & 3.1 & 2.6 & 1.9\rlap{\textcolor{black}{$^{\ast}$}} & \textbf{{2.3}} \\ 
  chess & \textbf{{2.8}} & 8.7 & 8.3 & 3.5 & 4.1 & 3.6\rlap{\textcolor{black}{$^{\ast}$}} & 4.5\rlap{\textcolor{black}{$^{\ast}$}} & 2\rlap{\textcolor{black}{$^{\ast}$}} & 3.3 \\ 
  diabetes & 3 & 6 & 8 & 1.9 & 3.3\rlap{\textcolor{black}{$^{\ast}$}} & 5.4 & 3.9 & 2\rlap{\textcolor{black}{$^{\ast}$}} & \textbf{{1.8}} \\ 
  glass-2 & 3\rlap{\textcolor{black}{$^{\ast}$}} & 3.8 & 1.9\rlap{\textcolor{black}{$^{\ast}$}} & 3\rlap{\textcolor{black}{$^{\ast}$}} & 3.3\rlap{\textcolor{black}{$^{\ast}$}} & 1.4 & 3 & 1.9 & \textbf{{1}} \\ 
  ionosphere & 3 & 5.7 & 5.2 & 1.2 & 2.5 & 3.5 & 13.1\rlap{\textcolor{black}{$^{\ast}$}} & 2\rlap{\textcolor{black}{$^{\ast}$}} & \textbf{{1}} \\ 
  magic & 3\rlap{\textcolor{black}{$^{\ast}$}} & 11.9 & 15.4 & \textbf{{3.6}} & 4.3 & 7\rlap{\textcolor{black}{$^{\ast}$}} & 3.7\rlap{\textcolor{black}{$^{\ast}$}} & 2\rlap{\textcolor{black}{$^{\ast}$}} & 3.7 \\ 
  mammography & 3\rlap{\textcolor{black}{$^{\ast}$}} & 7.6 & 9.7\rlap{\textcolor{black}{$^{\ast}$}} & 3.8 & 4.8\rlap{\textcolor{black}{$^{\ast}$}} & 4.3 & \textbf{{3.1}} & 2\rlap{\textcolor{black}{$^{\ast}$}} & 3.4 \\ 
  musk & 3 & 2 & \textbf{{1}} & 2 & 2 & 1.5 & --- & 1.3 & 2 \\ 
  optdigits & 3 & 4.8 & 4.4 & 4.8 & 4.7 & 2.6 & --- & \textbf{{2}} & 3.8 \\ 
  pendigits-2 & 3 & 4.7 & 4.2 & 4.3 & 4 & 2.6 & 5.5 & \textbf{{2}} & 3.3 \\ 
  satimage-2 & 3 & 3.4 & 3 & 3.3 & 2.5 & 2.3 & 16\rlap{\textcolor{black}{$^{\ast}$}} & \textbf{{2}} & 2.5 \\ 
  smtp & 3\rlap{\textcolor{black}{$^{\ast}$}} & 3.7 & 4.1 & 2.3 & 3.5 & 2.2\rlap{\textcolor{black}{$^{\ast}$}} & 2.1 & 2.7 & \textbf{{1.5}} \\ 
  thyroid & 3 & 3.8 & 3.4 & 3.7 & 3 & 2.5 & 4.1 & \textbf{{1.9}} & 2.9 \\ 
  tic-tac-toe & \textbf{{3}} & 5.5 & 6.8 & 3 & 4.3 & 3.1 & 3.9 & 2\rlap{\textcolor{black}{$^{\ast}$}} & 4.1 \\ 
  vowels & 3 & 5.2\rlap{\textcolor{black}{$^{\ast}$}} & 6.7 & 4.4 & 3.8 & 3 & 4.5 & \textbf{{1.9}} & 3.2 \\ 
  waveform-2 & 3\rlap{\textcolor{black}{$^{\ast}$}} & 7.7\rlap{\textcolor{black}{$^{\ast}$}} & 7.2\rlap{\textcolor{black}{$^{\ast}$}} & 4 & 6\rlap{\textcolor{black}{$^{\ast}$}} & 3.6 & 9.5\rlap{\textcolor{black}{$^{\ast}$}} & \textbf{{1.9}} & 2.9 \\ 
  wdbc & 3\rlap{\textcolor{black}{$^{\ast}$}} & 2.7 & 1.1 & 3.2\rlap{\textcolor{black}{$^{\ast}$}} & 2 & 1.9\rlap{\textcolor{black}{$^{\ast}$}} & 11.4\rlap{\textcolor{black}{$^{\ast}$}} & \textbf{{1}} & 2 \\
  \hline
  anuran & --- & 5.4 & 8.7 & 3.9 & \textbf{{2.8}} & 4.5 & 9 & 2\rlap{\textcolor{black}{$^{\ast}$}} & 3.6 \\ 
  avila & --- & 12.4 & 13.1 & 4.2 & 5 & \textbf{{3.6}} & 6.2\rlap{\textcolor{black}{$^{\ast}$}} & 2\rlap{\textcolor{black}{$^{\ast}$}} & 4.4 \\ 
  car & --- & 5.6 & 9.4 & 4.9 & 4.8 & \textbf{{2.4}} & 4.6 & 1.9\rlap{\textcolor{black}{$^{\ast}$}} & 5 \\ 
  contracept & --- & 10.9 & 12.9 & 2.7 & 4.3 & 3.4 & \textbf{{1}} & 2 & 2.1 \\ 
  drybeans & --- & 11.3 & 12.2 & \textbf{{2.8}} & 4.4 & 6 & 6.6 & 2\rlap{\textcolor{black}{$^{\ast}$}} & 3.5 \\ 
  glass & --- & 4.1 & 3.3 & 1.3 & 2.8 & 2.1 & 5.1 & 1.9\rlap{\textcolor{black}{$^{\ast}$}} & \textbf{{1.1}} \\ 
  heartcleveland & --- & 5.8 & 6.9\rlap{\textcolor{black}{$^{\ast}$}} & 2.5 & 3.4\rlap{\textcolor{black}{$^{\ast}$}} & 4.5 & 6 & 2\rlap{\textcolor{black}{$^{\ast}$}} & \textbf{{2.2}} \\ 
  iris & --- & 3.5 & 2.6 & 1.4 & 2.3 & 2.5 & 2.4 & 2\rlap{\textcolor{black}{$^{\ast}$}} & \textbf{{1.1}} \\ 
  pendigits & --- & 9.6 & 9.1 & \textbf{{3.6}} & 4.5 & 4.8 & 5.7 & 2\rlap{\textcolor{black}{$^{\ast}$}} & 4.6 \\ 
  vehicle & --- & 8.6 & 8 & \textbf{{2.6}} & 3.5 & 5.5 & 7.7\rlap{\textcolor{black}{$^{\ast}$}} & 2\rlap{\textcolor{black}{$^{\ast}$}} & 2.8 \\ 
  waveform & --- & 10.9 & 10.1 & 4 & 5.6 & 6.6\rlap{\textcolor{black}{$^{\ast}$}} & 5.1\rlap{\textcolor{black}{$^{\ast}$}} & 2\rlap{\textcolor{black}{$^{\ast}$}} & \textbf{{4}} \\ 
  wine & --- & 3.7 & 2.7 & 2.5 & 2.8 & 2.5 & 3.9 & \textbf{{2}} & 2.4 \\ 
   \hline
\end{tabular}
\caption{Average rule length in the learned models.
For decision trees, rule length is computed over root-to-leaf paths.
\crev{Entries marked with $^{\ast}$ indicate} models whose ROC-AUC is more than $0.1$ lower than that of TURS on the same dataset; bold entries denote the smallest average rule length among the remaining models.} \label{table:avg_rule_length}
\end{table}